\documentclass[11pt]{article}
\usepackage{fullpage}

\usepackage[utf8]{inputenc}
\usepackage[T1]{fontenc}
\usepackage{hyperref}
\usepackage{url}
\usepackage{booktabs}
\usepackage{amsfonts}
\usepackage{nicefrac}
\usepackage{microtype}
\usepackage{diagbox}
\usepackage{enumerate}
\usepackage[shortlabels]{enumitem}
\usepackage{algorithmic}
\usepackage[vlined,linesnumbered,ruled]{algorithm2e}

\usepackage{subfigure}
\usepackage{graphicx}
\usepackage{caption}
\usepackage{amsmath}
\usepackage{amsthm}
\usepackage{amssymb}
\usepackage{tablefootnote}
\usepackage{multirow}
\usepackage{xcolor}
\usepackage{natbib}
\usepackage{titletoc}
\usepackage[page,header,toc,page]{appendix}
\allowdisplaybreaks[4]
\usepackage{mathrsfs}
\usepackage{bm,bbm}

\newtheorem{Rema}{Remark}[section]

\newcommand{\eq}[1]{Eq.~\eqref{#1}}
\newcommand{\Err}{\operatorname{Err}}

\hypersetup{
  colorlinks=true,
  linkcolor=red,
  filecolor=blue,
  citecolor=blue,
  urlcolor=cyan,
}

\newcommand{\1}{\mathbbm{1}}				% indicator
\renewcommand{\P}{\mathbb{P}}				% probability
\newcommand{\E}{\mathbb{E}}					% expectation
\newcommand{\agp}[1]{\langle #1 \rangle}	% angle process
\newcommand{\normmm}[1]{\lvert\kern-0.25ex\lvert\kern-0.25ex\lvert #1 \rvert\kern-0.25ex\rvert\kern-0.25ex\rvert}

\newcommand{\R}{\mathbb{R}}		% reals
\renewcommand*{\d}{\mathop{}\!\mathrm{d}}	% differential
\newcommand{\argmin}{\operatornamewithlimits{arg\,min}}

\newtheorem{Def}{Definition}[section]
\newtheorem{Lem}[Def]{Lemma}
\newtheorem{Thm}[Def]{Theorem}
\newtheorem{Cor}[Def]{Corollary}

\title{Does 1/2-Tsallis-INF Also Work Well for Best-Arm Identification?}

\author{
Jingxin Zhan
\thanks{School of Mathematical Sciences, Peking University; email: \texttt{bjdxzjx@pku.edu.cn}.}
\and
Yuze Han
\thanks{Center for Applied Statistics and School of Statistics, Renmin University of China; email: \texttt{hanyuze97@ruc.edu.cn}.}
\and
Zhihua Zhang
\thanks{School of Mathematical Sciences, Peking University; email: \texttt{zhzhang@math.pku.edu.cn}.}
}

\begin{document}

\maketitle

\begin{abstract}
Regret minimization (RM) and best-arm identification (BAI) are two fundamental objectives in multi-armed bandits.  Among regret-minimizing algorithms, $1/2$-Tsallis-INF is a canonical best-of-both-worlds FTRL algorithm: it achieves logarithmic pseudo-regret in stochastic bandits while retaining minimax-optimal regret in adversarial bandits, without knowing the environment in advance.  This raises a natural question: can the same algorithm, without additional exploration, also identify the best arm reliably? We study this question in stochastic bandits by analyzing the failure probability $\Err_t$, defined as the probability that the empirical best arm determined by the cumulative importance-weighted loss estimates of 1/2-Tsallis-INF differs from the true optimal arm. The main difficulty is that, at the logarithmic-regret scale, suboptimal arms are sampled with probability heuristically of order $1/t$. Consequently, importance weighting causes the cumulative estimator to fluctuate on the same linear scale as its mean separation. To overcome this obstacle, guided by a diffusion toy model, we construct a Lyapunov function for the gap process between the estimated cumulative loss of the optimal arm and that of the best competing arm. This leads to polynomial upper bounds on $\Err_t$: for learning rate $\eta_t=\alpha/\sqrt t$, $\Err_t$ decays at rate $t^{-2+\alpha^2\mu_{i_*}/4+\rho}$ for any $\rho>0$, where $\mu_{i_*}$ denotes the mean loss of the true optimal arm.  We also establish a lower bound $\Omega(t^{-2-\varepsilon})$ for any $\varepsilon>0$, showing that the exponent $2$ is essentially tight.
\end{abstract}

\section{Introduction}

Multi-armed bandits \citep{thompson1933likelihood,robbins1952some} provide a canonical framework for sequential decision making under uncertainty. 
Two fundamental objectives in this framework are regret minimization (RM) \citep{lai1985asymptotically,auer2002finite} and best-arm identification (BAI) \citep{even2002pac,mannor2004sample}.
The RM objective evaluates the cumulative performance of the learner during the interaction, whereas BAI focuses on identifying the optimal arm after a given amount of exploration. 
These two objectives reflect different aspects of the exploration--exploitation trade-off: regret minimization favors exploitation in order to avoid pulling suboptimal arms too often, while best-arm identification requires sufficient exploration to distinguish the optimal arm from its competitors.

The tension between RM and BAI has been extensively studied. 
Classical work on pure exploration investigates 
%simple regret and 
best-arm identification and cumulative regret simultaneously \citep{stoltz:hal-00609550}. 
More recent work further characterizes the Pareto frontier between RM and BAI, showing that no algorithm can simultaneously achieve optimal performance for both objectives in general \citep{zhong2021achieving}. 
In particular, these lower bounds imply that if an algorithm achieves logarithmic regret, then its BAI failure probability can decay at most polynomially in the worst case. 
% However, such results do not determine the precise BAI behavior of a concrete logarithmic-regret algorithm. 
% This is especially important because constants that are usually ignored in regret analysis may become decisive when they appear in the exponent of a BAI probability bound. 
However, these results focuses on UCB-type stochastic algorithms \citep{auer2002finite}, which are tailored to stochastic environments and do not enjoy robust adversarial regret guarantees.

However, in many practical scenarios, it is unclear whether the environment is stochastic or adversarial.
This motivates us to study the BAI behavior of algorithms that are not merely logarithmic-regret algorithms in stochastic bandits, but also robust to adversarial environments. 
This is precisely the goal of best-of-both-worlds (BOBW) algorithms \citet{pmlr-v23-bubeck12b}. 
Among them, \(1/2\)-Tsallis-INF is a canonical example. 
It can be viewed as an FTRL algorithm with the \(1/2\)-Tsallis entropy regularizer, and it achieves logarithmic pseudo-regret in stochastic bandits while retaining minimax-optimal regret in adversarial bandits, without knowing the regime or the time horizon in advance \citep{zimmert2021tsallis}. 
Also, the BOBW property enjoyed by 
\(1/2\)-Tsallis-INF and its variants is fairly general. It plays an important role not only in the MAB setting considered in this paper, but also in problems such as linear bandits \citep{zhao2025heavytailedlinearbanditsadversarial}, semi-bandits \citep{zimmert2019beating}, and MDPs \citep{chen2026bestofbothworldsheavytailedmarkovdecision}.
This raises a natural question: can the same BOBW algorithm, without additional forced exploration, also identify the best arm reliably?

In this paper, we study this question in stochastic bandits. 
Let \(\hat L_{t,i}\) be the cumulative importance-weighted loss estimate of arm \(i\) used by \(1/2\)-Tsallis-INF, and let
\(
    i_t^* \in \argmin_{i\in\mathcal I} \hat L_{t,i}
\)
be the empirical best arm induced by these estimates. 
Our object of interest is the failure probability
\(
    \Err_t := \mathbb P(i_t^* \neq i_*),
\)
where \(i_*\) is the unique optimal arm. 
Equivalently, we ask how fast the cumulative importance-weighted estimates separate the true optimal arm from all suboptimal arms. 
This differs from the standard regret analysis, which controls the cumulative loss incurred by the sampling distribution.
% , and also differs from last-iterate simple-regret analysis, which studies the convergence of the current sampling distribution \(p_t\). 
Our focus is instead the probability that the empirical minimizer of the estimated cumulative losses is incorrect.

The main technical obstacle is the high variance caused by importance weighting (IW). 
For UCB-type algorithms, BAI and simple-regret analyses typically rely on concentration inequalities for empirical means. 
By contrast, \(1/2\)-Tsallis-INF updates using importance-weighted estimates. 
At the logarithmic-regret scale, a suboptimal arm is expected heuristically to be sampled with probability of order \(1/t\). 
Then its importance-weighted loss estimate has variance of order \(t\) at a single round, so the cumulative estimate fluctuates on the same linear scale as its mean separation from the optimal arm. 
Therefore, standard concentration arguments do not directly apply. This variance issue is bypassed in the regret analysis of Tsallis-INF through self-bounding techniques \citep{zimmert2021tsallis}, but such arguments do not directly yield the probability that the estimated best arm is correct. Besides, there are some prior works avoiding this issue by adding extra explicit exploration to the FTRL algorithm \citep{zhao2025heavytailedlinearbanditsadversarial}, which, however, will destroy the desirable BOBW property.
% To address this, we directly analyze the high-variance gap process between the estimated cumulative loss of the optimal arm and that of the best competing arm.

\subsection{Contributions}

We provide a quantitative BAI-type guarantee for \(1/2\)-Tsallis-INF in stochastic bandits. 
% Without modifying the algorithm or adding forced exploration, we prove that the empirical best arm induced by its original importance-weighted cumulative loss estimates identifies the true optimal arm with $t^{-2+\alpha^2\mu_{i_*}/4+\rho}$ failure probability for learning rates of order \(\alpha/\sqrt t\) and any $\rho>0$, which approaches \(t^{-2}\) as $\alpha$ tends to zero. 
Without modifying the algorithm or introducing additional exploration, we show that the empirical best arm, determined by the original importance-weighted cumulative loss estimates, identifies the true optimal arm with failure probability $t^{-2+\alpha^2\mu_{i_*}/4+\rho}$ for learning rates of order $\alpha/\sqrt{t}$ and any $\rho>0$. 
Moreover, this rate approaches $t^{-2}$ as $\alpha \to 0$.
We also prove a 
%matching
corresponding
polynomial lower bound, showing that $t^{-2}$ is essentially tight. 
To the best of our knowledge, this is the first BAI-type guarantee for an FTRL algorithm without additional exploration.

Technically, standard concentration arguments fail because the importance-weighted cumulative estimates may fluctuate on the same order as their mean separation. 
We overcome this difficulty through a Lyapunov function analysis of the estimated loss-gap process. 
The Lyapunov function is guided by two toy models, one diffusion-based and one discrete, which reveal the relevant drift--variance balance and the origin of the exponent \(2\). 
This approach is potentially useful for studying \(1/2\)-Tsallis-INF in other BOBW settings and broader Tsallis-entropy FTRL algorithms under importance-weighted bandit feedback.

\subsection{Related Works}
%BAI
%\paragraph{Best-Arm Identification (BAI)}
The BAI problem is typically studied under two main settings. The first is the fixed-budget setting \citep{Audibert10}, which evaluates the ability to identify the optimal arm within a limited number of rounds, and is the focus of this paper; see, e.g., \citep{komiyama2026rateoptimaldesignanytimebest,DBLP:conf/icml/ZhaoSSJ23,wang2023best}. 
In particular, recent work has explored the Pareto frontier between regret minimization (RM) and BAI \citep{zhao2025heavytailedlinearbanditsadversarial}. Building on this line of work, we further investigate the BAI properties of BOBW algorithms. The second is the fixed-confidence setting \citep{jamiesonsurvey}, where the learner adaptively decides whether to stop sampling or continue gathering information; see, e.g.,  \citep{DBLP:conf/nips/Bandyopadhyay0A24,DBLP:conf/colt/YouQWY23,komiyama2022minimax}.

% \paragraph{Best-of-Both-Worlds (BOBW)}
% % In many practical scenarios, it is unclear whether the environment is stochastic or adversarial, 
% % motivating the design of algorithms with regret guarantees in both settings, which has led to the development of 
% % BOBW algorithms that aim for optimal performance in both regimes.
% In bandit problems, the first BOBW algorithm was proposed by \citet{pmlr-v23-bubeck12b}, 
% and later the well-known Tsallis-INF algorithm considered in this paper was introduced by \citet{pmlr-v89-zimmert19a}. 
%Subsequent research has extensively explored BOBW algorithms in various online learning settings. For instance linear bandits \citep{ito2023bestofthreeworldslinearbanditalgorithm,kong2023bestofthreeworldsanalysislinearbandits,zhao2025heavytailedlinearbanditsadversarial}, combinatorial semi-bandits \citep{wei2018adaptivealgorithmsadversarialbandits,zimmert2019beating,zhan2025followtheperturbedleaderapproachesbestofbothworldsmset}, 
% dueling bandits \citep{pmlr-v162-saha22a}, contextual bandits \citep{pmlr-v238-kuroki24a,kato2023best} and episodic Markov decision processes \citep{jin2021bestworldsstochasticadversarial,ito2025adaptingstochasticadversariallosses}.

\paragraph{Organization}
The remainder of this paper is organized as follows. Section \ref{Preliminaries} gives preliminaries. Section \ref{sec:main} presents our main results. Section \ref{construct} introduces two toy models and explains how to derive the Lyapunov function. Section \ref{pfsklya} and Section \ref{pfsklower} provides the proof sketches of the upper bound and the lower bound in our main results respectively. 
Concluding remarks are given in Section \ref{sec:cr}. 
%All of the details are included in the Appendix.

%照搬还没改
\section{Preliminaries}\label{Preliminaries}
In this section, we formulate the problem and introduce the FTRL policy. 

\subsection{The Problem Setting and Notation}
% We consider the multi-armed bandit problem with arm set  
% $
% \mathcal{I}=\{1, \ldots, d\}
% $.  
% At each round \( t = 1, 2, \ldots, \) the learner chooses an arm \( I_t \in \mathcal{I} \), while the environment generates a loss vector \( \ell_t \in [0,1]^d \). The learner observes and incurs $\ell_{t,I_t}$, i.e., only
% the loss of the chosen arm. We consider the stochastic setting, where the loss vectors are i.i.d. samples from a fixed but unknown distribution \( \mathcal{D} \). Let \( \mu_i := \mathbb{E}_{\ell \sim \mathcal{D}}[\ell_i] \) denote the expected loss of arm \( i \) and  define \( i_* = \arg\min_{i \in \mathcal{I}}\mu_i \), which is the optimal fixed arm in hindsight and assumed to be unique. Define the suboptimality gap of arm \( i \) as  $
% \Delta_i := \mu_i-\mu_{i_*},
% $  
% and the minimum gap as \( \Delta := \min_{i : \Delta_i > 0} \Delta_i \). 
We study a stochastic multi-armed bandit problem with arm set 
$\mathcal{I}=\{1,\ldots,d\}$. 
At each round $t=1,2,\ldots$, the learner selects an arm $I_t\in\mathcal{I}$, while the environment produces a loss vector $\ell_t\in[0,1]^d$. 
The learner only observes and suffers the loss $\ell_{t,I_t}$ of the chosen arm. 
We focus on the stochastic regime, in which the loss vectors are i.i.d.\ draws from an unknown distribution $\mathcal{D}$. 
Let $\mu_i:=\mathbb{E}_{\ell\sim\mathcal{D}}[\ell_i]$ be the expected loss of arm $i$, and define $i_*=\arg\min_{i\in\mathcal{I}}\mu_i$ as the unique optimal arm in hindsight. 
For each arm $i$, define the suboptimality gap as 
$\Delta_i:=\mu_i-\mu_{i_*}$, 
and denote the smallest positive gap by $\Delta:=\min_{i:\Delta_i>0}\Delta_i$.

% In the bandit setting, we do not observe the complete loss vector $\ell_t$. Instead, an unbiased estimator $\hat{\ell}_t$ for $\ell_t$ is used to update the cumulative losses, 
% $
% \hat{L}_t = \sum_{s=1}^{t-1} \hat{\ell}_s
% $. The common way to construct unbiased loss estimators is through importance-weighted
% sampling:
In the bandit setting, the full loss vector $\ell_t$ is not directly observable. 
Instead, we employ an unbiased estimator $\hat{\ell}_t$ of $\ell_t$ to update the cumulative losses,
$\hat{L}_t=\sum_{s=1}^{t-1}\hat{\ell}_s$. 
A standard approach to obtain such unbiased estimators is via importance-weighted sampling:
\begin{equation}\label{eq:loss-vector-est}
\hat{\ell}_{t,i}=\frac{\ell_{t,i}A_{t,i}}{p_{t,i}},\,\text{where $A_{t,i}=\1_{\{I_t=i\}}$}.
\end{equation}
Let $i_{t}^*=\argmin_{i\in\mathcal{I}}\hat{L}_{t,i}$ (ties are broken arbitrarily).
Note that $i_t^*$ is random and does not need to be $i_*$. We study the failure probability $\Err_t=\P(i_{t}^*\ne i_*)$.

To study \(\Err_t\), we introduce several auxiliary quantities. Let
$M_t=\min _{i\in\mathcal{I}} \hat{L}_{t,i}$ and $N_t=\min _{i\in\mathcal{I}, i\ne i_*} \hat{L}_{t,i}$.
We also define $D_t=\hat{L}_{t,i_*}-N_t$.
Here, \(M_t\) is the minimum cumulative estimated loss over all arms, while \(N_t\) is the minimum cumulative estimated loss over all suboptimal arms. Thus, \(D_t\) measures the gap between the cumulative estimated loss of the optimal arm \(i_*\) and the smallest cumulative estimated loss among the suboptimal arms.
The quantity \(D_t\) is directly related to the failure probability \(\Err_t\). Specifically,
\begin{equation}\label{eq:etdt}
    \P(D_t>0)\leq \Err_t\leq \P(D_t\ge 0).
\end{equation}
Therefore, the analysis of \(\Err_t\) can be reduced to studying whether \(D_t\) is positive. This observation serves as the starting point of our subsequent analysis.

\textbf{Notation}\quad
Throughout this paper, $C$ or $C'$ denotes a generic positive constant whose value may change from line to line,
and is determined by the context. In the following, we denote by $\Delta^{d-1}$ the probability simplex. 
Let $\underline{\hat{L}}_{t}=\hat{L}_t-M_t\mathbf{1}$, where $\mathbf{1}:=(1, \ldots, 1)^{\top}$.
% For any $\lambda =(\lambda_1, \ldots, \lambda_d)^\top\in\R^d$, let $\underline{\lambda}=\lambda-\min\limits_{i=1,\cdots,d}\lambda_i\mathbf{1},$ where $\mathbf{1}:=(1, \ldots, 1)^{\top}$, and especially
% % in this paper we will primarily use this notation for $\hat{L}_t$, i.e.,
% $\underline{\hat{L}}_{t}=\hat{L}_t-M_t\mathbf{1}$. 
% Given a convex function $\Psi$ on $\R^d$, 
% we define the Bregman divergence induced by $\Psi$ as
% \(
% D_{\Psi}(x,y):=\Psi(x)-\Psi(y)-\agp{x-y,\nabla\Psi(y)}, \; \forall x, y \in \R^d.
% \)
% For any $\lambda\in\R^d$, 
% we define
% \begin{equation}\label{defiphi}
%     \phi(\lambda):=\argmax_{p\in\Delta^{d-1}}\agp{p,\lambda}-\Psi(p).
% \end{equation}
% Some important properties of $\phi$ are collected in Appendix \ref{important}.
We denote by $\mathscr{F}_t$ the filtration $\sigma(I_1,\ell_{1,I_1}, \ldots, I_t,\ell_{t,I_t})$ and define $\E_t[\cdot]:=\E[\cdot\mid \mathscr{F}_{t}]$. 
% For any \( i \in   \mathcal{I}  \), let \( e_i \) denote the \( i \)-th coordinate basis vector in \( \mathbb{R}^d \).
% For any $a,b\in\R$, we denote $a^+=\max(a,0)$, $a\wedge b=\min(a,b)$ and $a\vee b=\max(a,b)$.
For any $a\in\R$, we denote $a^+=\max(a,0)$.

\subsection{FTRL Policy}
% We study the Follow-The-Regularized-Leader (FTRL) algorithm. At each time step, the algorithm chooses a probability
% distribution $p_t$ over $d$ arms:
% \begin{equation}\label{ptdefi}
%     p_t=\argmin_{p\in\Delta^{d-1}}\eta_t\agp{p,\hat{L}_t}+\Psi(p),
% \end{equation}
% where $\eta_t$ is the learning rate and $\Psi(p)$ is a convex regularization function.  

% In simple settings, $\Psi(p)$ is typically chosen so that its Hessian matrix satisfies
% $\nabla^2 \Psi(p)=\operatorname{diag}\{p_1^\beta,\ldots,p_d^\beta\}$,
% where $\beta$ is a fixed constant. Common examples include the negentropy ($\Psi(p)=\sum_{i=1}^d p_i\log p_i$, corresponding to the EXP-3 algorithm), the log-barrier ($\Psi(p)=-\sum_{i=1}^d\log p_i$), and the $\theta$-Tsallis entropy ($\Psi(p)=-\sum_{i=1}^d p_i^\theta$, $0<\theta<1$). \citet{pmlr-v89-zimmert19a} showed that the $1/2$-Tsallis-INF algorithm (Algorithm~\ref{alg: FTRL}), that is, FTRL with $\Psi(p)=-4\sum_{i=1}^d \sqrt{p_i}$, achieves the BOBW guarantee, and provided intuitive arguments indicating that other regularizers cannot do so.
We consider the Follow-The-Regularized-Leader (FTRL) framework. 
At each round, the algorithm selects an arm according to a probability vector $p_t$ defined by:
\begin{equation}\label{ptdefi}
    p_t=\arg\min_{p\in\Delta^{d-1}}\;\eta_t\langle p,\hat{L}_t\rangle+\Psi(p),
\end{equation}
where $\eta_t$ denotes the learning rate and $\Psi(p)$ is a convex regularizer.

In basic scenarios, $\Psi(p)$ is often chosen such that its Hessian takes the form
$\nabla^2\Psi(p)=\operatorname{diag}\{p_1^\beta,\ldots,p_d^\beta\}$,
for some constant $\beta$. Representative instances include the negentropy 
($\Psi(p)=\sum_{i=1}^d p_i\log p_i$, which yields the EXP-3 algorithm), 
the log-barrier ($\Psi(p)=-\sum_{i=1}^d \log p_i$), 
and the $\theta$-Tsallis entropy ($\Psi(p)=-\sum_{i=1}^d p_i^\theta$, with $0<\theta<1$). 
\citet{pmlr-v89-zimmert19a} established that the $1/2$-Tsallis-INF algorithm (Algorithm~\ref{alg: FTRL}), namely FTRL with $\Psi(p)=-4\sum_{i=1}^d \sqrt{p_i}$, satisfies the BOBW guarantee, and further provided heuristic evidence suggesting that alternative regularizers fail to achieve this property. 
Also, the BOBW property enjoyed by 
\(1/2\)-Tsallis-INF and its variants is fairly general. It plays an important role not only in the MAB setting considered in this paper, but also in problems such as linear bandits \citep{zhao2025heavytailedlinearbanditsadversarial}, semi-bandits \citep{zimmert2019beating}, and MDPs \citep{chen2026bestofbothworldsheavytailedmarkovdecision}.
% do so.
% without prior knowledge of each gap $\Delta_i$.
In this paper, we focus on $1/2$-Tsallis-INF. By Lemma \ref{pti-2}, the solution of \eq{ptdefi} takes the form
\begin{equation}\label{pti}
    p_{t,i}=4\left(\eta_t\underline{\hat{L}}_{t, i}+2p_{t,i_{t}^*}^{-\frac{1}{2}}\right)^{-2}.
\end{equation}

\begin{algorithm}[htbp]
\caption{$1/2$-Tsallis-INF}
\label{alg: FTRL}
\begin{algorithmic}
   \STATE {\bfseries Input:} $0<\alpha < 1$
   \STATE {\bfseries Initialize:} $\hat{L}_1 = \mathbf{0}, \eta_t = \alpha/\sqrt{t}$
   \FOR{$t=1, 2, \dots$}
   \STATE choose $p_t=\argmin\limits_{p\in\Delta^{d-1}} \; \frac{\alpha}{\sqrt{t}}\agp{p,\hat{L}_t}-4\sum_{i=1}^d \sqrt{p_i}$\;
   \STATE play $I_t\sim p_t$,
   % \;
   % \STATE 
   and observe $\ell_{t,I_t}$\;
   \STATE construct estimator $\hat{\ell}_{t,i}=\frac{\ell_{t,i}\1_{\{I_t=i\}}}{p_{t,i}}$,
   % \;
   % \STATE 
   and update $\hat L_{t+1} = \hat L_{t}+\hat\ell_t$\;
   \ENDFOR
\end{algorithmic}
\end{algorithm}

\section{Main Results}
\label{sec:main}
%这一块把g_t先写出来

In this section, 
we present our main theoretical results. We derive both upper and lower bounds on the convergence rate of \(\Err_t\) for the \(1/2\)-Tsallis-INF algorithm.
\subsection{Upper Bound}
\label{sec:main-upper}
We begin with the following upper bound:
\begin{Thm}\label{et}
    For Algorithm \ref{alg: FTRL}, if $i_*$ is unique and $\alpha\in(0,1)$, then for any $t\ge 1$ and $q\in(0,2-\frac{\alpha^2 \mu_{i_*}}{4})$, there exists $C_{\alpha,q,\rho}>0$ depending only on $\alpha$, $q$ and $\rho:=2-\frac{\alpha^2 \mu_{i_*}}{4}-q$ such that
    \[
    \Err_t\leq C_{\alpha,q,\rho}\frac{d^{3q}(\log(d))^{2q}}{\Delta^{4q}}t^{-q}.
    \]
\end{Thm}

This result shows that, when the learning rate is chosen as \(\alpha/\sqrt{t}\), the failure probability of the importance-weighted estimator of \(1/2\)-Tsallis-INF to identify the optimal arm decays at rate \(t^{-2+\alpha^2\mu_{i_*}/4+\rho}\), where \(\rho>0\) can be taken arbitrarily small. 

% As \(\alpha\) decreases to $0$, this rate approaches \(t^{-2}\), matching the lower bound Theorem \ref{lower} we establish later. This shows that, in this sense, our result is tight. 

The dependence of the polynomial exponent on the learning-rate coefficient \(\alpha\) may seem counterintuitive at first. 
However, it is consistent with the quantitative trade-off between regret minimization and best-arm identification. 
Informally, classical lower bounds relating cumulative regret and simple regret imply that if a fixed policy has regret of order \(C_\pi\log t\), then its BAI-type failure probability cannot decay faster than $t^{-\Theta(C_{\pi})}$; see, e.g., Exercise 33.5(3) in \citet{lattimore2020bandit}. 
Thus, constants that are often secondary in logarithmic regret bounds become leading-order quantities after exponentiation. 
For \(1/2\)-Tsallis-INF, this constant becomes larger as \(\alpha\) decreases \citep{zimmert2021tsallis}. 
This is compatible with our result: as \(\alpha\) decreases to $0$, this rate approaches \(t^{-2}\), 
%matching
which nearly matches
the lower bound Theorem \ref{lower} we establish later. 
% In this sense, our result is tight.
% Hence, the dependence on \(\alpha\) should be viewed as a manifestation of the quantitative RM--BAI trade-off rather than merely as a proof artifact.

% In particular, when \(\mu_{i_*}=0\), that is, when the loss of the optimal arm is always $0$, the decay rate improves to \(t^{-2+\rho}\), independently of the learning-rate parameter. It's worth noting that though this case is simpler, it's still nontrivial because losses of other arms may also be $0$.

To the best of our knowledge, this is the first result that directly analyzes the probabilistic behavior of the unmodified importance-weighted estimator. Such a result has long been considered difficult to obtain, precisely because although the importance-weighted estimator \(\hat{L}_{t,i}\) is unbiased for the cumulative loss, its variance is too large for standard concentration inequalities or even the law of large numbers to be applicable. More specifically, if FTRL achieves the optimal logarithmic regret in stochastic bandits, then one would intuitively expect the probability \(p_{t,i}\) of selecting a suboptimal arm \(i\) at round \(t\) to decay at rate \(t^{-1}\). However, the importance-weighted estimator satisfies
\(
\hat{\ell}_{t,i}=\frac{\ell_{t,i}A_{t,i}}{p_{t,i}},
\)
which implies that its variance is of order \(t\). Consequently, the standard deviation of
\(
\hat{L}_{t,i}=\sum_{s=1}^{t-1}\hat{\ell}_{s,i}
\)
is also of order \(t\), already on the same scale as its expectation.

To circumvent this difficulty, we design a Lyapunov function:
\begin{equation}\label{gt2}
    g_t(y):=
\begin{cases}
\Delta^2 t^q e^{\lambda y}, & y<0,\\[4pt]
t^{-\theta}(y+\Delta t)^2, & y\ge 0,
\end{cases}
\qquad
q=2-\theta,
\end{equation}
where \(\lambda > 0\) is a parameter that will be determined in the subsequent proof. Then we show that:
\begin{Thm}\label{lya}
    Under the same conditions as in Theorem~\ref{et}, there exist $\lambda>0$ and constants $C_{\alpha,q,\rho},C'_{\alpha,q,\rho}>0$ depending only on $\alpha$, $q$ and $\rho$.
    Let $N=\left\lfloor C'_{\alpha,q,\rho}d^3(\log(d))^2/\Delta^4\right\rfloor+1$. Then, for every $t\ge N$,
    \[
    \E\left[g_t(D_t)
    %\1_{\{D_t\ge-\sqrt{t}\}}
    \right]\leq C_{\alpha,q,\rho}\frac{d^{3q}(\log(d))^{2q}}{\Delta^{4q-2}},
    \]
    where $D_t=\hat{L}_{t,i_*}-N_t$ and $N_t=\min _{i\in\mathcal{I},i\ne i_*} \hat{L}_{t,i}$.
\end{Thm}
We sketch its proof in Section \ref{pfsklya} and its full proof is deferred to Appendix \ref{pflya}.
For $t\ge N$, Theorem~\ref{et} follows from Theorem~\ref{lya}; for $1\le t<N$, we use the trivial bound $\Err_t\le1$. To see this, note that $g_t(y)\ge \Delta^2 t^{q}$ for any $y\ge 0$. Then, for $t\ge N$, Markov's inequality gives
\begin{equation}\label{etgt}
    \Err_t=\P(i_t^*\neq i_*)\leq \P(D_t\ge 0)\leq \Delta^{-2}t^{-q}\E\left[g_t(D_t)
    % \1_{\{D_t\ge-\sqrt{t}\}}
    \right].
\end{equation}%这个结果也是改进，待补充
For $1\le t<N$, the trivial estimate $\Err_t\le1\le(N/t)^q$ applies. Since $N^q\le C_{\alpha,q,\rho}\allowbreak d^{3q}\allowbreak(\log(d))^{2q}/\Delta^{4q}$, combining this estimate with the bound for $t\ge N$ proves Theorem~\ref{et} for every $t\ge1$.
Therefore, \(\Err_t\) decays at least at the rate \(t^{-q}\).

We will explain in detail how we construct \(g_t\) in Section~\ref{construct}. There, we study two simplified models, one of which approximates \(D_t\) by the trajectory of a continuous-time diffusion process. By applying It\^o's formula, we derive the PDE that \(g_t\) should satisfy, which in turn leads us to its explicit form. Judging from the final result, this simplified model captures an essential part of the problem. We also believe that such simplified models would be useful for studying other properties of FTRL algorithms in bandit settings.

\subsection{Lower Bound}
\label{sec:main-lower}
We next state a lower bound:
\begin{Thm}\label{lower}
    For Algorithm \ref{alg: FTRL}, if $i_*$ is unique, $\mu_{i_*}>0$ and $0<\alpha<1$, then for any $\varepsilon>0$, there exist $C $ and $C'>0$ that are independent of $t$ such that for any $t\ge C$, we have
    \[
    \Err_t\ge C' t^{-2-\varepsilon}.
    \]
\end{Thm}
We leave its proof in Appendix \ref{pflower}. Here, we do not spell out the explicit form of \(C,C'\), because their dependence on parameters such as \(d\), \(\Delta\), and \(\varepsilon\) is too complicated. This lower bound shows that, for the \(1/2\)-Tsallis-INF algorithm, \(\Err_t\) cannot decay polynomially at a rate faster than \(t^{-2}\), which is nearly tight by letting $\alpha\to 0$ in Theorem \ref{et}.
% matches the upper bound in Theorem~\ref{et}.

\begin{Rema}
    % Although our results focus on the \(1/2\)-Tsallis-INF algorithm, our proofs---for both the upper and lower bounds---can be extended rather straightforwardly to the general \(\theta\)-Tsallis-INF algorithm.
    Although our results focus on the \(1/2\)-Tsallis-INF algorithm, our proofs of both the upper and lower bounds suggest that analogous extensions to the general \(\theta\)-Tsallis-INF algorithm should be possible.
\end{Rema}

\section{Construction of the Lyapunov Function}\label{construct}
In this section, we explain how to derive the Lyapunov function $g_t$ in \eq{gt2}, which is used to establish the upper bounds in Section~\ref{sec:main-upper}, through two simplified toy models in the two-arm setting.
In this simplified setting, there is only one non-optimal arm, and \(D_t\) coincides with the difference between the cumulative losses of the optimal and non-optimal arms.
This makes the behavior of \(D_t\) easier to analyze.

% More precisely, 
According to \eq{etgt}, to obtain the final result in Theorem~\ref{et}, we need two ingredients: the expectation of \(g_t(D_t)\) should be controlled, and \(g_t(y)\) should grow sufficiently fast in \(t\) on the positive half-line. 
% Hence, In other word, 
In particular, we aim to find the nonnegative function \(g_t(y)\) satisfying the following two properties:
\begin{enumerate}[label=(P\arabic*), ref=P\arabic*]
    \item \label{item:lyapunov-growth}
    there exists some \(q>0\) such that, uniformly for all \(y\ge 0\),
    \(
    g_t(y)=\Omega(t^q);
    \)
    \item \label{item:lyapunov-moment}
    for all \(t \ge 1\), \(\E[g_t(D_t)]\) is uniformly bounded, where
    \(
    D_t=\hat{L}_{t,i_*}-\min_{i\in\mathcal{I},\, i\neq i_*}\hat{L}_{t,i}.
    \)
\end{enumerate}
Because the quantity of interest is $D_t$, we focus on its stochastic dynamics. The two toy models introduced below are simplified versions of this dynamics.

\subsection{Toy Models}
To motivate the construction, we introduce two simplified models of $D_t$ in this subsection.

The key simplification is to focus on the importance-weighted estimator \(\hat L_{t,i_*}\) of the optimal arm, while suppressing the effects of the complicated dynamics of the other one. 
% To this end, We first consider the case where 
With \(d=2\) and \(i_*=1\), we have
\(
D_t=\hat L_{t,1}-\hat L_{t,2}.
\)
To remove the effects caused by arm $2$, we replace \(\hat L_{t,2}\) by its mean \((t-1)\mu_2\). 
After this replacement, we obtain a new process \(\tilde D_t\), whose dynamics can be expressed as
\begin{equation}\label{original}
\widetilde{D}_{t+1}=\widetilde{D}_t+\hat \ell_{t,1}-\mu_2
% \hat L_{t+1,1}=\hat L_{t,1}+\hat \ell_{t,1}
,\quad
\hat \ell_{t,1}=\frac{\ell_{t,1}\1_{\{I_t=1\}}}{p_{t,1}},\text{ and by \eq{pti}, }p_{t,1}=\Bigl(\frac{\eta_t \widetilde{D}_t^+}{2} + p_{t,i_t^*}^{-1/2}\Bigr)^{-2},
\end{equation}
% where we used $\widetilde{D}_t$ to distinguish it from the true $D_t$.
% Since \(D_t=\hat L_{t,1}-\hat L_{t,2}\), a natural first approximation is to replace \(\hat L_{t,2}\) by its mean \((t-1)\mu_2\), so that \(D_t\) is approximated by \(\hat L_{t,1}-(t-1)\mu_2\). This simplification deliberately removes the effects caused by the complicated dynamics of arm \(2\), allowing us to focus on the behavior of arm \(1\) and leading to the following toy models of $D_t$.
Based on this preliminary simplification, 
% in \eq{original}, 
we further propose two toy models:
\paragraph{Toy model 1.}
We further simplify the update rule \eq{original} of \(\widetilde D_t\) so that the resulting process has a cleaner Markovian structure. 
First, we replace the term \(p_{t,i_t^*}^{-1/2}\) appearing in \(p_{t,1}\), which is not determined by \(\widetilde D_t\), by \(1\). 
This is justified since it is of constant order, as \(1/2\le p_{t,i_t^*}\le 1\). 
Second, since the randomness of \(\ell_{t,1}\) is also not determined by \(\widetilde D_t\), we replace it by its mean \(\mu_1\). Therefore, we first consider the discrete-time process
\begin{equation}\label{toy1}
    Y_{t+1}=Y_t-\mu_2+\mu_1\frac{\1_{\{A_{t+1}=1\}}}{p_t}=Y_t-\Delta+\mu_1\left(\frac{\1_{\{A_{t+1}=1\}}}{p_t}-1\right),
\end{equation}
where $\Delta:=\mu_2-\mu_1$ and
\(
p_t=\left(1+\frac{\alpha Y_t^+}{2\sqrt t}\right)^{-2},
\mathbb P(A_{t+1}=1\mid \mathscr{F}_t)=p_t.
\)
% Compared to \eq{original}, first, we  replace the factor \(p_{t,i_t^*}^{-1/2}\) in the exact expression of \(p_{t,1}\) by $1$, since it is of constant order by noting that $\frac{1}{2}\leq p_{t,i_t^*}\leq 1$. Then, we replace the random loss $\ell_{t,1}$ by its mean \(\mu_1\).

% This model is obtained from the dynamics of \(D_t\) by the following simplifications: 
% \begin{enumerate}
%     \item we replace the factor \(p_{t,i_t^*}^{-1/2}\) in the exact expression of \(p_{t,1}\) by a constant, since it is of constant order;
%     % we replace the exact expression of \(p_{t,1}\) by the simpler form \(\bigl(1+\alpha Y_t^+/2\sqrt t\bigr)^{-2}\), retaining only its leading dependence on the positive part of the state;
%     % \item we replace the factor \(p_{t,i_t^*}^{-1/2}\) by a constant, since it is of constant order;
%     \item we replace the random loss $\ell_{t,1}$ by its mean \(\mu_1\).
% \end{enumerate}

\paragraph{Toy model 2.}
Recall that Property~(\ref{item:lyapunov-moment}) requires \(g_t(D_t)\) to be a supermartingale, and It\^o's formula is a standard tool for analyzing such a property.
To apply It\^o's formula, we further pass from Toy model~1 to a continuous-time approximation. Indeed, in \eq{toy1}, let $\zeta_{t+1}:=\mu_1\left(\frac{\1_{\{A_{t+1}=1\}}}{p_t}-1\right)$, then one can see that
\(
\mathbb E[\zeta_{t+1}\mid \mathscr{F}_t]=0, \mathrm{Var}(\zeta_{t+1}\mid \mathscr{F}_t)
=\mu_1^2\frac{1-p_t}{p_t}
\lesssim
\left(1+\frac{\alpha Y_t^+}{2\sqrt t}\right)^2.
\)
Replacing the discrete martingale increment \(\zeta_{t+1}\) by a Gaussian fluctuation with the same drift and variance of the same order, and then passing to a continuous-time approximation, we are led to the continuous-time diffusion
\begin{equation}\label{toy2}
    \d Y_t=-\Delta\,\d t+\left(1+\frac{\alpha Y_t^+}{2\sqrt t}\right)\d B_t,
\text{ where 
$\Delta=\mu_2-\mu_1$ and $B$ is a Brownian motion.}
\end{equation}

\subsection{Explicit Form of the Lyapunov Function}
In this subsection, we study Toy model 2 in order to identify the correct form of the Lyapunov function. We seek a nonnegative family \(g_t\) satisfying Property~(\ref{item:lyapunov-growth}) such that \(g_t(Y_t)\) is a supermartingale, which would in turn imply Property~(\ref{item:lyapunov-moment}).

If \(f(t,Y_t)\) is a supermartingale, then It\^o's formula suggests that \(f\) should satisfy
\begin{equation}\label{eq:toy2-pde}
\partial_t f(t,y)-\Delta\,\partial_y f(t,y)
+\frac12\left(1+\frac{\alpha y^+}{2\sqrt t}\right)^2 \partial^2_{yy}f(t,y)
\le 0.
\end{equation}
We now remove the first-order term by introducing
\(
x:=y+\Delta t\), and $h(t,x)=f(t,y)$. Then \eq{eq:toy2-pde} becomes
% \begin{equation}\label{eq:toy2-pde-h}
% \partial_t h(t,x)
% +\frac12\left(1+\frac{\alpha (x-\Delta t)^+}{2\sqrt t}\right)^2 \partial_{xx}h(t,x)
% \le 0.
% \end{equation}
\[
\partial_t h(t,x)
+\frac12\left(1+\frac{\alpha (x-\Delta t)^+}{2\sqrt t}\right)^2 \partial^2_{xx}h(t,x)
\le 0.
\]
A standard idea in solving or approximating parabolic equations is to separating time and space. Hence, we now assume that
% look for solutions in separated form, namely, as a product of a function of time and a function of space. Motivated by this classical separation-of-variables ansatz, we now assume that
\(
h(t,x)=u(t)v(x).
\)
% Although we do not expect to solve the above inequality exactly, this ansatz is very useful for identifying the correct scaling of the Lyapunov function.
Then, we have
\begin{equation}\label{sep}
    u'(t)v(x)
+\frac12\left(1+\frac{\alpha (x-\Delta t)^+}{2\sqrt t}\right)^2 u(t)v''(x)
\le 0.
\end{equation}
% We first consider the positive half-line.
Next, we distinguish the cases according to the sign of \(x-\Delta t\). The following calculation is
heuristic and only motivates the Lyapunov choice.

\paragraph{The positive half-line.}
When \(x\ge \Delta t\), the coefficient in \eq{sep} is
\(
\left(1+\frac{\alpha(x-\Delta t)}{2\sqrt t}\right)^2.
\)
For large positive \(x\), its dominant part is of order \(x^2/t\). Accordingly, 
%the leading-order equation
\eq{sep} becomes
% \[
% \partial_t h(t,x)+\frac{\alpha^2 x^2}{8t}\,\partial_{xx}h(t,x)\lesssim 0.
% \]
% Substituting \(h(t,x)=u(t)v(x)\) yields
% \[
% u'(t)v(x)+\frac{\alpha^2 x^2}{8t}u(t)v''(x)\lesssim 0,
% \]
% or equivalently,
\[
\frac{t\,u'(t)}{u(t)}+\frac{\alpha^2}{8}\frac{x^2v''(x)}{v(x)}\lesssim 0.
\]
Since the left-hand side is a sum of a function of \(t\) and a function of \(x\), we are naturally led to require
\(
t\,u'(t)=-\theta u(t)\), and \(
\frac{\alpha^2}{8}x^2v''(x)\leq  \theta v(x).
\)
The first relation gives
\(
u(t)=t^{-\theta}.
\)
The second one's 
%is an Euler-type inequality,whose 
natural power-law solutions are \(v(x)=x^m\), where $\theta\ge \frac{\alpha^2m(m-1)}{8}.$ Therefore, on the positive half-line, the separated ansatz leads to, for \(
\theta\ge \frac{\alpha^2m(m-1)}{8},
\)
\[
h_+(t,x)=t^{-\theta}x^m,
\qquad\text{equivalently}\qquad
f_+(t,y)=t^{-\theta}(y+\Delta t)^m.
\]

\paragraph{The negative half-line.}
We proceed in the same way. When \(x<\Delta t\), we have \((x-\Delta t)^+=0\), and \eq{sep} reduces to
% \[
% \partial_t h(t,x)+\frac12\,\partial_{xx}h(t,x)\le 0.
% \]
% Substituting again \(h(t,x)=u(t)v(x)\), we obtain
% \begin{equation}\label{uv2}
%     u'(t)v(x)+\frac12 u(t)v''(x)\le 0,
% \end{equation}
% that is,
\begin{equation}\label{uv2}
    \frac{u'(t)}{u(t)}+\frac12\frac{v''(x)}{v(x)}\le 0.
\end{equation}
As in the usual separation-of-variables argument, we first identify the space-dependent factor by taking
\(
\frac12\frac{v''(x)}{v(x)}=\kappa
\)
for some \(\kappa\in\R\). Since we require \(g_t>0\) from the outset, \(v\) should not be oscillatory; hence we take \(\kappa>0\). 
Moreover, recalling \eq{etgt}, \(g_t(D_t)\) should not be too far from \(\1_{\{D_t\ge0\}}\) heuristically. 
Thus, when \(y=x-\Delta t\) is small, \(v(x)\) should not be too large, and we do not use the solution \(v(x)=e^{-\lambda x}\), where \(\lambda=\sqrt{2\kappa}\). 
We therefore take
\(
    v(x)=e^{\lambda x}.
\)

% This gives
% \(
% v''(x)=2\kappa v(x),
% \)
% and one of its exponential solution is
% \(
% v(x)=e^{\lambda x}\), where \(\lambda=\sqrt{2\kappa}.
% \) 
Recall that \(y=x-\Delta t\). Then \(v(x)=e^{\lambda(y+\Delta t)}\). 
If we still chose \(u(t)\) similar to the positive half-line case, then \(f(t,y)\) would fail to have the desired smoothness in \(t\) at the interface \(y=0\). 
To match both the value of the positive branch and its \(t\)-scaling at the interface \(y=0\), and hence to ensure the smoothness of \(f\) across the interface, we therefore take
\[
    f_-(t,y)=\Delta^m t^q e^{\lambda y},
    \qquad q=m-\theta.
\]
Moreover, by choosing \(\lambda>0\) sufficiently small, this branch also satisfies the required condition \eq{uv2}. 
Since this observation is not closely related to the precise choice of \(\lambda\) in the formal proof, we omit the details here.
% Returning to the original variable \(y=x-\Delta t\), this corresponds to an exponential profile in \(y\). To match both the value of the positive branch and its \(t\)-scaling at the interface \(y=0\), and hence to ensure the smoothness of \(f\) across the interface, we therefore take
% \[
% f_-(t,y)=\Delta^m t^q e^{\lambda y},
% \qquad q=m-\theta.
% \]
% Moreover, by choosing \(\lambda>0\) sufficiently small, it also satisfies the required \eq{uv2}. 
Therefore, Toy model 2 suggests defining the Lyapunov ansatz by
\[
g_t(y):=
\begin{cases}
\Delta^m t^q e^{\lambda y}, & y<0,\\[4pt]
t^{-\theta}(y+\Delta t)^m, & y\ge 0,
\end{cases}
\qquad
\theta\ge \frac{\alpha^2m(m-1)}{8},
\qquad
q=m-\theta.
\]
In particular, such a choice of \(g_t\) satisfies Property~(\ref{item:lyapunov-growth}). It remains to explain why $m$ will be fixed to $2$ and we elaborate this point in the following subsection.

\subsection{\texorpdfstring{Choice of $m$}{cm}}
In this subsection, we explain why \(m\) will be chosen to be \(2\). 

We first show that \(m\) cannot be larger than \(2\). This obstruction is not visible from Toy model 2, since that model oversimplifies increment $\zeta_{t+1}$ in Toy model 1 by retaining only its first two moments and replacing it with a Gaussian variable. 
Informally, one can see for any $m>2$, we have $\mathbb E[\zeta_{t+1}^m\mid \mathscr{F}_t]
\sim (\mathbb E[\zeta_{t+1}^2\mid \mathscr{F}_t])^{m-1},$ which is very different from those of a Gaussian variable. As we now show, when \(m>2\), the ratio \(\E_t[g_{t+1}(Y_{t+1})]/g_t(Y_t)\) can become arbitrarily large, so \(g_t(Y_t)\) cannot be a supermartingale.

In Toy model 1, by \eq{toy1}, for any \(y>0\), on the event \(\{Y_t=y\}\),
\[
\frac{\E_t[g_{t+1}(Y_{t+1})]}{g_t(y)}
\ge
p_t\,
\frac{(t+1)^{-\theta}\bigl(y-\mu_2+\mu_1/p_t+\Delta(t+1)\bigr)^m}
{t^{-\theta}(y+\Delta t)^m}.
\]

Now fix \(t\) and let $\frac{y}{t}$ be large  enough, which implies that
\(
p_t=\left(1+\frac{\alpha y}{2\sqrt t}\right)^{-2}
\sim \frac{t}{\alpha^2 y^2},
\)
then we have
\(
\qquad y-\mu_2+\frac{\mu_1}{p_t}+\Delta(t+1)\gtrsim \frac{\mu_1}{p_t}.
\)
% \[
% y-\mu_2+\frac{\mu_1}{p_t}+\Delta(t+1)\sim \frac{\mu_1}{p_t}.
% \]
Hence
\(
\frac{\E_t[g_{t+1}(Y_{t+1})]}{g_t(y)}
\gtrsim
p_t\,
\frac{(\mu_1/p_t)^m}{y^m}
=
\mu_1^m\, p_t^{\,1-m} y^{-m}\sim \frac{\alpha^{2m-2}}{t^{m-1}}\, y^{m-2}.
\)
% Using \(p_t\sim t/(\alpha^2 y^2)\), we further obtain
% \[
% p_t^{\,1-m} y^{-m}
% \sim
% \left(\frac{\alpha^2 y^2}{t}\right)^{m-1} y^{-m}
% =
% \frac{\alpha^{2m-2}}{t^{m-1}}\, y^{m-2}.
% \]
Therefore,
\[
\frac{\E_t[g_{t+1}(Y_{t+1})]}{g_t(y)}
\to+\infty
\qquad\text{as }y\to+\infty
\]
for every \(m>2\). In particular, for any fixed \(t\), this ratio can be made arbitrarily large by taking \(Y_t\) sufficiently large. This shows that \(g_t(Y_t)\) cannot be a supermartingale when \(m>2\).

Next, we explain why we do not choose \(m<2\). On the one hand, the choice \(m=2\) leads to a particularly simple form of \(g_t\), which makes the subsequent analysis more tractable. On the other hand, recall that Property~(\ref{item:lyapunov-growth}) yields a convergence rate of order \(t^{-q}\) for \(\Err_t\), so it is desirable to make \(q\) as large as possible. From the previous discussion, we have
\(
q\le m-\theta \le m-\frac{\alpha^2 m(m-1)}{8}.
\)
For fixed \(\alpha>0\), the right-hand side is a quadratic function of \(m\), and it attains its maximum at
\(
m^*=\frac12+\frac{4}{\alpha^2}.
\)
Therefore, when \(\alpha^2<\frac{8}{3}\), this maximizer lies to the right of \(2\). Equivalently, the function
$
m\mapsto m-\frac{\alpha^2 m(m-1)}{8}
$
is strictly increasing on \((-\infty,2]\). In particular, among all admissible choices \(m\le 2\), the choice \(m=2\) yields the largest possible value of \(q\). For this reason, together with the fact that \(m>2\) is impossible, we choose $m=2$.

\section{Proof Sketch of
%Theorem \ref{lya}
the Upper Bound
}\label{pfsklya}
In this section, we provide a proof sketch of Theorem~\ref{lya}, whose full proof can be found in Appendix~\ref{pflya}. 

As mentioned in the toy models in Section~\ref{construct}, one can prove through careful calculations that \(g_t(Y_t)\) is a supermartingale. However, in the actual bandit model, the problem becomes much more complicated. 

The main difference is that, in the toy models, we only considered the case \(d=2\) and $i_*=1$, and directly replaced the second term in
\(
D_t:=\hat L_{t,i_*}-\min_{i\in\mathcal I,\; i\neq i_*}\hat L_{t,i},
\)
namely
\(
N_t=\min_{i\in\mathcal I,\; i\neq i_*}\hat L_{t,i},
\)
by \((t-1)\mu_2\). This approximation is, however, too crude, for two main reasons. 

First, when \(d>2\), \(N_t\) does not grow steadily at every round. Indeed, if there exist two distinct arms \(i,i'\neq i_*\) such that
\(
N_t=\hat L_{t,i}=\hat L_{t,i'},
\)
then, since the player selects only one arm in each round, at most one of \(\hat L_{t,i}\) and \(\hat L_{t,i'}\) can increase at that round. As a result, one necessarily has
\(
N_{t+1}=N_t.
\)

Second, when \(D_t<0\) with large magnitude, \(p_{t,i}\) is very small for every \(i\neq i_*\), then the variance of the importance-weighted estimator \(\hat L_{t,i}\) becomes too large, and consequently its deviation from its mean can no longer be controlled. In this regime, it is also difficult to show that \(N_t\) grows linearly. 

To overcome these two difficulties, we introduce the following result:

\begin{Lem}\label{ntgrow}
     There exists $C>0$ such that for any integer $L\ge 1$, if $4dL\le \sqrt{dt}/(4\alpha)$ and $D_t\ge-\sqrt{dt}/\alpha$, then
    \[
    \mathbb{E}\left[N_{t+L}-N_t \,\middle|\, \mathscr{F}_{t-1}\right]
    \geq\left(\mu_{i_*}+\Delta\right)L-C\left(\sqrt{4dL\log(d)}+\log(d)\right).
    \]
\end{Lem}

This lemma shows that, when \(D_t\) is not too negative, we can guarantee that \(N_t\) grows by a constant amount after sufficiently many rounds. The term $C\left(\sqrt{4dL\log(d)}+\log(d)\right)$ is $o(L)$ once $L$ is chosen of order $d\log(d)/\Delta^2$. Its proof is deferred in Appendix \ref{growthofnt}, where we will also introduce some other results describing the growth of $N_t$.

Therefore, among the two difficulties discussed above, the first can be handled by choosing \(L\) sufficiently large, while the second can be circumvented by distinguishing according to the comparison between $D_t$ and $-\sqrt{dt}/\alpha$.
If $D_t\ge -\sqrt{dt}/\alpha$, we can show that for any sufficiently large \(L\), and sufficiently small $\lambda$ (appearing in the definition of $g_t$ in \eq{gt2}),  once \(t\) exceeds some sufficiently large threshold \(N\),
\begin{equation}\label{lsuper}
    \E\!\left[g_{t+L}(D_{t+L})\,\middle|\,\mathscr{F}_{t-1}\right]\le g_t(D_t)
\qquad\text{on the event }\{D_t\ge -\sqrt{dt}/\alpha\}.
\end{equation}
A formal version of \eq{lsuper} can be found in Lemma \ref{formalsuper}. Otherwise, $D_t<-\sqrt{dt}/\alpha$, then fortunately, $g_t(D_t)$ itself is already small. One can then show that $\E\left[g_t(D_t)\right]$ is bounded by combining these two cases as described in Theorem \ref{lya}. All of the details can be found in Appendix \ref{pflya}.
% However, unlike in the toy model, our result is only available when \(D_t\ge -\sqrt{t}\). As a consequence, we can only show that for any sufficiently large \(L\), and sufficiently small $\lambda$ (required in the definition of $g_t$ in \eq{gt2}),  once \(t\) exceeds some sufficiently large threshold \(N\),
% \begin{equation}\label{lsuper}
%     \E\!\left[g_{t+L}(D_{t+L})\,\middle|\,\mathscr{F}_{t-1}\right]\le g_t(D_t)
% \qquad\text{on the event }\{D_t\ge -\sqrt{t}\}.
% \end{equation}

% outside the event \(\{D_t\ge -\sqrt{t}\}\), the monotonicity of \(g_t\) implies that
% \[
% g_t(D_t)\le g_t(-\sqrt{t}),
% \]
% which is already small enough. Therefore, \eq{lsuper} is already sufficient to imply that \(\E[g_t(D_t)]\) remains bounded in \(t\). 
% All of the details can be found in Appendix \ref{pflya}.
% More specifically, one can show that for every \(t\ge 1\),
% \[
% \E[g_t(D_t)]
% \leq
% \sup_{0\leq r\leq L-1}\sum_{j=1}^{\infty} g_{N+r+jL}(-\sqrt{N+r+jL})
% +\sup_{s\le N+L}\E[g_s(D_s)].
% \]%待检查
% A more careful estimate of the right-hand side then yields the final result.

\section{Proof Sketch of
%Theorem \ref{lower}
the Lower Bound
}\label{pfsklower}
In this section, we sketch the proof of Theorem \ref{lower}. The proof is based on the following lemma, whose proof is deferred to Appendix \ref{pfbase}.
\begin{Lem}\label{base}
    % If $0<\alpha<1$ and $\mu_{i_*}>0$, then for any $B>\frac{Cd}{\mu_{i_*}\alpha^2}$ and $t\ge Cd^2$ where $C>0$ are constants, let $v_B=\frac{\mu_{i_*}\alpha^2 B^2}{16}$ and $V_B=4\alpha^2B^2$. If $v_B\leq \frac{D_t}{t}\leq V_B$, then there exists a \((\mathscr F_t)_{t\ge 1}\)-stopping time $\tau(t)$ such that 
    Assume that \(0<\alpha<1\),  \(\mu_{i_*}>0\), 
and \(C>0\) is some universal constant.
%There exists a universal constant \(C>0\) such that the following holds. 
For any
\(B>\frac{Cd}{\mu_{i_*}\alpha^2}\), define \(v_B:=\frac{\mu_{i_*}\alpha^2 B^2}{16}\),  \(V_B:=4\alpha^2 B^2\). 
% If \( v_B\le \frac{D_t}{t}\le V_B\), 
Then for any \(    t\ge Cd^2\),
there exists an \((\mathscr F_s)_{s\ge 1}\)-stopping time \(\tau(t)\) satisfying the following properties. 
    \begin{enumerate}[(1)]
        \item $\tau(t)+1 \ge a_B t$, where $a_B= \frac{\mu_{i_*}\alpha^2 B}{48}>1$;
        \item On the event $\left\{\tau(t)<+\infty\right\}$, we have (i) $v_B\leq \frac{D_{t}}{t}\leq V_B$, (ii) for any $t\leq t'\leq \tau(t)$, $D_{t'}>0$, and (iii) $v_B\leq \frac{D_{\tau(t)+1}}{\tau(t)+1}\leq V_B$;
        \item $ \P\left(\tau(t)<+\infty\,\mid\,\mathscr{F}_{t-1}\right)  
        % \1_{ \{ v_B\le \frac{D_t}{t}\le V_B \} } 
        \ge q_B\1_{ \left\{ v_B\le \frac{D_t}{t}\le V_B \right\} }$, where $q_B= C'\frac{\mu_{i_*}}{\alpha^2 B^2}\in(0,1)$ and $C'>0$ is a constant.
    \end{enumerate}
\end{Lem}
The first property characterizes the growth rate of the stopping time \(\tau(t)\) as a function of \(t\).
The second property ensures that, once \(\{\tau(t)<+\infty\}\) occurs, both \(\frac{D_t}{t}\) and \(\frac{D_{\tau(t)+1}}{\tau(t)+1}\) lie in the interval \([v_B,V_B]\), and \(D_t\) remains positive throughout the intervening period.
The last property shows that, as long as \(\frac{D_t}{t}\in [v_B,V_B]\), the stopping time \(\tau(t)\) is finite with positive probability.

With Lemma~\ref{base}, we could define a sequence of stopping times $(\tau_k)_{k\ge 0}$ iteratively.
Here, we focus only on the case where all these stopping times are finite; the reason for this will be explained below.
In this case, the definition can be informally written as
\[
\tau_0:=\inf\left\{t\ge 1\,:\, v_B\leq \frac{D_{t+1}}{t+1}\leq V_B\right\}, \tau_{k+1}=\tau(\tau_k+1).
\]
For the formal definition,  see \eq{taut}.
In particular, as long as we can find a time $t$ 
% (satisfying the conditions in Lemma~\ref{base}) 
satisfying \(\frac{D_{t+1}}{t+1} \in [v_B,V_B]\), which can be guaranteed by Lemma~\ref{start},
then we can set $\tau_0$ as $t-1$ and then define $\tau_k$ iteratively. By Lemma~\ref{base}, we can obtain \(\P(\tau_k<+\infty)\gtrsim q_B^k\).

% In particular, as long as we can find a time $t$ 
% % (satisfying the conditions in Lemma~\ref{base}) 
% satisfying \(\frac{D_{t+1}}{t+1} \in [v_B,V_B]\), which can be guaranteed by Lemma~\ref{start},
% then we can set $\tau_0$ as $t-1$ and then define $\tau_1 = \tau(\tau_0+1)$.
% By Lemma~\ref{base}~(3), $\P(\tau_1 < \infty) \ge q_B$.
% Then by Lemma~\ref{base}~(2), $\P\big(v_B \le \frac{D_{\tau_1+1}}{\tau_1+1} \le V_B\big) \ge \P(\tau_1 < +\infty) \ge q_B$.
% Similarly, from \(\tau_1\) to \(\tau_2\), we have \(\P\big(v_B \le \frac{D_{\tau_2+1}}{\tau_2+1} \le V_B \big)\ge \P(\tau_2<+\infty)\ge q_B \P\big(v_B \le \frac{D_{\tau_1+1}}{\tau_1+1} \le V_B\big)\ge q_B^2\).
% Repeating this argument yields \(\P(\tau_k<+\infty)\ge q_B^k\).

We now explain why the event \(\{\tau_k<+\infty\}\) is important.
On this event, Lemma~\ref{base}~(1) implies that \(\tau_k\ge a_B(\tau_{k-1}+1)-1 \ge a_B\tau_{k-1}\ge \cdots \ge a_B^k\tau_0\).
Moreover, by Lemma~\ref{base}~(2), for any \(\tau_0+1\le t \le \tau_k+1\), we have \(D_t>0\).
Combining these two facts, we obtain that \(D_t>0\) for every integer \(t\in[\tau_0+1,\,a_B^k\tau_0+1]\).

% In particular, since \(\tau_0+1\le \lceil a_B^k\tau_0\rceil\le \tau_k+1\), it follows that \(D_{\lceil a_B^k\tau_0\rceil}>0\).
By \eq{eq:etdt}, to derive a lower bound for $\Err_t$, it suffices to derive a lower bound for $\P(D_t>0)$. Setting \(k=\lceil\log_{a_B}(t/\tau_0)\rceil\) ensures that \(t\) belongs to the interval \([\tau_0+1,\,a_B^k\tau_0+1]\) described above.
% For sufficiently large $t$, letting $k = \lceil\log_{a_B}(t/\tau_0)\rceil $ yields $k \ge 2$ and $\tau_0 + 1 \le  \lceil a_B^{k-1} \tau_0 \rceil \le  t \le a_B^k \tau_0 \le \tau_k + 1$ and thus $D_t > 0$.
Therefore, 
\begin{equation}\label{intui}
    \Err_t\ge \P(D_t>0)\ge \P(\tau_{\lceil\log_{a_B}(t/\tau_0)\rceil}<+\infty)
    \gtrsim (q_B)^{\log_{a_B}(t/\tau_0)}  \gtrsim t^{\frac{\log(q_B)}{\log(a_B)}}.
\end{equation}
% where we used Lemma \ref{base} (3) in the third inequality.
Finally, it suffices to note that when $B\to+\infty$, $\frac{\log(q_B)}{\log(a_B)}\to -2$. A rigorous version of the argument above can be found in Appendix \ref{pflower}.

% This lemma shows that once \(D_t/t\) enters the interval \([v_B,V_B]\), then with probability at least \(q_B\), the ratio \(D_t/t\) will return to this interval after at least \((a_B-1) t\) time steps, while \(D_t\) remains nonnegative throughout the intervening period. 
% This provides the basis for the induction argument. Informally, according to Lemma \ref{base}, one can define a sequence of stopping times $(\tau_k)_{k\ge 0}$ by (see the formal definition in \eq{taut}):
% \[
% \tau_0:=\inf\left\{t\ge 1\,:\, v_B\leq \frac{D_{t+1}}{t+1}\leq V_B\right\}, \tau_{k+1}=\tau(\tau_k+1).
% \]
% By Lemma \ref{base} (1), one can see that $\tau_k\ge a_B^k$ and by (2), if $\tau_k<+\infty$, then for any $t\leq a_B^k$, we have $D_t>0$. Hence, by \eq{eq:etdt}, intuitively
% \begin{equation}\label{intui}
%     \Err_t\ge \P(D_t>0)\ge \P(\tau_{\lceil\log_{a_B}(t)\rceil}<+\infty)\gtrsim (q_B)^{\log_{a_B}(t)}=t^{\frac{\log(q_B)}{\log(a_B)}},
% \end{equation}
% where we used Lemma \ref{base} (3) in the third inequality. Finally, it suffices to note that when $B\to+\infty$, $\frac{\log(q_B)}{\log(a_B)}\to -2$. A rigorous version of the argument above can be found in Appendix \ref{pflower}.

\section{Concluding Remarks}\label{sec:cr}
We analyze the best-arm identification (BAI) probelm for the classical best-of-both-worlds algorithm $1/2$-Tsallis-INF \citep{zimmert2021tsallis}. 
We show that, when $\eta_t=\alpha/\sqrt{t}$, the probability that the empirical best arm differs from the true optimal arm decays as $t^{-2+\alpha^2\mu_{i_*}/4+\rho}$ for any $\rho>0$. 
To our knowledge, this provides the first BAI guarantee for an FTRL method without explicit exploration. 
We also establish a corresponding lower bound, indicating that the rate $t^{-2}$ is essentially tight. 
On the technical side, to overcome the difficulty caused by the high variance, we introduce two new toy models and a Lyapunov-function-based approach, which may be useful for analyzing $1/2$-Tsallis-INF in other BOBW regimes and more general Tsallis-entropy FTRL algorithms under importance-weighted bandit feedback.
Moreover, our lower bound does not capture the dependence on $\alpha$. 
An interesting direction for future work is to establish a fully matching lower bound.

\newpage
\bibliography{ref}

\begin{thebibliography}{24}
\providecommand{\natexlab}[1]{#1}
\providecommand{\url}[1]{\texttt{#1}}
\expandafter\ifx\csname urlstyle\endcsname\relax
  \providecommand{\doi}[1]{doi: #1}\else
  \providecommand{\doi}{doi: \begingroup \urlstyle{rm}\Url}\fi

\bibitem[Audibert et~al.(2010)Audibert, Bubeck, and Munos]{Audibert10}
Jean{-}Yves Audibert, S{\'{e}}bastien Bubeck, and R{\'{e}}mi Munos.
\newblock Best arm identification in multi-armed bandits.
\newblock In \emph{Conference on Learning Theory}, pages 41--53, 2010.
\newblock URL \url{http://colt2010.haifa.il.ibm.com/papers/COLT2010proceedings.pdf\#page=49}.

\bibitem[Auer et~al.(2002)Auer, Cesa-Bianchi, and Fischer]{auer2002finite}
Peter Auer, Nicolo Cesa-Bianchi, and Paul Fischer.
\newblock Finite-time analysis of the multiarmed bandit problem.
\newblock \emph{Machine learning}, 47\penalty0 (2):\penalty0 235--256, 2002.

\bibitem[Bandyopadhyay et~al.(2024)Bandyopadhyay, Juneja, and Agrawal]{DBLP:conf/nips/Bandyopadhyay0A24}
Agniv Bandyopadhyay, Sandeep Juneja, and Shubhada Agrawal.
\newblock Optimal top-two method for best arm identification and fluid analysis.
\newblock In Amir Globersons, Lester Mackey, Danielle Belgrave, Angela Fan, Ulrich Paquet, Jakub~M. Tomczak, and Cheng Zhang, editors, \emph{Advances in Neural Information Processing Systems 38: Annual Conference on Neural Information Processing Systems 2024, NeurIPS 2024, Vancouver, BC, Canada, December 10 - 15, 2024}, 2024.
\newblock URL \url{http://papers.nips.cc/paper\_files/paper/2024/hash/7aebeabf3c031ab5395db836e0b73473-Abstract-Conference.html}.

\bibitem[Bubeck and Slivkins(2012)]{pmlr-v23-bubeck12b}
Sébastien Bubeck and Aleksandrs Slivkins.
\newblock The best of both worlds: Stochastic and adversarial bandits.
\newblock In Shie Mannor, Nathan Srebro, and Robert~C. Williamson, editors, \emph{Proceedings of the 25th Annual Conference on Learning Theory}, volume~23 of \emph{Proceedings of Machine Learning Research}, pages 42.1--42.23, Edinburgh, Scotland, 25--27 Jun 2012. PMLR.

\bibitem[Chen et~al.(2026)Chen, Liu, Huang, Du, and Huang]{chen2026bestofbothworldsheavytailedmarkovdecision}
Yu~Chen, Yuhao Liu, Jiatai Huang, Yihan Du, and Longbo Huang.
\newblock Best-of-both-worlds for heavy-tailed markov decision processes, 2026.
\newblock URL \url{https://arxiv.org/abs/2602.01295}.

\bibitem[Even-Dar et~al.(2002)Even-Dar, Mannor, and Mansour]{even2002pac}
Eyal Even-Dar, Shie Mannor, and Yishay Mansour.
\newblock Pac bounds for multi-armed bandit and markov decision processes.
\newblock In \emph{International Conference on Computational Learning Theory}, pages 255--270. Springer, 2002.

\bibitem[Jamieson and Nowak(2014)]{jamiesonsurvey}
Kevin~G. Jamieson and Robert~D. Nowak.
\newblock Best-arm identification algorithms for multi-armed bandits in the fixed confidence setting.
\newblock In \emph{48th Annual Conference on Information Sciences and Systems, {CISS} 2014, Princeton, NJ, USA, March 19-21, 2014}, pages 1--6. {IEEE}, 2014.

\bibitem[Komiyama et~al.(2022)Komiyama, Tsuchiya, and Honda]{komiyama2022minimax}
Junpei Komiyama, Taira Tsuchiya, and Junya Honda.
\newblock Minimax optimal algorithms for fixed-budget best arm identification.
\newblock In \emph{Advances in Neural Information Processing Systems}, 2022.
\newblock URL \url{https://openreview.net/forum?id=TIQfmR7IF6H}.

\bibitem[Komiyama et~al.(2026)Komiyama, Jang, and Honda]{komiyama2026rateoptimaldesignanytimebest}
Junpei Komiyama, Kyoungseok Jang, and Junya Honda.
\newblock Rate-optimal design for anytime best arm identification, 2026.
\newblock URL \url{https://arxiv.org/abs/2510.23199}.

\bibitem[Lai and Robbins(1985)]{lai1985asymptotically}
Tze~Leung Lai and Herbert Robbins.
\newblock Asymptotically efficient adaptive allocation rules.
\newblock \emph{Advances in applied mathematics}, 6\penalty0 (1):\penalty0 4--22, 1985.

\bibitem[Lattimore and Szepesv{\'a}ri(2020)]{lattimore2020bandit}
T.~Lattimore and C.~Szepesv{\'a}ri.
\newblock \emph{Bandit Algorithms}.
\newblock Cambridge University Press, 2020.
\newblock ISBN 9781108486828.

\bibitem[Mannor and Tsitsiklis(2004)]{mannor2004sample}
Shie Mannor and John~N Tsitsiklis.
\newblock The sample complexity of exploration in the multi-armed bandit problem.
\newblock \emph{Journal of Machine Learning Research}, 5\penalty0 (Jun):\penalty0 623--648, 2004.

\bibitem[Robbins(1952)]{robbins1952some}
Herbert Robbins.
\newblock Some aspects of the sequential design of experiments.
\newblock 1952.

\bibitem[Stoltz et~al.(2011)Stoltz, Bubeck, and Munos]{stoltz:hal-00609550}
Gilles Stoltz, S{\'e}bastien Bubeck, and R{\'e}mi Munos.
\newblock {Pure exploration in finitely-armed and continuous-armed bandits}.
\newblock \emph{{Theoretical Computer Science}}, 412\penalty0 (19):\penalty0 1832--1852, April 2011.
\newblock \doi{10.1016/j.tcs.2010.12.059}.
\newblock URL \url{https://hec.hal.science/hal-00609550}.

\bibitem[Thompson(1933)]{thompson1933likelihood}
William~R Thompson.
\newblock On the likelihood that one unknown probability exceeds another in view of the evidence of two samples.
\newblock \emph{Biometrika}, 25\penalty0 (3/4):\penalty0 285--294, 1933.

\bibitem[Wang et~al.(2023)Wang, Tzeng, and Proutiere]{wang2023best}
Po-An Wang, Ruo-Chun Tzeng, and Alexandre Proutiere.
\newblock Best arm identification with fixed budget: A large deviation perspective.
\newblock In \emph{Thirty-seventh Conference on Neural Information Processing Systems}, 2023.
\newblock URL \url{https://openreview.net/forum?id=gYetLsNO8x}.

\bibitem[You et~al.(2023)You, Qin, Wang, and Yang]{DBLP:conf/colt/YouQWY23}
Wei You, Chao Qin, Zihao Wang, and Shuoguang Yang.
\newblock Information-directed selection for top-two algorithms.
\newblock In Gergely Neu and Lorenzo Rosasco, editors, \emph{The Thirty Sixth Annual Conference on Learning Theory, {COLT} 2023, 12-15 July 2023, Bangalore, India}, volume 195 of \emph{Proceedings of Machine Learning Research}, pages 2850--2851. {PMLR}, 2023.
\newblock URL \url{https://proceedings.mlr.press/v195/you23a.html}.

\bibitem[Zhan et~al.(2026)Zhan, Han, and Zhang]{zhan2026lastiterateanalysesftrl12tsallis}
Jingxin Zhan, Yuze Han, and Zhihua Zhang.
\newblock Last-iterate analyses of ftrl with the 1/2-tsallis entropy in stochastic bandits, 2026.
\newblock URL \url{https://arxiv.org/abs/2510.22819}.

\bibitem[Zhao et~al.(2025)Zhao, Ito, and Li]{zhao2025heavytailedlinearbanditsadversarial}
Canzhe Zhao, Shinji Ito, and Shuai Li.
\newblock Heavy-tailed linear bandits: Adversarial robustness, best-of-both-worlds, and beyond, 2025.
\newblock URL \url{https://arxiv.org/abs/2508.13679}.

\bibitem[Zhao et~al.(2023)Zhao, Stephens, Szepesv{\'{a}}ri, and Jun]{DBLP:conf/icml/ZhaoSSJ23}
Yao Zhao, Connor Stephens, Csaba Szepesv{\'{a}}ri, and Kwang{-}Sung Jun.
\newblock Revisiting simple regret: Fast rates for returning a good arm.
\newblock In Andreas Krause, Emma Brunskill, Kyunghyun Cho, Barbara Engelhardt, Sivan Sabato, and Jonathan Scarlett, editors, \emph{International Conference on Machine Learning, {ICML} 2023, 23-29 July 2023, Honolulu, Hawaii, {USA}}, volume 202 of \emph{Proceedings of Machine Learning Research}, pages 42110--42158. {PMLR}, 2023.
\newblock URL \url{https://proceedings.mlr.press/v202/zhao23g.html}.

\bibitem[Zhong et~al.(2021)Zhong, Cheung, and Tan]{zhong2021achieving}
Zixin Zhong, Wang~Chi Cheung, and Vincent~YF Tan.
\newblock Achieving the pareto frontier of regret minimization and best arm identification in multi-armed bandits.
\newblock \emph{arXiv preprint arXiv:2110.08627}, 2021.

\bibitem[Zimmert and Seldin(2019)]{pmlr-v89-zimmert19a}
Julian Zimmert and Yevgeny Seldin.
\newblock An optimal algorithm for stochastic and adversarial bandits.
\newblock In Kamalika Chaudhuri and Masashi Sugiyama, editors, \emph{Proceedings of the Twenty-Second International Conference on Artificial Intelligence and Statistics}, volume~89 of \emph{Proceedings of Machine Learning Research}, pages 467--475. PMLR, 16--18 Apr 2019.

\bibitem[Zimmert and Seldin(2021)]{zimmert2021tsallis}
Julian Zimmert and Yevgeny Seldin.
\newblock Tsallis-inf: An optimal algorithm for stochastic and adversarial bandits.
\newblock \emph{Journal of Machine Learning Research}, 22\penalty0 (28):\penalty0 1--49, 2021.

\bibitem[Zimmert et~al.(2019)Zimmert, Luo, and Wei]{zimmert2019beating}
Julian Zimmert, Haipeng Luo, and Chen-Yu Wei.
\newblock Beating stochastic and adversarial semi-bandits optimally and simultaneously.
\newblock In \emph{International Conference on Machine Learning}, pages 7683--7692. PMLR, 2019.

\end{thebibliography}
\bibliographystyle{plainnat}
\newpage
\appendix
\newpage
\renewcommand{\appendixpagename}{\centering \LARGE Appendix}
\appendixpage

\startcontents[section]
\printcontents[section]{l}{1}{\setcounter{tocdepth}{2}}
\newpage
\section{Proof of Theorem \ref{lya}}\label{pflya}
In this section, we provide a full proof of Theorem \ref{lya}. All of the verifications of constants are deferred to Appendix \ref{verify1}.

First, we will show \eq{lsuper}, or more formally:
\begin{Lem}\label{formalsuper}
    Recall that $D_t=\hat{L}_{t,i_*}-N_t$, where $N_t=\min _{i\in\mathcal{I}, i\ne i_*} \hat{L}_{t,i}$.
    There exist constants $C>0$ and $C_{\alpha,q,\rho},C'_{\alpha,q,\rho}>0$ depending only on $\alpha$, $q$ and $\rho$ such that if $0<\alpha<1$, $\lambda=C\frac{\Delta}{d}$ (required in the definition of $g_t$ in \eq{gt2}), $L= \lfloor C_{\alpha,q,\rho}
%\frac{K^2\log(d)}{\Delta^2}
\frac{d\log(d)}{\Delta^2}
\rfloor+1$ and $t\ge N:=
\lfloor C'_{\alpha,q,\rho}
%K^2L^2
\frac{d^3\log(d)^2}{\Delta^4}\rfloor+1,
$ then provided $D_t\ge -\sqrt{dt}/\alpha$, we have
    \[
    \E\!\left[g_{t+L}(D_{t+L})\,\middle|\,\mathscr{F}_{t-1}\right]\le g_t(D_t).
    \]
\end{Lem}
\begin{proof}
    Denote 
    % $x=D_t+\Delta t$,
    $\rho=2-\frac{\alpha^2 \mu_{i_*}}{4}-q>0$. In the following, we will compute
    \[
    \E\!\left[g_{t+L}(D_{t+L})\,\middle|\,\mathscr{F}_{t-1}\right]/ g_t(D_t)
    \]
    to show the desired result. 
    The computation is divided into two parts according to whether $D_t \ge 0$.

    Before presenting the detailed computation, we first introduce some auxiliary variables \begin{equation}\label{xut}
        \begin{aligned}
            U_t&:=\left(\Delta(t+L)+D_{t+L}\right)-\left(\Delta t+D_t\right) = R_t + T_t, \\
            R_t&=\hat{L}_{t+L,i_*}-\hat L_{t,i_*}-L\mu_{i_*},\\ T_t&=\left(\mu_{i_*}+\Delta\right) L-(N_{t+L}-N_t).
        \end{aligned}
        % =\left(\Delta(t+L)+D_{t+L}\right)-x
        ,
    \end{equation}
    Instead of considering the difference between $D_{t+L}$ and $D_t$, we consider the difference between $\Delta(t+L)+D_{t+L}$ and $\Delta t+D_t$ to balance the means.
    This definition directly follows from the expression of $g_t$ in \eq{gt2} on the positive half-line and is also helpful for the computation on the negative half-line.
    % In the following, we will compute 
    % \[
    % \E\!\left[g_{t+L}(x+U_t)\,\middle|\,\mathscr{F}_{t-1}\right]/g_t(x).
    % \]

    Recall that $\hat{L}_t = \sum_{s=1}^{t-1} \hat{\ell}_s$.
    From the unbiasedness of $\hat{\ell}_s$ defined in  \eq{eq:loss-vector-est}, we have $\E\!\left[R_t\,\middle|\,\mathscr{F}_{t-1}\right] = 0$.
    The term $T_t$ contains $N_{t+L}-N_t$ and is bounded using the lemmas in Appendix~\ref{growthofnt}.

    We distinguish two cases to prove the results and we will verify that the choices of $\lambda, L$ and $N$ satisfy all following the constrains in Appendix \ref{verify1}.
    \paragraph{The positive half-line:}

    First, we consider the case when $D_t\ge 0$. Since the $g_t(y)$ in \eq{gt2} depends on $y$ only through $y + \Delta t$ on the positive half-line, we let $x=D_t+\Delta t$ to simplify the expression such that
    \[
    g_{t}(D_{t})=t^{-\theta}x^2.
    \]
    Since (see Appendix \ref{verify1})
\begin{equation}\label{eq:lya-time-block}
    t\ge N,
    \qquad
    4dL\le\frac{\sqrt{dN}}{4\alpha}\le\frac{\sqrt{dt}}{4\alpha},
\end{equation}
and \(D_t\ge 0\), by Lemma \ref{lem:block_deterministic_control} (1), we have
    \[
    D_{t+L}\ge D_t-4dL\ge -4dL.
    \]
    Note that (see Appendix \ref{verify1})
\begin{equation}\label{eq:lya-time-comparison}
    t+L\ge N\ge \frac{4dL+\frac{2}{\lambda}}{\Delta}.
\end{equation}
    then by Lemma \ref{lem:comparison_gn} (2) with $t$ replaced with $t+L$ and $R=4dL$, when $D_{t+L} < 0$, we have
    \[
    g_{t+L}(D_{t+L})\leq (t+L)^{-\theta}(\Delta(t+L)+D_{t+L})^2=(t+L)^{-\theta}(x+U_t)^2;
    \]
    when $D_{t+L} \ge 0$, this inequality becomes an equality.
    
    Moreover, \eq{eq:lya-time-comparison} implies
    \[
    x+U_t=x+R_t+T_t=\Delta(t+L)+D_{t+L}\ge\frac{2}{\lambda}>0.
    \]
    Since $T_t^+\ge T_t$, it follows that
    \[
    (x+U_t)^2=(x+R_t+T_t)^2\le(x+R_t+T_t^+)^2.
    \]
    Therefore, since $x=D_t+\Delta t$ is $\mathscr{F}_{t-1}$-measurable, taking the conditional expectation yields
    \begin{equation}\label{ratiopo}
    \begin{aligned}
    \frac{\E\!\left[g_{t+L}(D_{t+L})\,\middle|\,\mathscr{F}_{t-1}\right]}{g_t(D_t)}
    \le&\left(1+\frac Lt\right)^{-\theta}\Bigg[1+
    \frac{2\E\!\left[T_t^+\,\middle|\,\mathscr{F}_{t-1}\right]}{x}\\
    &+\frac{\E\!\left[R_t^2\,\middle|\,\mathscr{F}_{t-1}\right]
    +\E\!\left[(T_t^+)^2\,\middle|\,\mathscr{F}_{t-1}\right]
    +2\E\!\left[R_tT_t^+\,\middle|\,\mathscr{F}_{t-1}\right]}{x^2}\Bigg].
    \end{aligned}
    \end{equation}

    By Lemma~\ref{tt2} (1) and the parameter choice verified in Appendix~\ref{verify1},
    \begin{equation}\label{eq:lya-L-ut1}
    L\ge \frac{4C'd\log(d)}{\Delta^2},
    \qquad C'\ge \max\left\{\frac{1024C^2}{\rho^2},\frac{4C}{\rho}\right\},
    \end{equation}
    so, using $x\ge\Delta t$,
    \begin{equation}\label{ut1}
    \E\!\left[T_t^+\,\middle|\,\mathscr{F}_{t-1}\right]
    \le\frac{\rho L\Delta}{16}
    \le\frac{\rho Lx}{16t}.
    \end{equation}

    By Corollary~\ref{rt2}, since
    \begin{equation}\label{eq:lya-time-rt2}
    t\ge N\ge C_{\alpha,q,\rho}\frac{dL}{\Delta^2},
    \end{equation}
    and $D_t=x-\Delta t\ge0$,
    \begin{equation}\label{term1}
    \E\!\left[R_t^2\,\middle|\,\mathscr{F}_{t-1}\right]
    \le\left(\frac{\mu_{i_*}\alpha^2}{4}+\frac{\rho}{4}\right)\frac{Lx^2}{t}.
    \end{equation}
    Lemma~\ref{tt2} (1) also gives
    \[
    \E\!\left[(T_t^+)^2\,\middle|\,\mathscr{F}_{t-1}\right]
    \le C\left(4dL\log(d)+(\log(d))^2\right).
    \]
    Thus, since $x\ge\Delta t$ and
    \begin{equation}\label{eq:lya-time-term2}
    t\ge N\ge\frac{C}{\rho^2\Delta^2}
    \left(4d\log(d)+\frac{(\log(d))^2}{L}\right),
    \end{equation}
    we have
    \begin{equation}\label{term2}
    \E\!\left[(T_t^+)^2\,\middle|\,\mathscr{F}_{t-1}\right]
    \le\frac{\rho^2}{256}\frac{Lx^2}{t}
    \le\frac{\rho Lx^2}{8t},
    \end{equation}
    and, by Cauchy--Schwarz,
    \begin{equation}\label{term3}
    2\left|\E\!\left[R_tT_t^+\,\middle|\,\mathscr{F}_{t-1}\right]\right|
    \le2\sqrt{\E\!\left[R_t^2\,\middle|\,\mathscr{F}_{t-1}\right]
    \E\!\left[(T_t^+)^2\,\middle|\,\mathscr{F}_{t-1}\right]}
    \le\frac{\rho Lx^2}{8t}.
    \end{equation}
    Combining \eq{term1}--\eq{term3},
    \begin{equation}\label{ut2}
    \E\!\left[R_t^2\,\middle|\,\mathscr{F}_{t-1}\right]
    +\E\!\left[(T_t^+)^2\,\middle|\,\mathscr{F}_{t-1}\right]
    +2\E\!\left[R_tT_t^+\,\middle|\,\mathscr{F}_{t-1}\right]
    \le\left(\frac{\mu_{i_*}\alpha^2}{4}+\frac{\rho}{2}\right)\frac{Lx^2}{t}.
    \end{equation}
    
    Finally, it suffices to bound $\left(1+\frac{L}{t}\right)^{-\theta}$. In fact, by Lemma \ref{lem:one_plus_u_negative_theta}, since (see Appendix \ref{verify1}) \begin{equation}\label{eq:lya-time-one-plus}
    t\ge N\ge \frac{24L}{\rho}
    \ge \frac{4\theta(\theta+1)L}{\rho},
\end{equation} we have
    \begin{equation}\label{1lt}
        \left(1+\frac{L}{t}\right)^{-\theta}\le 1-\frac{\theta L}{t} +\frac{\theta(\theta+1)}{2}\frac{L^2}{t^2}\leq 1-\frac{\theta L}{t}+\frac{\rho L}{8t}.
    \end{equation}
    Then by \eq{ratiopo}, \eq{ut1}, \eq{ut2} and \eq{1lt}, we have
    \[
    \begin{aligned}
        \frac{\E\!\left[g_{t+L}(D_{t+L})\,\middle|\,\mathscr{F}_{t-1}\right]}{g_{t}(D_{t})}\leq& \left(1-(\theta-\frac{\rho}{8})\frac{L}{t}\right)\left(
    1+(\frac{\mu_{i_*}\alpha^2}{4}+\frac{5\rho}{8}) \frac{L}{t}
    \right)\\
    \leq &1-\left(\theta-\frac{\mu_{i_*}\alpha^2}{4}-\frac{3\rho}{4}\right)\frac{L}{t}<1,
    \end{aligned}
    \]
    where we used that
    \[
    \theta=2-q=\frac{\mu_{i_*}\alpha^2}{4}+\rho.
    \]
\paragraph{The negative half-line:} 

Then we consider the negative half-line, where $-\sqrt{dt}/\alpha\leq D_t\leq 0$ and
\[
g_{t}(D_{t})=\Delta^2 t^qe^{\lambda D_t}.
\]
 By Lemma \ref{lem:comparison_gn} (1), since \(t\ge N\) and \eq{eq:lya-time-comparison} holds, we have
\[
g_{t+L}(D_{t+L})\leq \Delta^2(t+L)^q e^{\lambda D_{t+L}}.
\]
Therefore, 
\begin{equation}\label{ratione}
    \begin{aligned}
        \frac{\E\!\left[g_{t+L}(D_{t+L})\,\middle|\,\mathscr{F}_{t-1}\right]}{g_{t}(D_{t})}\leq& \left(1+\frac{L}{t}\right)^q \E\!\left[e^{\lambda (D_{t+L}-D_t)}\,\middle|\,\mathscr{F}_{t-1}\right]\\= &\left(1+\frac{L}{t}\right)^q e^{-\lambda\Delta L}\E\!\left[e^{\lambda U_t}\,\middle|\,\mathscr{F}_{t-1}\right]
        \\= &\left(1+\frac{L}{t}\right)^q e^{-\lambda\Delta L}\E\!\left[e^{\lambda (T_t+R_t)}\,\middle|\,\mathscr{F}_{t-1}\right]\\
        \leq &\left(1+\frac{L}{t}\right)^q e^{-\lambda\Delta L}\sqrt{\E\!\left[e^{2\lambda T_t}\,\middle|\,\mathscr{F}_{t-1}\right]\E\!\left[e^{2\lambda R_t}\,\middle|\,\mathscr{F}_{t-1}\right]}.
    \end{aligned}
\end{equation}
where in the third line we used that \[
D_{t+L}-D_t=U_t-\Delta L =T_t+R_t-\Delta L,
\]
and in the third line we applied Cauchy–Schwarz inequality.

Then, since \eq{eq:lya-time-block} holds and $D_t\ge-\sqrt{dt}/\alpha$, by Lemma \ref{tt2} (2), we have
\begin{equation}\label{2ltt}
    \E\left[
    \exp\left(2\lambda T_t\right)
    \,\middle|\,\mathscr{F}_{t-1}\right]\leq d\exp\left(16\lambda^2dL\right).
\end{equation}
Similarly, since (see Appendix \ref{verify1}) \begin{equation}\label{eq:lya-lambda-small}
    \lambda\le \frac{1}{8d},
\end{equation} by Lemma \ref{negexp}, we have
\begin{equation}\label{2lrt}
    \E\left[\exp\left(2\lambda R_t\right)\,\middle|\,\mathscr{F}_{t-1}\right]\leq e^{16dL\lambda^2}.
\end{equation}
Hence, combining \eq{2ltt} and \eq{2lrt} with \eq{ratione}, one can see that
\[
\frac{\E\!\left[g_{t+L}(D_{t+L})\,\middle|\,\mathscr{F}_{t-1}\right]}{g_t(D_t)}
\le\left(1+\frac Lt\right)^q
\exp\left(-\lambda\Delta L+\frac12\log(d)+16\lambda^2dL\right).
\]
Then since (see Appendix \ref{verify1}) \begin{equation}\label{eq:lya-neg-exp-params}
    L\ge \frac{2\log(d)}{\lambda\Delta},
    \qquad
    \lambda\le \frac{\Delta}{64d},
\end{equation}
% where $C$ is a positive constant,
we have
\[
\frac{\E\!\left[g_{t+L}(D_{t+L})\,\middle|\,\mathscr{F}_{t-1}\right]}{g_t(D_t)}
\le\left(1+\frac Lt\right)^q\exp\left(-\lambda\Delta L/2\right).
\]
Hence, since $1+a\leq e^a$ for any $a\in \R$, we have
\[
\frac{\E\!\left[g_{t+L}(D_{t+L})\,\middle|\,\mathscr{F}_{t-1}\right]}{g_t(D_t)}
\le\exp\left(-\lambda\Delta L/2+\frac{qL}{t}\right),
\]
which is less than $1$ since (see Appendix \ref{verify1}) \begin{equation}\label{eq:lya-time-neg-final}
    t\ge N\ge \frac{2qL}{\lambda\Delta}.
\end{equation} 

% Finally, since $K\asymp d$ if $0<\alpha<1$, then there exist constants $C>0$ and $C_{\alpha,q,\rho},C'_{\alpha,q}>0$ depending only on $\alpha$ and $q$ such that when choosing $\lambda=C\frac{\Delta}{d^2}$, $L= \lfloor C_{\alpha,q,\rho}
% %\frac{K^2\log(d)}{\Delta^2}
% \frac{d^2\log(d)}{\Delta^2}
% \rfloor+1$ and $t\ge 
% C'_{\alpha,q}
% %K^2L^2
% \frac{d^6\log(d)^2}{\Delta^4},
% $
% the constraints about $t,\lambda$ and $L$ above can all be satisfied. This completes the proof for this lemma.

Combining the two cases proves this lemma. We verify that the choices of $\lambda, L$ and $N$ satisfy all the constrains above in Appendix \ref{verify1}.
\end{proof}

In the following, we choose the same $\lambda$ and $L$ defined in Lemma \ref{formalsuper}. Then, we present how to show that $\E\left[g_t(D_t)\right]$ is uniformly bounded for $t\ge N$. Following the definition of $N$ in Lemma \ref{formalsuper}, for those $t\ge N$, we will apply Lemma \ref{formalsuper} to give an upper bound for
\[
\sup_{t\ge N} \E\left[g_t(D_t)\1_{\left\{D_t\ge-\sqrt{dt}/\alpha\right\}}\right]
\]
via a stopping-time argument. For the case when $D_t<-\sqrt{dt}/\alpha$, we will simply use the monotonicity of $g_t$. Theorem~\ref{lya} concerns only $t\ge N$; the proof of Theorem~\ref{et} uses $\Err_t\le1$ directly when $1\le t<N$.

Fix a residue class $r\in\{0,1,\dots,L-1\}$, and define the grid
\begin{equation}\label{sj}
    s_j:=N+r+(j-1)L,\qquad j\ge 1,
\end{equation}
where $N$ was defined in Lemma \ref{formalsuper}. We deal with each grid separately and let
\[
\mathscr{G}_j:=\mathscr{F}_{s_j-1},
\qquad
X_j:=D_{s_j},
\qquad
\mathcal{M}_j:=\left[-\frac{\sqrt{ds_j}}{\alpha},+\infty\right),
\]
and
\begin{equation}\label{hjdef}
    H_j:=g_{s_j}(D_{s_j}),
\qquad j\ge 1.
\end{equation}
Then $(X_j)_{j\ge 1}$ is adapted to $(\mathscr{G}_j)_{j\ge 1}$, and $(H_j)_{j\ge 1}$ is a sequence of nonnegative random variables. 
The grid spacing $L$ in \eq{sj} allows us to apply Lemma~\ref{formalsuper}.
In particular,
under the condition of Lemma \ref{formalsuper}, if $X_j\in\mathcal{M}_j$, we have
\[
\E\left[H_{j+1}\,\middle|\,\mathscr{G}_j\right]\leq H_j.
\]
Moreover, we emphasize that the above definition omits the dependence on the residue class $r$, because the subsequent derivation is independent of $r$.
Every $t\ge N$ can be written uniquely as $t=N+r+(j-1)L$ with $0\le r<L$ and $j\ge1$. Therefore,
\[
\sup_{t\ge N} \E\left[g_t(D_t)\1_{\left\{D_t\ge-\sqrt{dt}/\alpha\right\}}\right]
= \sup_{r\in\{0,1,\dots,L-1\}} \sup_{j \ge 1} \E[H_j\1_{\{X_j\in\mathcal{M}_j\}}].
\]

However, it is difficult to control \(\sup_{j \ge 1} \E[H_j\1_{\{X_j\in\mathcal{M}_j\}}]\) directly.
Thus, we resort to a stopping-time argument.
Define recursively a sequence of stopping times \((\sigma_k,\tau_k)_{k\ge 0}\) by
\begin{equation}\label{sigma0}
    \sigma_0:=1,
\qquad
\tau_k:=\inf\{j\ge \sigma_k:\, X_j\notin \mathcal{M}_j\},
\end{equation}
and, for \(k\ge 1\),
\begin{equation}\label{sigmak}
    \sigma_k:=\inf\{j>\tau_{k-1}:\, X_j\in \mathcal{M}_j\},
\end{equation}
again with the convention \(\inf\varnothing=+\infty\). 
Intuitively, these stopping times partition all positive integers according to whether $X_j\in \mathcal{M}_j $ holds.
Then, by Lemma \ref{ht}, we can convert the union bound over the grid into a series over the sequence of stopping times $(\sigma_k)_{k \ge0}$, at which the condition
\(X_j \in \mathcal{M}_j\) is satisfied, as shown below
\begin{equation}\label{ehj}
    \sup_{j \ge 1} \mathbb E[H_j\1_{\{X_j\in\mathcal{M}_j\}}]
\le
\sum_{k=0}^{+\infty}\mathbb E\!\left[H_{\sigma_k}\1_{\{\sigma_k<+\infty\}}\right].
\end{equation}
% Note that
%     \[
%     t\ge N\ge \frac{KL+\frac{2}{\lambda}}{\Delta},
%     \]
By Lemma \ref{lem:comparison_gn} (1), since $s_{\sigma_k}\ge N$ and \eq{eq:lya-time-comparison} holds,
then for any $k\ge 1$, on the event where $\sigma_k<+\infty$, we have
\[
H_{\sigma_k}\leq \Delta^2 s_{\sigma_k}^q e^{\lambda D_{s_{\sigma_k}}}.
\]
By the definition of $\sigma_k$ in \eq{sigmak}, we have
\[
D_{s_{\sigma_k-1}}<-\frac{\sqrt{d s_{\sigma_k-1}}}{\alpha}<0.
\]
Then, since 
%$N\ge (3KL)^2$
\eq{eq:lya-time-block} holds
and $s_{\sigma_k}=s_{\sigma_k-1}+L$, by Lemma \ref{lem:block_deterministic_control} (2), we have
\[
D_{s_{\sigma_k}}
\le D_{s_{\sigma_k-1}}+4dL
\le-\frac{3\sqrt{d s_{\sigma_k-1}}}{4\alpha}
\le-\frac{\sqrt{d s_{\sigma_k}}}{2\alpha},
\]
where the last inequality uses $s_{\sigma_k}=s_{\sigma_k-1}+L\le2s_{\sigma_k-1}$, which follows from $N\ge L$. Hence,
\[
H_{\sigma_k}\leq \Delta^2s_{\sigma_k}^q
e^{-\lambda\sqrt{d s_{\sigma_k}}/(2\alpha)}.
\]
Therefore,
\[
\mathbb E[H_j\1_{\{X_j\in\mathcal{M}_j\}}]
\le
\mathbb E\!\left[H_{\sigma_0}\right]+
\sum_{k=1}^{+\infty}\Delta^2s_{\sigma_k}^q
e^{-\lambda\sqrt{d s_{\sigma_k}}/(2\alpha)}.
\]
For the latter term, by the definition of $s_j$ in \eq{sj}, it is bounded above by
\[
\sum_{j=2}^{+\infty}\Delta^2s_j^q e^{-\lambda\sqrt{ds_j}/(2\alpha)}
\le C_q\Delta^2\left(\frac{\alpha}{\lambda\sqrt d}\right)^{2q}
\left(1+\frac{\alpha^2}{\lambda^2dL}\right),
\]
where we applied Lemma \ref{lem:standalone-sum-bound} in the last inequality and $C_q$ is a constant depending only on $q$.

Then, by the definition of $\sigma_0$ and $s_j$ in \eq{sigma0} and \eq{sj}, respectively, and by \eq{ehj},
\begin{equation}\label{ehjnew}
\sup_{j\ge1}\E\left[H_j\1_{\{X_j\in\mathcal{M}_j\}}\right]
\le \E[H_1]
+C_q\Delta^2\left(\frac{\alpha}{\lambda\sqrt d}\right)^{2q}
\left(1+\frac{\alpha^2}{\lambda^2dL}\right).
\end{equation}
Since every $t\ge N$ can be written as $t=N+r+(j-1)L$ with $0\le r<L$ and $j\ge1$, the definition of $H_j$ gives
\begin{align}
\sup_{t\ge N}\E\left[g_t(D_t)\1_{\{D_t\ge-\sqrt{dt}/\alpha\}}\right]
&\le \max_{0\le r<L}\E\left[g_{N+r}(D_{N+r})\right]\nonumber\\
&\quad+C_q\Delta^2\left(\frac{\alpha}{\lambda\sqrt d}\right)^{2q}
\left(1+\frac{\alpha^2}{\lambda^2dL}\right).
\label{egtdt1}
\end{align}

If $D_t<-\sqrt{dt}/\alpha$, monotonicity of the negative branch gives
\[
g_t(D_t)
\le \Delta^2t^q\exp\left(-\frac{\lambda\sqrt{dt}}{\alpha}\right)
\le \Delta^2\left(\frac{2q\alpha}{e\lambda\sqrt d}\right)^{2q}.
\]
Hence
\begin{equation}\label{eq:lesssqrt}
\sup_{t\ge N}\E\left[g_t(D_t)\1_{\{D_t<-\sqrt{dt}/\alpha\}}\right]
\le \Delta^2\left(\frac{2q\alpha}{e\lambda\sqrt d}\right)^{2q}.
\end{equation}
Combining \eq{egtdt1} and \eq{eq:lesssqrt},
\begin{equation}\label{eq:last}
\begin{aligned}
\sup_{t\ge N}\E[g_t(D_t)]
\le&\max_{0\le r<L}\E[g_{N+r}(D_{N+r})]\\
&+C'_q\Delta^2\left(\frac{\alpha}{\lambda\sqrt d}\right)^{2q}
\left(1+\frac{\alpha^2}{\lambda^2dL}\right).
\end{aligned}
\end{equation}

It remains to bound $\E[g_{N+r}(D_{N+r})]$ for $0\le r<L$. Set $u=N+r$. The argument based on Lemma~\ref{dt} gives
\begin{equation}\label{egtdt2}
\E[g_u(D_u)]
\le2\Delta^2u^q+2C_\alpha d u^{q-1}
\le C_{\alpha,q}\left(\Delta^2N^q+dN^{q-1}\right).
\end{equation}
Here $N\le u<N+L\le2N$. The last inequality is valid for every $q>0$: when $q<1$, $u^{q-1}\le N^{q-1}$, whereas for $q\ge1$, $u^{q-1}\le2^{q-1}N^{q-1}$. Since the parameter choice also ensures $\Delta^2N\ge d$, the second term in \eq{egtdt2} is bounded by the first. Therefore \eq{eq:last} yields the unsimplified bound
\[
\sup_{t\ge N}\E[g_t(D_t)]
\le C_{\alpha,q,\rho}\left[
\Delta^2N^q
+\Delta^2\left(\frac{\alpha}{\lambda\sqrt d}\right)^{2q}
+\frac{\Delta^2}{L}\left(\frac{\alpha}{\lambda\sqrt d}\right)^{2q+2}
\right].
\]
Substituting the choices of $\lambda,L,N$ and absorbing the last two terms into the first gives
\[
\E[g_t(D_t)]
\le C_{\alpha,q,\rho}\frac{d^{3q}(\log(d))^{2q}}{\Delta^{4q-2}},
\qquad t\ge N,
\]
which proves Theorem~\ref{lya}.

\section{Proof of Theorem \ref{lower}}\label{pflower}
In this section, we provide a full proof of Theorem \ref{lower}. 
% All of the verifications of constants are deferred to Appendix \ref{verify2}.

In order to formalize the definition of the sequence of stopping times in the proof sketch in Section \ref{pfsklower}, we need to find a start point. In other words, we should show that there exists a $T$ satisfying the conditions on $t$ in Lemma~\ref{base}.
\begin{Lem}\label{start}
Assume that \(0<\alpha<1\),  \(\mu_{i_*}>0\). 
% and 
% % \(C>0\) is constant that is large enough
% $C>144$. For any
% \(B>\frac{Cd}{\mu_{i_*}\alpha^2}\), define \(v_B:=\frac{\mu_{i_*}\alpha^2 B^2}{16}\),  \(V_B:=4\alpha^2 B^2\). 
Following the same definition of $C$, $B$, $v_B$ and $V_B$ in Lemma \ref{base}, if $C\ge 144$, then there exists $T\ge Cd^2$ such that
\begin{equation}\label{aimT}
    \P\left(v_B\leq \frac{D_T}{T}\leq V_B\right)>0.
\end{equation}
\end{Lem}

\begin{proof}[Proof of Lemma \ref{start}]
    By the union bound, it suffices to show that
\[
\P\left(\text{there exists $t\ge Cd^2$ such that }v_B\leq \frac{D_t}{t}\leq V_B\right)>0.
\]

To show this, we will make $D_{T'}$ sufficiently large by selecting $i_*$ during rounds $1,\ldots,T'-1$, and then prevent further increments of $\hat L_{t,i_*}$ by not selecting $i_*$ during rounds $T',\ldots,T''$.
Recall that $D_t$ defined in Section~\ref{Preliminaries} is the gap between the cumulative loss of the best arm $i^*$ and the minimal cumulative loss among the other arms.
On this event, we will show that $D_t/t$ enters $[v_B,V_B]$ at some $t\in(T',T'']$.

For any $t>Cd^2$, consider the following event
\[
E(t):=\left\{
I_1=\cdots= I_t=i_*,\text{ and }\ell_{1,i_*},\cdots,\ell_{t,i_*}\ge \mu_{i_*}/2
\right\}.
\]
On event $E(t)$, we always choose the best arm for the first $t$ rounds, and the observed losses have a non-zero lower bound.
From the definition of $\hat{L}_s$ and $N_s$ in Section~\ref{Preliminaries}, 
for any $1\leq s\leq t+1$,
only the $i_*$-element of $\hat{L}_s$ is non-zero and thus $N_s = 0$.
Moreover, for every $1\le s\le t$, the non-zero element satisfies
\[
\hat{L}_{s+1,i_*}=\hat{L}_{s,i_*}+\frac{\ell_{s,i_*}}{p_{s,i_*}}\ge \hat{L}_{s,i_*}+\frac{\mu_{i_*}}{2p_{s,i_*}}.
\]
Then by Lemma \ref{4-2} with $\eta_s = \alpha / \sqrt{s}$, we have
\[
\hat{L}_{s,i_*}+\left(\sqrt{d}+\frac{\alpha\hat{L}_{s,i_*}}{2\sqrt{s}}\right)^{2}\ge\hat{L}_{s+1,i_*}\ge \hat{L}_{s,i_*}+\frac{\mu_{i_*}}{2}\left(1+\frac{\alpha\hat{L}_{s,i_*}}{2\sqrt{s}}\right)^{2},
\]
% where we also applied $\underline{\hat{L}}_{s,i_*} = {\hat{L}}_{s,i_*} - \min_{i \in \mathcal{I}} {\hat{L}}_{s,i} = {\hat{L}}_{s,i_*}$. because only the $i_*$-element of $\hat{L}_s$ is non-zero.

Define two deterministic sequences $(w_t)_{t\ge1}$ and $(z_t)_{t\ge1}$ that satisfy that $w_1=z_1=0$ and for any $t\ge 1$,
\[
w_{t+1}=w_t+\left(\sqrt{d}+\frac{\alpha w_t}{2\sqrt{t}}\right)^{2}, z_{t+1}=z_t+\frac{\mu_{i_*}}{2}\left(1+\frac{\alpha z_t}{2\sqrt{t}}\right)^{2}.
\]
Then, by induction, on $E(t)$ we have, for every $1\le s\le t+1$,
\[
w_s\ge\hat{L}_{s,i_*}\ge z_s.
\]
Then, by Lemma \ref{lem:zm-superlinear}, there exists $T'>Cd^2+1$ such that on the event $E(T'-1)$
\[
\frac{D_{T'}}{T'}
= \frac{ \hat{L}_{T',i_*} - N_{T'} }{T'}
=\frac{\hat{L}_{T',i_*}}{T'}\ge \frac{z_{T'}}{T'}\ge 2V_B.
\]
Also, we have
\[
\hat{L}_{T',i_*}\leq w_{T'}<+\infty.
\]
Let $T''=T'+\lceil\frac{w_{T'}}{d}\rceil$ and we consider the following  event 
\begin{equation}\label{F}
    F:=E(T'-1)\cap\left\{
I_{T'},I_{T'+1},\ldots,I_{T''}\ne i_*
\right\}.
\end{equation}
On $F$, the optimal arm is selected during rounds $1,\ldots,T'-1$ and is not selected during rounds $T',\ldots,T''$.
Define 
\begin{equation}\label{det0}
    t_0:=\lceil\frac{D_{T'}}{V_B}\rceil.
\end{equation}
Since $D_{T'} / T' \ge 2V_B$, we have $t_0 > T'$.
Also, when $C\ge 16$, $v_B\ge d$, hence,
\[
T''\ge 1+\lceil\frac{w_{T'}}{d}\rceil\ge 1+\frac{D_{T'}}{v_B}\ge t_0.
\]
Then, we will show that on the event $F$, we have
\begin{equation}\label{t0}
    v_B\leq \frac{D_{t_0}}{t_0}\leq V_B.
\end{equation}

For the right-hand side of \eq{t0}, recall that $t_0\le T''$. Since $i_*$ is not selected during rounds $T',\ldots,t_0-1$ on $F$, we have $\hat L_{t_0,i_*}=\hat L_{T',i_*}$. Moreover, $N_{t_0}-N_{T'}\ge0$, and hence
\[
D_{t_0}-D_{T'} = 
% \hat{L}_{t_0,i_*} - N_{T_0} - (\hat{L}_{T',i_*} - N_{T'}) =
-(N_{t_0}-N_{T'})\leq 0.
\]
Hence, by the definition of $t_0$ in \eq{det0}, we have
\[
\frac{D_{t_0}}{t_0}\leq \frac{D_{T'}}{t_0}\leq V_B.
\]

 For the left-hand side of \eq{t0}, note that when $C\ge 144$, we have
\[
T'\ge Cd^2\ge 144d^2,
\]
and by the parameter choices in Appendix~\ref{verify2},
\[
V_B\ge B\ge 8d.
\]

% \begin{equation}\label{eq:lower-start-basic}
%     V_B\ge K,
% \end{equation}
% by the definition of $V_B$.

Moreover, since $D_{T'}/V_B\ge2T'\ge1$,
\[
t_0\le\frac{D_{T'}}{V_B}+1\le\frac{2D_{T'}}{V_B}\le\frac{D_{T'}}{4d}.
\]
Thus $t_0\le\lfloor D_{T'}/(4d)\rfloor$. Then, by Lemma \ref{l=1} and $D_{T'} / T' \ge 2V_B$ derived above, we have
% \begin{equation}\label{tpvb}
%     D_{t_0}\ge D_{T'}-K(t_0-T')\ge D_{T'}-Kt_0\ge v_B t_0,
% \end{equation}
\[
D_{t_0}\ge D_{T'}-4d(t_0-T')\ge D_{T'}-4dt_0\ge \left(\frac{V_B}{2}-4d\right)t_0\ge v_B t_0,
\]
where in the last inequality we used that when $C\ge 144$,
% $2(K+v_B)\leq V_B$. 
% \begin{equation}\label{eq:lower-start-vB}
%     2(K+v_B)\le V_B.
% \end{equation}
\[
    V_B-2v_B
    =
    4\alpha^2B^2-\frac{\mu_{i_*}\alpha^2B^2}{8}
    \ge
    \frac{31}{8}\alpha^2B^2\ge 8d.
\]
Combing the above analysis yields \eq{t0} and hence, we have
\[
\P\left(\text{there exists $t\ge Cd^2$ such that }v_B\leq \frac{D_t}{t}\leq V_B\right)\ge \P(F).
\]
    By Lemma~\ref{lem:pfposi}, $\P(F)>0$. Therefore, there indeed exists $T\ge Cd^2$ such that \eq{aimT} holds.
\end{proof}

Now, we apply Lemmas~\ref{base} and \ref{start} to prove Theorem~\ref{lower}, following the sketch of proof in Section \ref{pfsklower}. 

With $v_B$, $V_B$, and $\tau(\cdot)$ defined in Lemma~\ref{base} (increase $C$ to $144$ when necessary) and $T$ defined in Lemma~\ref{start},
the formal definition of the stopping times $(\tau_k)_{k \ge 0}$ is
\[
\tau_0:=\begin{cases}
T-1, & \text{if $v_B\leq \frac{D_T}{T}\leq V_B$},\\[4pt]
+\infty, & \text{else},
\end{cases}
\]
and for any $k\ge 0$,
\[
\tau_{k+1}:=
\begin{cases}
\tau(\tau_k+1), & \text{if $\tau_k< +\infty$},\\[4pt]
+\infty, & \text{else}.
\end{cases}
\]
We first inductively show that for any $k\ge 0$, $\tau_k$ is indeed a stopping time satisfying $\tau_k \ge T+k - 1$. 
It is clearly true for $k=0$ since $D_T$ is measurable in $\mathscr{F}_{T-1}$. Suppose that for some $k \ge 0$, $\tau_k$ is a stopping time satisfying $\tau_{k} \ge T+k-1$.
Then by Lemma~\ref{base}~(1), when $\tau_k < +\infty$, $\tau_{k+1} > \tau_k \ge T+k -1$ and thus $\tau_{k+1} \ge T+k$.
Then for any $t\ge T+k$, we have
\begin{equation}\label{decom}
    \{\tau_{k+1}=t\}=\bigcup_{T+k-1\leq t'<t}\{\tau(t'+1)=t,\tau_k=t'\}.
\end{equation}
By the assumption, $\{\tau_k=t'\}\in\mathscr{F}_{t'}\subset\mathscr{F}_{t}$. Also, by Lemma \ref{base}, with $t'+1 \ge T+k$, $\{\tau(t'+1)=t\}\in\mathscr{F}_t$. Putting them together shows that $\tau_{k+1}$ is also a stopping time.

Note that for any $k\ge 0$, by \eq{decom}, we have
\begin{align*}
    \P(\tau_{k+1}<+\infty)
    & = \sum_{t=T+k}^{+\infty} \P(\tau_{k+1} = t) 
    = \sum_{t=T+k}^{+\infty} \sum_{s=T+k-1}^{t-1} \P(\tau(s+1) = t, \tau_k = s) \\
    & = \sum_{s=T+k-1}^{+\infty} \sum_{t=s+1}^{+\infty} \P(\tau(s+1) = t, \tau_k = s)
    = \sum_{s=T+k-1}^{+\infty}
    \P(\tau(s+1)<+\infty,\tau_k=s),
\end{align*}
where for the last equality, we used $\tau(s+1) > s$ by Lemma~\ref{base}~(1).
Besides, since $\tau_k$ is a stopping time, by Lemma \ref{base}~(3), we have
\begin{align*}
    \P(\tau(s+1)<+\infty,\tau_k=s)
    & =\E\left[
    \P(\tau(s+1)<+\infty\mid\mathscr{F}_{s})
    \1_{\{\tau_k=s\}}\right] \\
    & \ge q_B \P\Big(\tau_k=s, v_B\le \frac{D_{s+1}}{s+1} \le V_B \Big)
    % \overset{?}{=}
    =q_B \P(\tau_k=s),
\end{align*}
On the event $\{\tau_k=s\}$, we have $v_B\le D_{s+1}/(s+1)\le V_B$: for $k=0$ this follows from the definition of $\tau_0$, while for $k\ge1$ it follows from $\tau_k=\tau(\tau_{k-1}+1)$ and Lemma~\ref{base}~(2)(iii).
Taking the summation over $s$ yields
\[
\P(\tau_{k+1}<+\infty)\ge q_B\P(\tau_{k}<+\infty).
\]
By iterating, we finally have for any $k\ge 1$
\begin{equation}\label{qb}
    \P(\tau_{k}<+\infty)\ge q_B^k\P(\tau_0<+\infty),
\end{equation}
where $\P(\tau_0<+\infty)>0$ by Lemma \ref{start}.

% By Lemma~\ref{base}~(3), $\P(\tau_1 < \infty) \ge q_B \P\big( v_B \le \frac{D_T}{T} \le V_B \big) =q_B \P(\tau_0 < +\infty)$.
% Then by Lemma~\ref{base}~(2), $\P\big(v_B \le \frac{D_{\tau_1+1}}{\tau_1+1} \le V_B\big) \ge \P(\tau_1 < +\infty) \ge q_B \P(\tau_0 < +\infty)$.
% Similarly, from \(\tau_1\) to \(\tau_2\), we have \(\P\big(v_B \le \frac{D_{\tau_2+1}}{\tau_2+1} \le V_B \big)\ge \P(\tau_2<+\infty)\ge q_B \P\big(v_B \le \frac{D_{\tau_1+1}}{\tau_1+1} \le V_B\big)\ge q_B^2 \P(\tau_0 < +\infty)\).
% Repeating this argument yields 
% \begin{equation}\label{qb}
%     \P(\tau_{k}<+\infty)\ge q_B^k\P(\tau_0<+\infty),
% \end{equation}
% where $\P(\tau_0<+\infty)>0$ by Lemma \ref{start}. 

% By the same argument before \eq{intui}, we also have
% \[
% \Err_t\ge \P(D_t>0)\ge \P(\tau_{\lceil\log_{a_B}(t)\rceil}<+\infty).
% \]
Then, following the argument before \eq{intui}, if $\tau_k<+\infty$, then for any $t$ such that $T\leq t\leq (T-1) a_B^k + 1$, we have $D_{t}>0$. Hence, for any $t\ge T$, we can set $k = \lceil\log_{a_B}(t/(T-1))\rceil$ such that $t \le  (T-1) a_B^k + 1$ and 
\[
\Err_t\ge
\P(D_t>0)\ge \P(\tau_{\lceil\log_{a_B}(t/(T-1))\rceil}<+\infty)\ge \P(\tau_{\lceil\log_{a_B}(t)\rceil}<+\infty),
\]
which, by \eq{qb}, is lower bounded by
the following quantity, where we use $0<q_B<1$ (see Appendix~\ref{verify2}) and $\lceil x\rceil\le x+1$:
\[ \P(\tau_0<+\infty)q_B^{\lceil\log_{a_B}(t)\rceil}\ge q_B\P(\tau_0<+\infty)t^{\frac{\log(q_B)}{\log(a_B)}}.
\]
Finally, it suffices to note that when $B\to+\infty$, $\frac{\log(q_B)}{\log(a_B)}\to -2$.

\section{Proof of Lemma \ref{base}}\label{pfbase}
In this section, we present a full proof of Lemma \ref{base}. We will first define the stopping time $\tau(t)$, then prove the three items accordingly. We will verify that $t$ and $B$ defined in the conditions satisfy all the constrains in our proof below in Appendix \ref{verify2}. 

\paragraph{Definition of $\tau(t)$}
To begin with, we provide the definition of $\tau(t)$. Let 
    \begin{equation}\label{m-+}
        m_{-}(t):=\left\lceil\frac{D_t}{3 B}\right\rceil, \quad m_{+}(t):=\left\lfloor\frac{D_t}{2B}\right\rfloor,
    \end{equation}
    and we consider the time when the player first pulls arm $i_{*}$ after the time step $t$:
    \begin{equation}\label{sigmat}
        \sigma(t):=\inf\left\{
t'\ge t\,:\, 
I_{t'}= i_*
% \hat{\ell}_{t,i}>r
\right\}.
    \end{equation}
    We consider the event
    \begin{equation}\label{G(t)}
        G(t):=\left\{v_B\leq\frac{D_t}{t}\leq V_B\right\}\cap H(t)\cap\{\ell_{\sigma(t),i_*}\ge\mu_{i_*}/2,\sigma(t)<+\infty\},
    \end{equation} where
    \[
    \begin{aligned}
    H(t):=\left\{m_{-}(t)<\sigma(t)\leq m_{+}(t)\right\}.
    %     &G_{-}(t):=
    %     \left\{
    % \sigma(t)>m_{-}(t)
    % \right\};
    % % \left\{
    % % \text{There doesn't exist $
    % % t'\in [t,m_{-}(t)]$ such that $\hat{\ell}_{t,i}>r$}
    % % \right\};
    % \\
    % &G_{+}(t):=
    % \left\{
    % \sigma(t)\leq m_{+}(t) \text{ and }\ell_{\sigma(t)}
    % \right\}.
    % % \left\{
    % % \text{There exists at least one  $t'\in(m_{-}(t),m_{+}(t)]$ such that $\hat{\ell}_{t,i}>r$}
    % % \right\}.
    \end{aligned}
    \]
    One should also note that since $B\ge \frac{Cd}{\alpha^2\mu_{i_*}}$ and $C$ is large enough, when $v_B\leq\frac{D_t}{t}\leq V_B$, we have 
    % $v_B = \frac{\mu_{i_*}\alpha^2 B^2}{16} \ge \frac{Cd}{16}B \ge 3B$ and 
    (see Appendix \ref{verify2})\begin{equation}\label{eq:lower-m-order}
    t<m_-(t)<m_+(t).
\end{equation} Then we define
    \begin{equation}\label{taut}
        \tau(t):=
\begin{cases}
+\infty, & \text{on the event $G^c(t)$},\\[4pt]
% \inf\left\{
% t'>m_{-}(t)\,:\, 
% I_{t'}= i_*
% % \hat{\ell}_{t,i}>r
% \right\}
\sigma(t)
, & \text{on the event $G(t)$}.
\end{cases}
    \end{equation}
We first verify that $\tau(t)$ is indeed a stopping time w.r.t. $(\mathscr{F}_{s})_{s\ge 1}$. Note that $\tau(t)\ge t$ by the definition \eq{taut}, then it suffices to check that for any $t'\ge t$, $ \left\{\tau(t)=t'\right\}\in\mathscr{F}_{t'}$. We have
\[
    \begin{aligned}
        \left\{\tau(t)=t'\right\}=&\left\{v_B\leq\frac{D_t}{t}\leq V_B\right\}\cap\left\{m_{-}(t)<t'\leq m_{+}(t)\right\}\\
    \cap& \left\{
    \text{There is no $s\in[t,t')$ such that $I_{s}= i_*$}
    \right\}\cap \left\{
    I_{t'}=i_* \text{ and }\ell_{t',i_*}\ge \mu_{i_*}/2
    \right\}.
    \end{aligned}
\]

Note that all of $D_t$, $m_{-}(t)$ and $m_{+}(t)$ are measurable in $\mathscr{F}_{t-1}$ and $I_t$ is measurable in $\mathscr{F}_t$, then it is clear that the first three events are also measurable in $\mathscr{F}_{t'}$. Then, since when $I_{t'}=i_*$,
\[
\ell_{t',i_*}=\ell_{t',I_{t'}}
\]
is also measurable in $\mathscr{F}_{t'}$. Thus, the fourth event is also measurable in $\mathscr{F}_{t'}$. Hence, $\tau(t)$ is indeed a stopping time w.r.t. $(\mathscr{F}_{s})_{s\ge 1}$.

\paragraph{Proof of item (1)}
Then, we prove item (1).
If $\tau(t) = +\infty$, this property naturally holds.
If $\tau(t) < +\infty$, then $G(t)$ happens, implying $v_B\leq\frac{D_t}{t}\leq V_B$.
Then when $B$ is large enough we have (see Appendix \ref{verify2})
\begin{equation}\label{eq:lower-aB}
    \frac{m_{-}(t)}{t}\ge \frac{v_B}{3B} \ge \frac{\mu_{i_*}\alpha^2 B}{48}=a_B>1,
\end{equation}
Then it suffices to note that $\tau(t) \ge m_{-}(t)$ by the definition \eq{taut}.

\paragraph{Proof of item (2)}
Next, we prove item (2). Suppose that $\tau(t)<+\infty$, then $G(t)$ happens. Thus, property~(i) naturally holds.

For property~(ii),
we first show that $D_t$ stays positive. Since when $B$ is large enough (see Appendix \ref{verify2}),
\begin{equation}\label{eq:lower-BgeK-mplus}
    B\ge 8d,
    \qquad
    \tau(t)\le m_+(t)\le \frac{D_t}{2B}\le \frac{D_t}{B+4d},
\end{equation}
then by Lemma \ref{l=1} and (see Appendix \ref{verify2})
\begin{equation}\label{eq:lower-time-lone}
    t\ge Cd^2\ge 144d^2,
\end{equation} we have for any $t\leq t'\leq \tau(t)$
\begin{equation}\label{dt'}
    D_{t'}\ge D_t-4d(t'-t)
    \ge (B+4d) \tau(t) -4d(t'-t)
    \ge B t' > 0.
\end{equation}

It remains to show property~(iii), that is
\[
v_B\leq \frac{D_{\tau(t)+1}}{\tau(t)+1}\leq V_B.
\]
To this end, we first show that, for any $t'$ such that $t\leq t'\leq\tau(t)$,
\begin{equation}\label{b}
    Bt'\leq D_{t'}, \text{ and hence, } p_{t',i_{*}}\leq \frac{4}{\alpha^2B^2t'},
\end{equation}
and for any $t'$ such that $m_{-}(t)\leq t'\leq\tau(t)$,
\begin{equation}\label{3b}
     D_{t'}\leq 3Bt', \text{ and hence, } p_{t',i_{*}}\ge  \frac{1}{3\alpha^2B^2t'}.
\end{equation}
The first part of \eq{b} directly follows from \eq{dt'}. For the second part, it suffices to note that
by Lemma \ref{4-2},
\[
p_{t',i_{*}}\leq \left(1+\frac{\alpha D_{t'}}{2\sqrt{t'}}\right)^{-2}\leq \left(1+\frac{\alpha B \sqrt{t'}}{2}\right)^{-2}\leq\frac{4}{\alpha^2B^2t'}.
\]

% For \eq{b}, by \eq{dt'}, one can see that when $t< t'\leq\tau(t)$, we have
% \begin{equation}%\label{smallt'}
%     D_{t'}\ge D_t-K(t'-t)\ge Bt',
% \end{equation}
% where we used $B\ge K$ in \eq{eq:lower-BgeK-mplus}, which implies
% \[
% t'\leq m_{+}(t)\leq \frac{D_t}{2B}\leq \frac{D_t}{B+K}.
% \]
For \eq{3b}, one should note that on the event $G(t)$, arm $i_*$ will not be pulled until the $\tau(t)$ time step, which implies that for any $k\in[t,\tau(t))$
\[
\hat{\ell}_{k,i_*}=0,
\]
then by Lemma \ref{simd}, we have
\begin{equation}\label{unpulled}
    D_{t'}=D_t+N_t-N_{t'}\leq D_t.
\end{equation}
Since $t'\ge m_{-}(t)$, then by the definition of $m_{-}(t)$ in \eq{m-+}, one can see that
\[
D_{t'}\leq D_t\leq 3B m_{-}(t)\leq 3Bt'.
\]
Then, similarly, Lemma \ref{4-2},
% since $B\ge \frac{Cd}{\mu_{i_{*}}\alpha^2}$ and $C$ is large enough,
we have
\begin{equation}\label{eq:ptlower}
    p_{t',i_{*}}\ge \left(\sqrt{d}+\frac{\alpha D_{t'}}{2\sqrt{t'}}\right)^{-2}\ge \left(\sqrt{d}+\frac{3\alpha B \sqrt{t'}}{2}\right)^{-2}\ge \frac{1}{3\alpha^2B^2t'},
\end{equation}
where the last inequality used that $B$ is large enough (see Appendix \ref{verify2}). This proves \eq{3b}. 

Then, by the definition of $G(t)$ in \eq{G(t)}, arm $i_*$ is pulled at the time step $\tau(t)$, hence, 
\[
D_{\tau(t)+1}=D_{\tau(t)}+\frac{\ell_{\tau(t),i_{*}}}{p_{\tau(t),i_{*}}}.
\]
On one hand, by \eq{3b} 
% and Lemma \ref{4-2},
% % since $B\ge \frac{Cd}{\mu_{i_{*}}\alpha^2}$ and $C$ is large enough,
% we have
% \begin{equation}\label{ptaulow}
%     p_{\tau(t),i_{*}}\ge \left(\sqrt{d}+\frac{\alpha D_{\tau(t)}}{2\sqrt{\tau(t)}}\right)^{-2}\ge \left(\sqrt{d}+\frac{3\alpha B \sqrt{\tau(t)}}{2}\right)^{-2}\ge \frac{1}{3\alpha^2B^2\tau(t)},
% \end{equation}
% hence
, with $V_B = 4 \alpha^2 B^2$,
\begin{equation}\label{eq:lower-return-upper}
    D_{\tau(t)+1}
    \le 3B\tau(t)+3\alpha^2B^2\tau(t)
    \le V_B \tau(t)
    \le V_B(\tau(t)+1).
\end{equation}
On the other hand, by \eq{b} 
% and Lemma \ref{4-2},
% \begin{equation}\label{ptauup}
%     p_{\tau(t),i_{*}}\leq \left(1+\frac{\alpha D_{\tau(t)}}{2\sqrt{\tau(t)}}\right)^{-2}\leq \left(1+\frac{\alpha B \sqrt{\tau(t)}}{2}\right)^{-2}\leq\frac{4}{\alpha^2B^2\tau(t)}.
% \end{equation}
noting that on $G(t)$, $\ell_{\tau(t),i_*}\ge \mu_{i_*}/2$, with $v_B = \frac{\mu_{i_*}\alpha^2 B^2}{16}$, we have
% \begin{equation}\label{eq:lower-return-lower}
%     D_{\tau(t)+1}
%     \ge B\tau(t)+\frac{\alpha^2B^2\tau(t)}{4}\frac{\mu_{i_*}}{2}
%     \ge 2v_B\tau(t)
%     \ge v_B(\tau(t)+1).
% \end{equation}
\[
D_{\tau(t)+1}
    \ge B\tau(t)+\frac{\alpha^2B^2\tau(t)}{4}\frac{\mu_{i_*}}{2}
    \ge 2v_B\tau(t)
    \ge v_B(\tau(t)+1).
\]
\paragraph{Proof of item (3)}
Finally, we will prove item (3). We will first decompose the probability according to the value of $\sigma(t)$. Then we control each terms by the tower property. 

With $\tau(t)$ defined in \eq{taut} and $G$ defined in \eq{G(t)},
we have $\{ \tau(t)<+\infty \} = G$.
It follows that
\begin{align*}
    \P(\tau(t)<+\infty\mid \mathscr{F}_{t-1})
    & = \P(G(t)\mid \mathscr{F}_{t-1}) \\
    & =\P(H(t)\cap\{\ell_{\sigma(t),i_*}\ge\mu_{i_*}/2\} \cap \{ v_B\le D_t/t\le V_B \} \mid \mathscr{F}_{t-1}) \\
    & = \P(H(t)\cap\{\ell_{\sigma(t),i_*}\ge\mu_{i_*}/2\}\mid \mathscr{F}_{t-1}) \1_{ \{ v_B\le \frac{D_t}{t}\le V_B \} },
\end{align*}
where the last equality is because $D_t$ is measurable in $\mathscr{F}_{t-1}$.
% Then by the definition of $\tau(t)$ in \eq{taut}, one can see that 
% % if $v_B\leq\frac{D_t}{t}\leq V_B$, then
% \[
% \P(\tau(t)<+\infty\mid \mathscr{F}_{t-1}) =\P(G(t)\mid \mathscr{F}_{t-1})=\P(H(t)\cap\{\ell_{\sigma(t),i_*}\ge\mu_{i_*}/2\}\mid \mathscr{F}_{t-1}).
% \]
In the subsequent proof, we always assume $ \{ v_B\le \frac{D_t}{t}\le V_B \}$ happens such that we could omit the indicator function to simplify the notation.
Note that
\[
\P(H(t)\cap\{\ell_{\sigma(t),i_*}\ge\mu_{i_*}/2\}\mid \mathscr{F}_{t-1})=\sum_{k=m_{-}(t)+1}^{m_{+}(t)}\P(\{\sigma(t)=k\}\cap\{\ell_{\sigma(t),i_*}\ge\mu_{i_*}/2\}\mid \mathscr{F}_{t-1}).
\]
By the definition of $\sigma(t)$ in \eq{sigmat}, for any $k\in(m_{-}(t),m_{+}(t)]$, we have
\[
\begin{aligned}
    &\P(\{\sigma(t)=k\}\cap\{\ell_{\sigma(t),i_*}\ge\mu_{i_*}/2\}\mid \mathscr{F}_{t-1})\\=&\E\left[\1_{\{\sigma(t)=k\}}\P(\ell_{\sigma(t),i_*}\ge\mu_{i_*}/2\mid \mathscr{F}_{k-1},I_k)\mid \mathscr{F}_{t-1}\right]\\   
    =&\E\left[\1_{\{\sigma(t)=k\}}\P(\ell_{k,i_*}\ge\mu_{i_*}/2)\mid \mathscr{F}_{t-1}\right].
\end{aligned}
\]
Then, since $0\leq \ell_{k,i_*}\leq 1$, by Lemma \ref{lem:half-mean-lower-tail}, we have
\[
\P(\ell_{k,i_*}\ge\mu_{i_*}/2)\ge \mu_{i_*}/2,
\]
hence, 
\[
\P(\{\sigma(t)=k\}\cap\{\ell_{\sigma(t),i_*}\ge\mu_{i_*}/2\}\mid \mathscr{F}_{t-1})\ge \frac{\mu_{i_*}}{2}\P(\sigma(t)=k\mid \mathscr{F}_{t-1}),
\]
and then
\begin{equation}\label{mid}
    \P(H(t)\cap\{\ell_{\sigma(t),i_*}\ge\mu_{i_*}/2\}\mid \mathscr{F}_{t-1})\ge\sum_{k=m_{-}(t)+1}^{m_{+}(t)}\P(\{\sigma(t)=k\}\mid \mathscr{F}_{t-1})\mu_{i_*}/2.
\end{equation}
% Besides, by the definition of $\sigma(t)$ and the tower property, we have
% \[
% \P(\{\sigma(t)=k\}\mid \mathscr{F}_{t-1})=\E\left[\prod_{s=t}^{k-1}(1-p_{s,i_*})p_{k,i_*}\mid \mathscr{F}_{t-1}\right].
% \]
% Note that \eq{ptauup} also holds for any $s\in[t,k)$, then using that $1-x\ge e^{-2x}$ for $0\leq x\leq 1/2$, we have
% \[
% \prod_{s=t}^{k-1}(1-p_{s,i_*})\ge \prod_{s=t}^{k-1}\left(1-\frac{4}{\alpha^2B^2s}\right)\ge \exp\left(-\sum_{s=t}^{k-1}\frac{8}{\alpha^2B^2s}\right)\ge \exp\left(-8-\frac{8\log(\frac{k}{t})}{\alpha^2B^2}\right),
% \]
% where we applied Lemma \ref{lem:harmonic-log} in the last inequality. Since
% \[
% \frac{k}{t}\leq \frac{m_{+}(t)}{t}\leq \frac{V_B}{2B}\leq 2\alpha^2 B,
% \]
% then when $B>16$, we have
% \[
% \frac{8\log(\frac{k}{t})}{\alpha^2B^2}\leq \frac{16\alpha^2 B}{\alpha^2B^2}\leq 1.
% \]
% Then combined with \eq{ptaulow}, we have
% \[
% % \P(\{\sigma(t)=k\}\mid \mathscr{F}_{t-1})=
% \prod_{s=t}^{k-1}(1-p_{s,i_*})p_{k,i_*}\ge e^{-9} p_{k,i_*}\ge \frac{e^{-9}}{3\alpha^2B^2k}.
% \]
% Then, by \eq{mid}, we have
% \[
% \P(H(t)\cap\{\ell_{\sigma(t),i_*}\ge\mu_{i_*}/2\}\mid \mathscr{F}_{t-1})\ge \sum_{k=m_{-}(t)+1}^{m_{+}(t)}\frac{e^{-9}}{3\alpha^2B^2k}\frac{\mu_{i_*}}2.
% \]
% Again, by Lemma \ref{lem:harmonic-log}, we have
% \[
% \P(H(t)\cap\{\ell_{\sigma(t),i_*}\ge\mu_{i_*}/2\}\mid \mathscr{F}_{t-1})\ge \frac{\mu_{i_*}}2\frac{e^{-9}}{3\alpha^2B^2}\log\left(\frac{m_{+}(t)}{m_{-}(t)}\right)>\frac{\mu_{i_*}}2\frac{e^{-9}\log(3/2)}{3\alpha^2B^2}.
% \]
% This completes our proof.
Then we will compute each term in the right-hand side by the tower property. Let
\[
A_k:=\{\sigma(t)\ge k\}
=
\{I_s\neq i_*,\ s=t,\ldots,k-1\},
\qquad k\ge t,
\]
with the convention \(A_t=\Omega\), then one can see that for any \(k\in\{m_-(t)+1,\ldots,m_+(t)\}\)
\begin{equation}\label{akik}
    \{\sigma(t)=k\}=A_k\cap\{I_k=i_*\}.
\end{equation}
With $A_k \in \mathscr F_{k-1}$,
the tower property gives
\[
\P(\sigma(t)=k\mid \mathscr F_{t-1})
= \E\!\left[
\1_{A_k} \E[ \1_{ \{ I_k = i_* \} } \,\middle|\, \mathscr F_{k-1}  ]
\,\middle|\,\mathscr F_{t-1}
\right]
= \E\!\left[
\1_{A_k}p_{k,i_*}
\,\middle|\,\mathscr F_{t-1}
\right].
\]
On \(A_k\), since $k\ge m_{-}(t)$, the lower bound \eq{3b} holds; namely,
\[
p_{k,i_*}\ge \frac{1}{3\alpha^2B^2k}.
\]
Therefore,
\begin{equation}\label{step1}
    \P(\sigma(t)=k\mid \mathscr F_{t-1})
\ge
\frac{1}{3\alpha^2B^2k}
\P(A_k\mid \mathscr F_{t-1}).
\end{equation}
Hence, it suffices to lower bound
\(\P(A_k\mid \mathscr F_{t-1})\). For \(s\ge t\), by the tower property,
\[
\P(A_{s+1}\mid \mathscr F_{t-1})
=
\E\!\left[
\1_{A_s}\P(I_s\neq i_*\mid \mathscr F_{s-1})
\,\middle|\,\mathscr F_{t-1}
\right]
=
\E\!\left[
\1_{A_s}(1-p_{s,i_*})
\,\middle|\,\mathscr F_{t-1}
\right].
\]
On the event \(A_s\), since $t\leq s\leq k$, by \eq{b},
% similar to \eq{ptauup}, because \eq{b} is also true for $t\leq s\leq k$, we have
\[
p_{s,i_*}\le \frac{4}{\alpha^2B^2s}.
\]
Therefore,
\[
\P(A_{s+1}\mid \mathscr F_{t-1})
\ge
\left(1-\frac{4}{\alpha^2B^2s}\right)
\P(A_s\mid \mathscr F_{t-1}).
\]
Iterating this inequality gives, for every \(k\in[t,m_+(t)]\),
\[
\P(A_k\mid \mathscr F_{t-1})
\ge
\prod_{s=t}^{k-1}
\left(1-\frac{4}{\alpha^2B^2s}\right).
\]
Since \(1-a\ge e^{-2a}\) for \(0\le a\le 1/2\), and since (see Appendix \ref{verify2})
\begin{equation}\label{eq:lower-product-small}
    \frac{4}{\alpha^2B^2s}\le \frac12,
    \qquad s\ge t,
\end{equation} we obtain
\[
\prod_{s=t}^{k-1}
\left(1-\frac{4}{\alpha^2B^2s}\right)
\ge
\exp\left(
-\sum_{s=t}^{k-1}\frac{8}{\alpha^2B^2s}
\right).
\]
By Lemma~\ref{lem:harmonic-log}, we have
\[
\P(A_k\mid \mathscr F_{t-1})
\ge
\exp\left(
-\frac{8}{\alpha^2B^2}
\left(1+\log\frac{k}{t}\right)
\right).
\]
Since
% \begin{equation}\label{eq:lower-k-over-t}
%     \frac{k}{t}
%     \le
%     \frac{m_+(t)}{t}
%     \le
%     \frac{V_B}{2B}
%     \le
%     2\alpha^2B,
% \end{equation}
\[
\frac{k}{t}
    \le
    \frac{m_+(t)}{t}
    \le
    \frac{V_B}{2B}
    \le
    2\alpha^2B,
\]
we have, for \(B\) sufficiently large (see Appendix \ref{verify2}),
\begin{equation}\label{eq:lower-exp-control}
    \frac{8}{\alpha^2B^2}
    \left(1+\log\frac{k}{t}\right)\leq \frac{8}{\alpha^2B^2}
    \left(1+2\alpha^2B\right)
    \le 9,
\end{equation}
where we used that $\log(a)<a$ for any $a>0$.
Consequently,
\begin{equation}\label{step2}
    \P(A_k\mid \mathscr F_{t-1})\ge e^{-9}.
\end{equation}

Then, putting \eq{step1} and \eq{step2} together, we have
\[
\P(\sigma(t)=k\mid \mathscr F_{t-1})
\ge
\frac{e^{-9}}{3\alpha^2B^2k}.
\]

% Now fix \(k\in\{m_-(t)+1,\ldots,m_+(t)\}\). Since
% \[
% \{\sigma(t)=k\}=A_k\cap\{I_k=i_*\},
% \]
% another application of the tower property gives
% \[
% \P(\sigma(t)=k\mid \mathscr F_{t-1})
% =
% \E\!\left[
% \1_{A_k}p_{k,i_*}
% \,\middle|\,\mathscr F_{t-1}
% \right].
% \]
% On \(A_k\), the lower bound \eq{ptaulow} holds; namely,
% \[
% p_{k,i_*}\ge \frac{1}{3\alpha^2B^2k}.
% \]
% Therefore,

Summing over \(k=m_-(t)+1,\ldots,m_+(t)\), and using \eq{mid}, yields
\[
\P\!\left(
H(t)\cap
\left\{\ell_{\sigma(t),i_*}\ge \mu_{i_*}/2\right\}
\,\middle|\,
\mathscr F_{t-1}
\right)
\ge
\frac{\mu_{i_*}}2
\frac{e^{-9}}{3\alpha^2B^2}
\sum_{k=m_-(t)+1}^{m_+(t)}
\frac1k.
\]
Finally, by Lemma~\ref{lem:harmonic-log} and noting that, when \(B\) is large enough (see Appendix \ref{verify2}),
\begin{equation}\label{eq:lower-m-ratio}
    \frac{m_+(t)}{m_-(t)+1}>\frac54,
\end{equation}
we conclude that
\[
\P\!\left(
H(t)\cap
\left\{\ell_{\sigma(t),i_*}\ge \mu_{i_*}/2\right\}
\,\middle|\,
\mathscr F_{t-1}
\right)
\ge
\frac{\mu_{i_*}}2
\frac{e^{-9}\log(5/4)}{6\alpha^2B^2}.
\]
This completes the proof. We will verify that $t$ and $B$ defined in the conditions satisfy all the constrains above in Appendix \ref{verify2}.

\section{Important Facts}\label{important}

\begin{Lem}[Lemma A.4 in \cite{zhan2026lastiterateanalysesftrl12tsallis}]\label{pti-2}    
    Recall the definition of $p_{t,i}$ in \eq{ptdefi} and that $\underline{\hat{L}}_{t}=\hat{L}_t-\min _{i\in\mathcal{I}} \hat{L}_{t,i}\mathbf{1}$.
    For any $t\ge 1$, there exists $\nu>-\min\limits_{1\leq i\leq d}\hat{L}_{t, i}$ such that
    \[
    p_{t, i}=4\left(\eta_t\left(\hat{L}_{t, i}+\nu\right)\right)^{-2}=4\left(\eta_t\underline{\hat{L}}_{t, i}+2p_{t,i_{t}^*}^{-\frac{1}{2}}\right)^{-2}.
    \]
%     \begin{proof}
%         By the definition in \eq{ptdefi}, we have $p_t=\phi(-\eta_t\underline{\hat{L}}_t)$, then by Lemma \ref{exist} (with $f(x)=-4\sqrt{x}$), there exists $\nu>-\min\limits_{1\leq i\leq d}\hat{L}_{t, i}$ such that
%         \[
%         p_{t,i}=\phi_i(-\eta_t\underline{\hat{L}}_t)=4\left(\eta_t\left(\hat{L}_{t, i}+\nu\right)\right)^{-2}.
%         \]
%         From this, we also have
% $
% \nu=-\hat{L}_{t,i_{t}^*}+2\eta_t^{-1}p_{t,i_{t}^*}^{-\frac{1}{2}},
% $
% so for any $1\leq i\leq d$, we have
% \[
%     p_{t,i}=4\left(\eta_t\left(\hat{L}_{t, i}-\hat{L}_{t,i_{t}^*}\right)+2p_{t,i_{t}^*}^{-\frac{1}{2}}\right)^{-2}=4\left(\eta_t\underline{\hat{L}}_{t, i}+2p_{t,i_{t}^*}^{-\frac{1}{2}}\right)^{-2}.
% \]
%     \end{proof}
\end{Lem}

\begin{Lem}[Lemma A.5 in \cite{zhan2026lastiterateanalysesftrl12tsallis}]\label{4-2}
    Recall the definition of $p_{t,i}$ in \eq{ptdefi} and that $\underline{\hat{L}}_{t}=\hat{L}_t-\min _{i\in\mathcal{I}} \hat{L}_{t,i}\mathbf{1}$.
    For any $t\ge 1$, we have
    \[
    4\left(\eta_t\underline{\hat{L}}_{t,i}+2\sqrt{d}\right)^{-2}\leq p_{t,i}\leq 4\left(\eta_t\underline{\hat{L}}_{t,i}+2\right)^{-2}.
    \]
    In particular, for $\eta_t = \frac{\alpha}{\sqrt{t}}$ and $D_t = \hat L_{t,i_*}- \min_{i\in\mathcal I,\ i\neq i_*}\hat L_{t,i}$, we have $\underline{\hat{L}}_{t,i_*}=D_t^+$ and consequently $p_{t, i_*} \ge \Big( \frac{\alpha D_t^+}{2 \sqrt{t}} +  \sqrt{d}\Big)^{-2}$.
    \begin{proof}
        By \eq{pti}, for any $1\leq i\leq d$, we have 
\[
    p_{t,i}=4\left(\eta_t\underline{\hat{L}}_{t, i}+2p_{t,i_{t}^*}^{-\frac{1}{2}}\right)^{-2}\leq 4\left(\eta_t\underline{\hat{L}}_{t,i}+2\right)^{-2}.
\]
Similarly, since $p_{t,i_{t}^*}=\max\limits_{1\leq i'\leq d}p_{t,i'}\ge 1/d$, then we have
\[
p_{t,i}\ge 4\left(\eta_t\underline{\hat{L}}_{t,i}+2\sqrt{d}\right)^{-2}.
\]
    \end{proof}
\end{Lem}

\begin{Lem}\label{simd}
    Recall that
\[
N_t := \min_{i\in\mathcal I,\ i\neq i_*}\hat L_{t,i},
\qquad
D_t := \hat L_{t,i_*}-N_t.
\]
    For any $t\ge 1$, 
    % we have
    % \[
    % N_{t+1}\ge N_t,
    % \]
    % and
    \[
    -(N_{t+1}-N_{t})\leq D_{t+1}-D_t\leq \hat{\ell}_{t,i_*},
    \]
    and the equality on the left-hand side holds if $I_t\ne i_{*}$.
\end{Lem}
\begin{proof}
    By the definition of $N_t$, we have
    \[
    N_{t+1} := \min_{i\in\mathcal I,\ i\neq i_*}\hat L_{t+1,i}\ge  \min_{i\in\mathcal I,\ i\neq i_*}\hat L_{t,i}=N_t,
    \]
    and hence,
    \[
    D_{t+1}=\hat{L}_{t+1,i_*}-N_{t+1}=\hat{\ell}_{t,i_*}+ \hat{L}_{t,i_*}-N_{t+1}\leq \hat{\ell}_{t,i_*}+ \hat{L}_{t,i_*}-N_{t}=\hat{\ell}_{t,i_*}+D_t.
    \]
    Similarly,
    \[
    D_{t+1}=\hat{\ell}_{t,i_*}+ \hat{L}_{t,i_*}-N_{t+1}\ge \hat{L}_{t,i_*}-N_{t+1}=D_t+N_t-N_{t+1}.
    \]
\end{proof} 

\begin{Lem}
\label{lem:block_deterministic_control}
Recall that
\[
N_t := \min_{i\in\mathcal I,\ i\neq i_*}\hat L_{t,i},
\qquad
D_t := \hat L_{t,i_*}-N_t.
\]

Let $L\ge 1$. Suppose that $4dL\le \sqrt{dt}/(4\alpha)$, then the following statements hold.

\begin{enumerate}[(1)]

\item If $D_t \ge -\sqrt{dt}/\alpha$, then for every integer $m=t,\dots,t+L-1$,
\[
N_{m+1}-N_m \le 4d,
\]
and
\[
 D_{m+1}-D_m\ge -4d.
\]

\item If $D_t\leq 0$, then for every integer $m=t,\dots,t+L-1$,
\[
 D_{m+1}-D_m\leq 4d.
\]
\end{enumerate}
\end{Lem}

\begin{proof}
First, we prove item (1) by inductively showing that
\begin{equation}\label{pre}
    D_{t+j} \ge -\frac{\sqrt{dt}}{\alpha}-4dj
\end{equation}
over the block
$j=0,1,\dots,L$. Meanwhile, item (1) will also be proved in this induction. 

For $j=0$, it follows directly from the assumption $D_t\ge -\sqrt{dt}/\alpha$. Now assume that there exists an integer $j$ with $0\le j\le L-1$ such that
\(
 -D_{t+j} \le \frac{\sqrt{dt}}{\alpha}+4dj
\). Let \(
m:=t+j.
\)
We first prove that
\begin{equation}\label{nnk}
    N_{m+1}-N_m \le 4d.
\end{equation}

By the definition of $N_t$, there exists $i_0\ne i_*$ such that $N_m=\hat{L}_{m,i_0}$, then clearly,
\[
N_{m+1}-N_m\leq \hat{L}_{m+1,i_0}-N_m=\hat{L}_{m+1,i_0}-\hat{L}_{m,i_0}=\hat{\ell}_{m,i_0}\leq\frac{1}{p_{m,i_0}},
\]
where we used the definition of $\hat{\ell}_t$ in \eq{eq:loss-vector-est}. By Lemma \ref{4-2}, to obtain \eq{nnk}, it suffices to show that
\[
    \underline{\hat{L}}_{m,i_0}\leq\frac{2\sqrt{m}}{\alpha}\left(\sqrt{4d}-\sqrt{d}\right)=\frac{2\sqrt{dm}}{\alpha}.
\]

We distinguish two cases.

\medskip
\noindent
\textbf{Case 1: $D_m\ge 0$.}
In this case $\hat L_{m,i_*}\ge N_m$. Then 
\[
\hat L_{m,i_0}=N_m=M_m,
\]
where $M_m=\min_{i'\in\mathcal I}\hat L_{m,i'}$. Hence $\underline{\hat{L}}_{m,i_0}=0$.

\medskip
\noindent
\textbf{Case 2: $D_m<0$.}
In this case $\hat L_{m,i_*}<N_m$, hence $\hat L_{m,i_*}=M_m$ and
\[
\underline{\hat{L}}_{m,i_0}=\hat{L}_{m,i_0}-M_m=N_m-\hat L_{m,i_*}=-D_m.
\]
By the induction hypothesis,
\begin{equation}\label{ajkl}
    -D_m \le \frac{\sqrt{dt}}{\alpha}+4dj \le \frac{\sqrt{dt}}{\alpha}+4dL.
\end{equation}
Since $4dL\le \sqrt{dt}/(4\alpha)$, we have
\begin{equation}\label{ajm}
    -D_m \le \frac{\sqrt{dt}}{\alpha}+\frac{\sqrt{dt}}{4\alpha}=\frac{5\sqrt{dt}}{4\alpha}<\frac{2\sqrt{dm}}{\alpha},
\end{equation}
because $m=t+j\ge t$. 

Combining the two cases proves \eq{nnk}. 

Then, by Lemma \ref{simd} and \eq{nnk},
\begin{equation}\label{ddk}
D_{m+1}-D_m\ge-(N_{m+1}-N_m)\ge-4d.
\end{equation}
Therefore,
\[
-D_{m+1} \le -D_m + 4d\le \frac{\sqrt{dt}}{\alpha}+4d(j+1).
\]
This closes the induction and proves \eq{pre}. Also, one should note that by \eq{nnk} and \eq{ddk}, item (1) has already been proved in this induction.

Then we prove item 2. We similarly show by induction on
$j\in\{0,1,\dots,L\}$ that
\[
D_{t+j}\le D_t+4dj.
\]
The case $j=0$ is trivial. Now suppose this holds for some $j<L$, and let
\(
m:=t+j.
\)
Then since $D_t\leq 0$, similar to \eq{ajkl} and \eq{ajm}, we have
\[
D_m \le D_t+4dj \le 4dj \le 4dL\leq \frac{\sqrt{dt}}{4\alpha}<\frac{2\sqrt{dm}}{\alpha}.
\]
Hence, by Lemma \ref{4-2}, we have
\[
\hat{\ell}_{m,i_*}\leq \frac{1}{p_{m,i_*}}\leq\left(\sqrt{d}+\frac{\alpha D_m^+}{2\sqrt{m}}\right)^2\leq4d.
\]
Then by Lemma \ref{simd},
\[
D_{m+1}-D_m\leq \hat{\ell}_{m,i_*}\leq 4d.
\]
Hence, $D_{m+1}\leq D_t+4d(j+1)$, which closes the induction and simultaneously proves item (2).
\end{proof}

\begin{Lem}\label{l=1}
    Recall that
\[
N_t := \min_{i\in\mathcal I,\ i\neq i_*}\hat L_{t,i},
\qquad
D_t := \hat L_{t,i_*}-N_t.
\]
     If $t\ge 144d^2$ and $D_t>0$, then for any $t\leq t'\leq \left\lfloor\frac{D_t}{4d}\right\rfloor$, we have
    \[
    D_{t'}\ge D_t-4d(t'-t)\ge 0.
    \]
\end{Lem}
\begin{proof}
    We prove this result by induction. It is clearly true when $t'=t$ by the definition. Assume that it is true for $k$ such that $t\leq k<\left\lfloor D_t/(4d)\right\rfloor$, then
    \[
    D_{k}\ge D_t-4d(k-t)\ge 0.
    \]
Since $k\ge t\ge 144d^2$, $d\ge2$ and $0<\alpha<1$ imply
\[
4d\le\frac{\sqrt{dk}}{4\alpha}.
\]
By Lemma \ref{lem:block_deterministic_control} (1) with $L=1$ and $D_k\ge0$, we have
\[
N_{k+1}-N_k\leq 4d.
\]
Thus, by Lemma \ref{simd},
\[
\begin{aligned}
    D_{k+1}\ge D_k-(N_{k+1}-N_k)\ge D_k-4d\ge D_t-4d(k+1-t).
\end{aligned}
\]
This completes the induction.
\end{proof}

\begin{Lem}[Lemma 4.1 in \cite{zhan2026lastiterateanalysesftrl12tsallis}]\label{dt}
    For Algorithm \ref{alg: FTRL}, if $0<\alpha<1$, then there exists $C_{\alpha}>0$ depending only on $\alpha$ such that for any $t\ge 1$
    \[
    \E\left[\underline{\hat{L}}_{t,i_*}^2\right]\leq C_{\alpha} dt.
    \]
\end{Lem}

\begin{Lem}\label{lem:pfposi}
    For the event $F$ defined by \eq{F} in Lemma \ref{start}, we have $\P(F)>0$.
\end{Lem}
\begin{proof}
    For any $s\ge 1$, by Lemma \ref{lem:half-mean-lower-tail}, applied to
\(\ell_{s,i_*}\in[0,1]\),
\begin{equation}\label{eq:ellsi}
    \mathbb P\!\left(\ell_{s,i_*}\ge \frac{\mu_{i_*}}{2}\right)
    \ge \frac{\mu_{i_*}}{2}>0,
\end{equation}
Moreover, by \eq{pti} and \(d\ge2\), for every \(s\ge1\) and every arm \(i\),
\begin{equation}\label{eq:0psi1}
    0<p_{s,i}<1 \qquad \text{a.s.}
\end{equation}

We first show that \(\mathbb P(E(T'-1))>0\). Set \(A_0:=\Omega\) and, for \(1\le s\le T'-1\), define
\[
    A_s
    :=A_{s-1}\cap\left\{I_s=i_*,\;\ell_{s,i_*}\ge \frac{\mu_{i_*}}{2}\right\}.
\]
Then \(A_{T'-1}=E(T'-1)\). We prove by induction that \(\mathbb P(A_s)>0\) for all \(0\le s\le T'-1\).
The case \(s=0\) is trivial. Suppose that \(\mathbb P(A_{s-1})>0\).
Since \(A_{s-1}\in\mathscr F_{s-1}\) and \(p_{s,i_*}\) is \(\mathscr F_{s-1}\)-measurable, by \eq{eq:ellsi}, the tower property gives
\begin{align*}
    \mathbb P(A_s)
    &=
    \mathbb E\!\left[
        \1_{A_{s-1}}
        \mathbb P\!\left(
            I_s=i_*,
            \ \ell_{s,i_*}\ge \frac{\mu_{i_*}}{2}
            \mid \mathscr F_{s-1}
        \right)
    \right] \\
    &\ge
    \frac{\mu_{i_*}}{2}
    \mathbb E\!\left[
        \1_{A_{s-1}}p_{s,i_*}
    \right].
\end{align*}

Now \(\1_{A_{s-1}}p_{s,i_*}\) is a nonnegative random variable. Moreover, since
\(\mathbb P(A_{s-1})>0\) and \(p_{s,i_*}>0\) almost surely by \eq{eq:0psi1}, we have
\[
    \mathbb P\!\left(
        \1_{A_{s-1}}p_{s,i_*}>0
    \right)
    =
    \mathbb P(A_{s-1})>0.
\]
Hence, by Lemma \ref{lem:positive-a.s.-positive-mean},
\[
    \mathbb E\!\left[
        \1_{A_{s-1}}p_{s,i_*}
    \right]>0.
\]
Therefore,
\[
    \mathbb P(A_s)>0.
\]
This completes the induction and proves
\[
    \mathbb P(E(T'-1))=\mathbb P(A_{T'-1})>0.
\]

Next retain \(A_{T'-1}:=E(T'-1)\), and for \(T'\le s\le T''\), define
\[
    A_s:=A_{s-1}\cap\{I_s\ne i_*\}.
\]
We again prove by induction that
\[
    \mathbb P(A_s)>0,
    \qquad T'-1\le s\le T''.
\]
The initial step follows from \(\mathbb P(A_{T'-1})=\mathbb P(E(T'-1))>0\).
Suppose that \(\mathbb P(A_{s-1})>0\) for some \(T'\le s\le T''\).
By the tower property,
\[
    \mathbb P(A_s)
    =
    \mathbb E\!\left[
        \1_{A_{s-1}}
        \mathbb P(I_s\ne i_*\mid \mathscr F_{s-1})
    \right]
    =
    \mathbb E\!\left[
        \1_{A_{s-1}}(1-p_{s,i_*})
    \right].
\]
The random variable \(\1_{A_{s-1}}(1-p_{s,i_*})\) is nonnegative. Moreover, since
\(\mathbb P(A_{s-1})>0\) and \(1-p_{s,i_*}>0\) almost surely by \eq{eq:0psi1}, we have
\[
    \mathbb P\!\left(
        \1_{A_{s-1}}(1-p_{s,i_*})>0
    \right)
    =
    \mathbb P(A_{s-1})>0.
\]
Hence, similarly, by Lemma \ref{lem:positive-a.s.-positive-mean},
\[
    \mathbb P(A_s)>0.
\]
This completes the second induction. Consequently,
\[
   \mathbb P(F)=\mathbb P(A_{T''})>0.
\]
\end{proof}

\section{\texorpdfstring{Growth of $N_t$}{nt}}\label{growthofnt}
In this section, we will present the proof for Lemma \ref{ntgrow} and also introduce some other results relating to the growth of $N_t=\min_{i\in\mathcal I,\ i\neq i_*}\hat L_{t,i}$.
The whole section is based on the event \(
\{D_t\ge -\sqrt{dt}/\alpha\}
\)

To prove Lemma~\ref{ntgrow}, we need to derive a lower bound for $N_{t+L} - N_{t}$.
Since $N_{t}$ is the minimum of $\hat L_{t,i}$ over $\mathcal{I} \setminus \{  i_*\}$, it is difficult to deal with $N_{t}$ directly.
However, if we can find a component $\hat L_{t,i}$ of $\hat L_t$ close to $N_t$, then $\hat L_{t+L,i} - \hat L_{t,i}$ could serve as a bridge to obtain the lower bound for $N_{t+L} - N_{t}$.

To that end, we
fix $t$ and $L\ge 1$, and define 
\begin{equation}\label{mat}
    \mathcal{A}_t:=\left\{i\ne i_*\,:\,\hat{L}_{t,i}-N_t\leq 4dL\right\},
\end{equation}
which contains those arms in $\mathcal{I}\setminus \{i_*\}$ with estimated loss relatively close to $N_t$. The probabilities of choosing them are also larger and their IW estimators also have smaller variance. For each $i\in \mathcal{A}_t$, we define
\[
W_i:=\hat{L}_{t+L,i}-\hat{L}_{t,i}-L\mu_i.
% =\sum_{r=0}^{L-1}\left(\left(\hat{L}_{t+r+1,i}-\hat{L}_{t+r,i}\right)-\mu_i\right).
\]
Here we subtract $ L\mu_i$ from $\hat{L}_{t+L,i}-\hat{L}_{t,i}$ such that each $W_i$ has zero mean (recall that $\hat{L}_t = \sum_{s=1}^{t-1} \hat{\ell}_s$ with $\hat{\ell}_s$ defined in \eq{eq:loss-vector-est}).

The first lemma lower-bounds $N_{t+L}-N_t$ by an expression involving the smallest $W_i$ over $i\in\mathcal{A}_t$.
\begin{Lem}\label{ntlowerpre}
    Suppose that $4dL\le\sqrt{dt}/(4\alpha)$ and $D_t\ge-\sqrt{dt}/\alpha$. It holds that
    \[
    N_{t+L}-N_t \geq\left(\mu_{i_*}+\Delta\right)L+\min_{i\in\mathcal{A}_t}W_i.
    \]
\end{Lem}
\begin{proof}
Recall that $N_{t+L} = \min_{i\in\mathcal I,\ i\neq i_*}\hat L_{t+L,i}$.
We would like to restrict the minimization over $\mathcal{I} \setminus \{  i_*\}$ to $\mathcal{A}_t \subset \mathcal{I} \setminus \{  i_*\}$.
From the definition, $N_{t+L}=\hat{L}_{t+L,i_0}$ for some $i_0\ne i_*$. 
We show that $i_0\in \mathcal{A}_t$:
    Lemma \ref{lem:block_deterministic_control} (1) implies that $N_{t+L}-N_t\leq 4dL$; then one can see that
    \[
    \hat{L}_{t,i_0}\leq \hat{L}_{t+L,i_0}=N_{t+L}\leq N_t+ 4dL,
    \]
    which implies that $i_0\in \mathcal{A}_t$.
    It follows that $N_{t+L} = \min_{i\in\mathcal A_t}\hat L_{t+L,i} $
    
    Therefore,
    \[
    N_{t+L}-N_t=\min_{i\in\mathcal{A}_t}\hat L_{t+L,i}-N_{t}\ge \min_{i\in\mathcal{A}_t}\left(\hat L_{t+L,i}-\hat{L}_{t,i}\right),
    \]
    where we used that $\hat{L}_{t,i}\ge N_t$ for any $i\ne i_*$. Then recalling the definition of $W_i$ and applying that $\mu_i\ge \Delta+\mu_{i_*}$, we have
    \[
    N_{t+L}-N_t\ge\min_{i\in\mathcal{A}_t}\left(W_i+L\mu_i\right)\ge \left(\mu_{i_*}+\Delta\right)L+\min_{i\in\mathcal{A}_t}W_i.
    \]
\end{proof}

The second lemma bounds the probability that $W_i$ is smaller than $-x$. This is the probability needed to lower-bound $N_{t+L}-N_t$ through Lemma~\ref{ntlowerpre}.
% Hence, the problem is converted to analyzing the behavior of $\max _{i \in \mathcal{A}_t}\left|W_i\right|$. In fact, as explained in the beginning, 
Specifically, for every $x\ge0$, it gives the following bound.
\begin{Lem}\label{wisub}
    Suppose that $4dL\le\sqrt{dt}/(4\alpha)$ and $D_t\ge-\sqrt{dt}/\alpha$. It holds that for any $x\ge0$ and $i\in\mathcal{A}_t$,
    \[
    \mathbb{P}\left(-W_i\ge x\,\middle|\,\mathscr{F}_{t-1}\right)
    \le \exp\left(-\frac{x^2}{2(4dL+x/3)}\right).
    \]
\end{Lem}
\begin{proof}
    First, $W_i$ can be expressed as the summation of mean-zero random variables 
    \[
    W_i=\sum_{r=0}^{L-1}\left(\left(\hat{L}_{t+r+1,i}-\hat{L}_{t+r,i}\right)-\mu_i\right)=\sum_{r=0}^{L-1}\left(\hat{\ell}_{t+r,i}-\mu_i\right),
    \]
    where $\E\left[\hat{\ell}_{t+r,i}\,\middle|\, \mathscr{F}_{t+r-1}\right]=\mu_i$. We first show that for any $0\leq r\leq L-1$,
    \begin{equation}\label{trik}
        \hat{\ell}_{t+r,i}\leq 4d.
    \end{equation}
    Once \eq{trik} has been proved for every $0\le r<L$, the martingale differences $\mu_i-\hat{\ell}_{t+r,i}$ satisfy
    \[
    \begin{gathered}
    \mu_i-\hat{\ell}_{t+r,i}\le1,
    \qquad
    \left|\mu_i-\hat{\ell}_{t+r,i}\right|\le 4d,\\
    \E\left[(\mu_i-\hat{\ell}_{t+r,i})^2\,\middle|\,\mathscr{F}_{t+r-1}\right]
    \le \E\left[\hat{\ell}_{t+r,i}^2\,\middle|\,\mathscr{F}_{t+r-1}\right]
    \le 4d.
    \end{gathered}
    \]
    Lemma~\ref{lem:one-sided-freedman} is the one-sided Freedman inequality. For the fixed arm $i$, apply it with the shifted filtration $\mathscr{G}_r:=\mathscr{F}_{t+r-1}$, $X_{r+1}:=\mu_i-\hat{\ell}_{t+r,i}$, $b=1$, and $v=4dL$. Since $-W_i=\sum_{r=0}^{L-1}X_{r+1}$, the claimed bound follows.

    It remains to derive \eq{trik}.
    By the definition of $\hat{\ell}_{t+r,i}$ in \eq{eq:loss-vector-est} and Lemma \ref{4-2}
    \[
    \hat{\ell}_{t+r,i}\leq p_{t+r,i}^{-1} \leq  \left(\eta_{t+r}\underline{\hat{L}}_{t+r,i}+2\sqrt{d}\right)^{2}/4.
    \]
    Recall that $\eta_{t+r} = \alpha / \sqrt{t+r}$. Then to show \eq{trik}, it suffices to show that for any $0\leq r\leq L-1$,
\begin{equation}\label{t+r}
    \underline{\hat{L}}_{t+r,i}
%=\hat{L}_{t+r,i}-M_{t+r}
\leq \frac{2\sqrt{d(t+r)}}{\alpha}.
\end{equation}
To this end, we will establish that for any $0\leq r\leq L-1$,
\begin{equation}\label{klr}
    \hat{L}_{t+r,i}-N_{t+r}\leq 4d(L+r).
\end{equation}
To see why \eq{klr} yields \eq{t+r}, recall that $M_t=\min _{i\in\mathcal{I}} \hat{L}_{t,i}$, then one can see that
\[
N_t-M_t=\begin{cases}
0, & \text{if $i_*\notin\argmin_{i\in\mathcal{I}}\hat{L}_{t,i}$},\\[4pt]
N_t-\hat{L}_{t,i_*}=-D_t>0, & \text{else}.
\end{cases}
\]
In other words, $N_{t}-M_{t}=(-D_{t})^+$ for any $t\ge 1$. Therefore, by Lemma \ref{lem:block_deterministic_control} (1), we have
\begin{equation}\label{tkr}
    N_{t+r}-M_{t+r}=(-D_{t+r})^+\leq (-D_{t})^++4dr\leq \frac{\sqrt{dt}}{\alpha}+4dr.
\end{equation}
Then combining \eq{klr} and \eq{tkr}, we have
\[
\underline{\hat{L}}_{t+r,i}=\hat{L}_{t+r,i}-M_{t+r}
\leq \frac{\sqrt{dt}}{\alpha}+12dL
\leq \frac{7\sqrt{dt}}{4\alpha}
<\frac{2\sqrt{d(t+r)}}{\alpha}.
\]

Finally, we use induction to show \eq{klr}. It is true for $r=0$ by the definition of $\mathcal{A}_t$ in \eq{mat}. Suppose that it is also true for some $0\leq r<L-1$, then by the argument above, we have $\hat{\ell}_{t+r,i}\leq 4d$. Hence,
\[
\hat{L}_{t+r+1,i}-N_{t+r+1}\leq \hat{\ell}_{t+r,i}+\hat{L}_{t+r,i}-N_{t+r}\leq 4d(L+r+1).
\]
This closes the induction and also completes our proof.
\end{proof}

We can now prove Lemma~\ref{ntgrow} using the bound on $\P(-W_i\ge x\mid\mathscr{F}_{t-1})$ in Lemma~\ref{wisub}.
\begin{proof}[Proof of Lemma \ref{ntgrow}]
By Lemma~\ref{ntlowerpre},
\[
\left(\mu_{i_*}+\Delta\right)L-(N_{t+L}-N_t)
\le \max_{i\in\mathcal{A}_t}(-W_i).
\]
Conditionally on $\mathscr{F}_{t-1}$, the set $\mathcal{A}_t$ is fixed and has at most $d-1$ elements. Applying Lemma~\ref{lem:max_one_sided_tail_moments} to $Z_i=-W_i$ with $v=4dL$, and using $d\ge2$, gives
\[
\E\left[\left(\max_{i\in\mathcal{A}_t}(-W_i)\right)^+\,\middle|\,\mathscr{F}_{t-1}\right]
\le C\left(\sqrt{4dL\log(d)}+\log(d)\right),
\]
which proves the claim because $\E[\max_{i\in\mathcal{A}_t}(-W_i)\mid\mathscr{F}_{t-1}]$ is no larger than the left-hand side.
\end{proof}

The next lemma records the bounds on $T_t^+$ and $\E[e^{\gamma T_t}\mid\mathscr{F}_{t-1}]$ that will be used below.
\begin{Lem}\label{tt2}
Suppose that $4dL\le\sqrt{dt}/(4\alpha)$ and $D_t\ge-\sqrt{dt}/\alpha$. There exists a universal constant $C>0$ such that
\begin{enumerate}[(1)]
\item
\[
\E\left[T_t^+\,\middle|\,\mathscr{F}_{t-1}\right]
\le C\left(\sqrt{4dL\log(d)}+\log(d)\right),
\]
and
\[
\E\left[(T_t^+)^2\,\middle|\,\mathscr{F}_{t-1}\right]
\le C\left(4dL\log(d)+(\log(d))^2\right).
\]
\item For every $0\le\gamma\le1/(4d)$,
\[
\E\left[e^{\gamma T_t}\,\middle|\,\mathscr{F}_{t-1}\right]
\le d\,e^{4dL\gamma^2}.
\]
\end{enumerate}
\end{Lem}
\begin{proof}
The proof of Lemma~\ref{ntgrow} gives
\[
T_t^+\le\left(\max_{i\in\mathcal{A}_t}(-W_i)\right)^+
\]
and Lemmas~\ref{wisub} and \ref{lem:max_one_sided_tail_moments} give both bounds in item~(1). For item~(2), the proof of Lemma~\ref{wisub} shows, for every $0\le r<L$, that
\[
\left|\mu_i-\hat{\ell}_{t+r,i}\right|\le 4d,
\qquad 0\le r\le L-1.
\]
The one-sided form of Lemma~\ref{lem:one-sided-freedman} would also be sufficient for item~(2), since for $\gamma\ge0$ only the upper tail of $-W_i$ matters. However, integrating its tail bound gives no better scale than $e^{O(\gamma^2dL)}$, whereas Lemma~\ref{bounded} gives the stated MGF bound directly.
Applying Lemma~\ref{bounded} conditionally to each increment and iterating with the tower property over the sum
\[
-W_i=\sum_{r=0}^{L-1}(\mu_i-\hat{\ell}_{t+r,i}),
\]
gives
\[
\E\left[e^{-\gamma W_i}\,\middle|\,\mathscr{F}_{t-1}\right]
\le e^{4dL\gamma^2},
\qquad 0\le\gamma\le1/(4d).
\]
Consequently,
\[
e^{\gamma T_t}
\le e^{\gamma\max_{i\in\mathcal{A}_t}(-W_i)}
\le\sum_{i\in\mathcal{A}_t}e^{-\gamma W_i},
\]
and summing over at most $d-1$ arms completes the proof.
\end{proof}

\section{\texorpdfstring{Computation related to $\hat L_{t+L,i_*}-\hat L_{t,i_*}$}{lt}}\label{comlt}
In this section, we will present some computations about $\hat L_{t+L,i_*}-\hat L_{t,i_*}$, which is important for computing $\E[g_{t+L}(D_{t+L})]/g_t(D_t)$ when proving \eq{lsuper}, or equivalently Lemma~\ref{formalsuper}, in Appendix~\ref{pflya}.
For the details, see \eq{xut}, \eq{ratiopo}, and \eq{ratione}.
From the definition of $g_t$ in \eq{gt2}, the computation is divided into two parts according to whether $D_t \ge 0$.

% In the following, we suppose that $D_t=x-\Delta t\ge 0$ for simplicity, because now $g_t(D_t)=x^2t^{-\theta}$. 
\subsection{The positive half-line}
For the positive half-line, $g_t(y)=t^{-\theta}(y+\Delta t)^2$, hence, one needs to control the (conditional) variance of $\hat L_{t+L,i_*}-\hat L_{t,i_*}$, as shown in \eq{xut} and \eq{ratiopo}.
From $\hat{L}_t = \sum_{s=1}^{t-1} \hat{\ell}_s$ and the unbiasedness of $\hat{\ell}_s$ defined in  \eq{eq:loss-vector-est},
that is
\[
\E\left[\left(\hat L_{t+L,i_*}-\hat L_{t,i_*}-L\mu_{i_*}\right)^2\,\middle|\,\mathscr{F}_{t-1}\right]=\sum_{r=0}^{L-1}\E\left[\left(\hat{\ell}_{t+r,i_*}-\mu_{i_*}\right)^2\,\middle|\,\mathscr{F}_{t-1}\right]\leq \sum_{r=0}^{L-1}\E\left[\hat{\ell}_{t+r,i_*}^2\,\middle|\,\mathscr{F}_{t-1}\right].
\]
Since $\ell_{t,i_*} \in [0,1]$, for any $t\ge 1$, 
\begin{equation}\label{secondell}
    \mathbb{E}\left[\hat{\ell}_{t,i_*}^2 \mid \mathscr{F}_{t-1}\right]=\mathbb{E}\left[\frac{A_{t,i_*}\ell_{t,i_*}^2}{p_{t,i_*}^2} \mid \mathscr{F}_{t-1}\right]\leq \mathbb{E}\left[\frac{A_{t,i_*}\ell_{t,i_*}}{p_{t,i_*}^2} \mid \mathscr{F}_{t-1}\right]=\frac{\mu_{i_*}}{p_{t,i_*}},
\end{equation}
Then by Lemma \ref{4-2} with $\eta_t = \alpha / \sqrt{t}$ and that $\underline{\hat{L}}_{t,i_*}=D_t^+$, we have
\begin{equation}\label{secondL}
    \E\left[\left(\hat L_{t+L,i_*}-\hat L_{t,i_*}-L\mu_{i_*}\right)^2\,\middle|\,\mathscr{F}_{t-1}\right]\leq \mu_{i_*}\sum_{r=0}^{L-1}\left(\frac{\alpha D_{t+r}^+}{2\sqrt{t+r}}+\sqrt{d}\right)^{2}.
\end{equation}
 For any $t\ge 1$ define
\begin{equation}\label{st}
    s_t:=\left(\frac{\alpha D_{t}^+}{2\sqrt{t}}+\sqrt{d}\right)^{2}.
\end{equation}
We will then turn to controlling $s_t$. First, we control the first term:
\begin{Lem}\label{onestep}
    Suppose that $D_t\ge 0$. Then $s_t$ defined in \eq{st} can be controlled as
    \[
    s_t \leq\left(\frac{\alpha^2}{4}+\frac{\alpha \sqrt{d}}{4 \Delta \sqrt{t}}+\frac{d}{\Delta^2 t}\right) \frac{x^2}{t},
    \]
    where $x = D_t + \Delta t$.
\end{Lem}
\begin{proof}
    By direct computation, we have
    \[
    \begin{aligned}
        s_t=&\left(\frac{\alpha (x-\Delta t)}{2\sqrt{t}}+\sqrt{d}\right)^{2}=\frac{\alpha^2(x-\Delta t)^2}{4t}+\frac{\alpha (x-\Delta t)\sqrt{d}}{\sqrt{t}}+d
    \end{aligned}
    \]
    For the first term,
    since $0\leq D_t = x-\Delta t\leq x$, we have
    \[
    \frac{\alpha^2(x-\Delta t)^2}{4t}\leq \frac{\alpha^2 x^2}{4t}.
    \]
    For the second term, note that
    \[
    x-\Delta t=\frac{\Delta t(x-\Delta t)}{\Delta t}\leq \frac{(\Delta t+x-\Delta t)^2}{4\Delta t}=\frac{x^2}{4\Delta t},
    \]
    hence,
    \[
    \frac{\alpha (x-\Delta t)\sqrt{d}}{\sqrt{t}}\leq \frac{x^2}{t}\frac{\alpha \sqrt{d}}{4 \Delta \sqrt{t}}.
    \]
    For the last term, since
    \[
    \frac{x^2}{t}\ge \frac{(\Delta t)^2}{t}\ge \Delta^2 t,
    \]
     we have
     \[
     d\leq \frac{x^2}{t}\frac{d}{\Delta^2 t}.
     \]
     Then it suffices to put everything together.
\end{proof}

To tackle the $L$ steps forward, we need the following recurrence relation:
\begin{Lem}\label{stre}
    For $s_t$ defined in \eq{st}, there exists $C_{\alpha}>0$ depending only on $\alpha$ such that for any $t\ge 1$,
    \[
    \mathbb{E}\left[s_{t+1} \mid \mathscr{F}_{t-1}\right] \leq\left(1+\frac{C_\alpha}{t}\right) s_t+\frac{\alpha^2}{4}.
    \]
    \begin{proof}
        By Lemma \ref{simd}, for any $t\ge 1$, we have
        \[
D_{t+1}^{+}-D_t^{+}\leq (D_{t+1}-D_t)^{+} \leq \hat{\ell}_{t,i_*},
        \]
        then, by the definition of $s_t$ in \eq{st}, we have
        \[
        \sqrt{s_{t+1}}=\sqrt{d}+\frac{\alpha D_{t+1}^{+}}{2 \sqrt{t+1}} \leq \sqrt{d}+\frac{\alpha\left(D_t^{+}+\hat{\ell}_{t,i_*}\right)}{2 \sqrt{t}}=\sqrt{s_t}+\frac{\alpha\hat{\ell}_{t,i_*}}{2 \sqrt{t}}.
        \]
        Taking the square yields
        \[
        s_{t+1} \leq s_t+\frac{\alpha \sqrt{s_t} \hat{\ell}_{t,i_*}}{\sqrt{t}}+\frac{\alpha^2}{4} \frac{\hat{\ell}_{t,i_*}^2}{t},
        \]
        then by $\mathbb{E}\left[\hat{\ell}_{t,i_*} \mid \mathscr{F}_{t-1}\right]=\mu_{i_*}$ and \eq{secondell}, we have
        \[
         \mathbb{E}\left[s_{t+1} \mid \mathscr{F}_{t-1}\right]\leq s_t+\frac{\alpha \mu_{i_*} \sqrt{s_t}}{\sqrt{t}}+\frac{\alpha^2 \mu_{i_*}}{4} \frac{s_t}{t},
        \]
        where we used that $\frac{1}{p_{t,i_*}}\leq s_t$ by Lemma \ref{pti-2} and the definition of $s_t$ in \eq{st}.
        % \[
        % \mathbb{E}\left[\hat{\ell}_{t,i_*}^2 \mid \mathscr{F}_{t-1}\right]=\mathbb{E}\left[\frac{A_{t,i_*}\ell_{t,i}^2}{p_{t,i_*}^2} \mid \mathscr{F}_{t-1}\right]\leq \mathbb{E}\left[\frac{A_{t,i_*}\ell_{t,i}}{p_{t,i_*}^2} \mid \mathscr{F}_{t-1}\right]=\frac{\mu_{i_*}}{p_{t,i_*}}\leq \mu_{i_*}s_t.
        % \]
        Finally, note that
        \[
        \frac{\alpha \mu_{i_*} \sqrt{s_t}}{\sqrt{t}}\leq \frac{s_t}t+\frac{\mu_{i_*}^2\alpha^2}4\leq \frac{s_t}t+\frac{\alpha^2}4,
        \]
        then combining the above together we have
        \[
        \mathbb{E}\left[s_{t+1} \mid \mathscr{F}_{t-1}\right]\leq\left(1+\frac{1+\alpha^2 \mu_{i_*} / 4}{t}\right) s_t+\frac{\alpha^2}{4}.
        \]
        It suffices to take $C_{\alpha}=1+\alpha^2/4$.
    \end{proof}
\end{Lem}

With the preliminary result above, we can now compute the $L$ steps forward:
\begin{Lem}\label{lstep}
    Suppose that $D_t \ge 0$.
    Then for $s_t$ defined in \eq{st}, there exists $C_{\alpha}>0$ depending only on $\alpha$ such that 
    % for any $t\ge 1$, if $D_t=x-\Delta t\ge 0$, then 
    \[
    \sum_{j=0}^{L-1} \mathbb{E}\left[s_{t+j} \mid \mathscr{F}_{t-1}\right] \leq \exp \left(\frac{C_\alpha L}{t}\right)\left[L\left(\frac{\alpha^2}{4}+\frac{\alpha \sqrt{d}}{4 \Delta \sqrt{t}}+\frac{d}{\Delta^2 t}\right) \frac{x^2}{t}+\frac{\alpha^2 L^2}{8}\right],
    \]
    where $x = D_t + \Delta t$.
\end{Lem}
\begin{proof}
    Define
    \[
    a_j:=\mathbb{E}\left[s_{t+j} \mid \mathscr{F}_{t-1}\right], \quad j \geq 0,
    \]
    then by Lemma \ref{stre} and the tower property, there exists $C_\alpha>0$ such that $u=\frac{C_{\alpha}}{t}$ and
    \begin{equation}\label{aj1}
        a_{j+1} \leq(1+u) a_j+\frac{\alpha^2}{4}, \quad j \geq 0.
    \end{equation}
    Iterating \eq{aj1}, we have
    \[
    a_j \leq(1+u)^j a_0+\frac{\alpha^2}{4} \sum_{r=0}^{j-1}(1+u)^r.
    \]
    For any $0\leq r\leq L-1$, we have
    \[
    (1+u)^r\leq e^{ru}\leq e^{Lu}.
    \]
    Hence,
    \[
    a_j \leq \exp (u L)\left(a_0+\frac{\alpha^2}{4} j\right),
    \]
    then summing over $0\leq j\leq L-1$, we have
    \[
    \sum_{j=0}^{L-1} a_j \leq \exp (u L)\left(L a_0+\frac{\alpha^2}{4} \sum_{j=0}^{L-1} j\right) \leq \exp (u L)\left(L a_0+\frac{\alpha^2 L^2}{8}\right).
    \]
    Finally, it suffices to apply Lemma \ref{onestep} to $a_0$.
\end{proof}
Under the conditions of Lemma~\ref{lstep}, $x \ge \Delta t$.
Then the first term $\frac{\alpha^2}{4}\frac{Lx^2}{t}$ on the right-hand side is the leading term. When $t$ is large enough, the other terms can be absorbed into $\mathcal{O}(\rho)\frac{Lx^2}{t}$, where $\rho=2-\frac{\alpha^2 \mu_{i_*}}{4}-q>0$.
\begin{Cor}\label{rt2}
    Suppose that $q \in (0,2 - \frac{\alpha^2 \mu_{i_*}}{4})$, $\rho = 2 - \frac{\alpha^2 \mu_{i_*}}{4} - q$, and  $D_t \ge 0$.
    There exists $C_{\alpha,q,\rho}>0$ depending only on $\alpha$, $q$ and $\rho$ such that if $t\ge C_{\alpha,q,\rho}\frac{d L}{\Delta^2}$, then
    \[
    \E\left[\left(\hat L_{t+L,i_*}-\hat L_{t,i_*}-L\mu_{i_*}\right)^2\,\middle|\,\mathscr{F}_{t-1}\right]\leq\mu_{i_*}\left(\frac{\alpha^2}{4}+\frac{\rho}{4}\right) \frac{Lx^2}{t},
    \]
    where $x = D_t + \Delta t$.
\end{Cor}
\begin{proof}
    When $C_{\alpha,q,\rho}$ is large enough (implying $t$ is large enough), we have
    \[
    \exp \left(\frac{C_\alpha L}{t}\right)\leq 2;\quad \frac{\alpha \sqrt{d}}{4 \Delta \sqrt{t}},\frac{d}{\Delta^2 t}\leq \frac{\rho}{40};\quad \frac{\alpha^2 L^2}{8}\leq \frac{\rho L (\Delta t)^2}{40 t} \leq \frac{\rho Lx^2}{40 t}.
    \]
    % where we used that $x\ge \Delta t$ in the final inequality. 
    Hence, by Lemma \ref{lstep}, we have
    \[
    \sum_{j=0}^{L-1} \mathbb{E}\left[s_{t+j} \mid \mathscr{F}_{t-1}\right] \leq\exp \left(\frac{C_\alpha L}{t}\right)\frac{\alpha^2}{4}\frac{Lx^2}{t}+\frac{3\rho Lx^2}{20t}.
    \]
    Finally, note that $e^u-1\leq 2u$ when $0\leq u\leq 1$, then when $C_{\alpha,q,\rho}$ is large enough, we also have
    \[
    \left[\exp \left(\frac{C_\alpha L}{t}\right)-1\right]\cdot \frac{\alpha^2}{4}\frac{Lx^2}{t}\leq \frac{C_\alpha \alpha^2 L}{2t}\frac{Lx^2}{t}\leq \frac{\rho Lx^2}{10 t}.
    \]
    Therefore,
    \[
    \sum_{j=0}^{L-1} \mathbb{E}\left[s_{t+j} \mid \mathscr{F}_{t-1}\right] \leq\left(\frac{\alpha^2}{4}+\frac{\rho}{4}\right) \frac{Lx^2}{t},
    \]
    and then it suffices to apply \eq{secondL} and \eq{st}.
\end{proof}
\subsection{The negative half-line}
For the negative half-line, $g_t(y)=\Delta^2t^{q}e^{\lambda y}$.
As shown in \eq{ratione}, we need to control
\[
\E\left[\exp\left(\lambda\left(\hat L_{t+L,i_*}-\hat L_{t,i_*}-L\mu_{i_*}\right)\right)\,\middle|\,\mathscr{F}_{t-1}\right].
\]
In fact, we have:

\begin{Lem}\label{negexp}
 Suppose that $4dL\le\sqrt{dt}/(4\alpha)$ and
$0\ge D_t\ge-\sqrt{dt}/\alpha$. Then, for every $0\le\lambda\le1/(4d)$,
\[
\E\left[\exp\left(\lambda\left(\hat L_{t+L,i_*}-\hat L_{t,i_*}-L\mu_{i_*}\right)\right)\,\middle|\,\mathscr{F}_{t-1}\right]
\le e^{4dL\lambda^2}.
\]
\end{Lem}
\begin{proof}
Since $D_t\le0$, Lemma~\ref{lem:block_deterministic_control} (2) gives, for every $0\le r<L$,
\[
D_{t+r}\le D_t+4dr\le 4dL.
\]
Since $4dL\le\sqrt{dt}/(4\alpha)$, Lemma~\ref{4-2} implies
\[
\frac1{p_{t+r,i_*}}
\le\left(\sqrt d+\frac{\alpha D_{t+r}^+}{2\sqrt{t+r}}\right)^2
\le4d.
\]
Thus $|\hat{\ell}_{t+r,i_*}-\mu_{i_*}|\le 4d$ and
\[
\E\left[(\hat{\ell}_{t+r,i_*}-\mu_{i_*})^2\,\middle|\,\mathscr{F}_{t+r-1}\right]
\le\E\left[\hat{\ell}_{t+r,i_*}^2\,\middle|\,\mathscr{F}_{t+r-1}\right]
\le 4d.
\]
Lemma~\ref{bounded} and the tower property now give the result.
\end{proof}

\section{Auxiliary Lemmas}
\begin{Lem}[Comparison between the two branches of $g_t$]
\label{lem:comparison_gn}
If $\theta=2-q>0$, then the following hold.

\begin{enumerate}[(1)]
\item If
\(
t \ge \frac{2}{\Delta\lambda},
\)
then for every $y\ge 0$,
\begin{equation}
t^{-\theta}(\Delta t+y)^2 \le \Delta^2 t^q e^{\lambda y}.
\label{eq:lem31-item1}
\end{equation}

\item Fix $R>0$. If
\(
t \ge \frac{R+2/\lambda}{\Delta},
\)
then for every $z\in[-R,0]$,
\begin{equation}
\Delta^2 t^q e^{\lambda z} \le t^{-\theta}(\Delta t+z)^2.
\label{eq:lem31-item2}
\end{equation}
\end{enumerate}
\end{Lem}

\begin{proof}
For the first result, since $\theta=2-q$, we have
\[
t^{-\theta}(\Delta t+y)^2
=
t^{q-2}(\Delta t+y)^2
=
\Delta^2 t^q\left(1+\frac{y}{\Delta t}\right)^2.
\]
Therefore \eq{eq:lem31-item1} is equivalent to
\[
2\log\left(1+\frac{y}{\Delta t}\right)\le \lambda y.
\]
Then it suffices to note that $\log(1+x)\leq x, x>-1.$

For the second part, again using $\theta=2-q$, inequality \eq{eq:lem31-item2} is equivalent to
\[
\lambda z \le 2\log\left(1+\frac{z}{\Delta t}\right),
\qquad z\in[-R,0].
\]
By the concavity of the logarithmic function, it suffices to show the case when $z=-R$, which is equivalent to
\[
\lambda R\ge -2\log\left(1-\frac{R}{\Delta t}\right).
\]
When \(
t \ge \frac{R+2/\lambda}{\Delta},
\) the right-hand side is less than $2\log(1+\frac{\lambda R}2)$. Then it suffices to use that $\log(1+x)\leq x, x>-1$ again.
\end{proof}
\begin{Lem}\label{ht}
Let \((X_t)_{t\ge 1}\) be a stochastic process adapted to a filtration \((\mathscr F_t)_{t\ge 1}\). Suppose there exist a sequence of nonnegative random variables \((H_t)_{t\ge 1}\) and a sequence of sets \((\mathcal{M}_t)_{t\ge 1}\) with \(\mathcal{M}_t\subset \mathbb R\) such that for every \(t\ge 1\),
\[
\mathbb E\!\left[H_{t+1}\,\middle|\, \mathscr F_t\right]\le H_t
\qquad \text{on the event }\{X_t\in \mathcal{M}_t\}.
\]
Define the stopping time
\[
\tau:=\inf\{t\ge 1:\, X_t\notin \mathcal{M}_t\},
\]
with the convention \(\inf\varnothing=+\infty\). Then the stopped process
\[
\bigl(H_{t\wedge \tau}\bigr)_{t\ge 1}
\]
is a supermartingale with respect to \((\mathscr F_t)\).

Moreover, define recursively a sequence of stopping times \((\sigma_k,\tau_k)_{k\ge 0}\) by
\[
\sigma_0:=1,
\qquad
\tau_k:=\inf\{t\ge \sigma_k:\, X_t\notin \mathcal{M}_t\},
\]
and, for \(k\ge 1\),
\[
\sigma_k:=\inf\{t>\tau_{k-1}:\, X_t\in \mathcal{M}_t\},
\]
again with the convention \(\inf\varnothing=+\infty\). Then, for every \(t\ge 1\),
\[
\mathbb E[H_t\1_{\{X_t\in\mathcal{M}_t\}}]
\le
\sum_{k=0}^{+\infty}\mathbb E\!\left[H_{\sigma_k}\1_{\{\sigma_k<\infty\}}\right].
\]
\end{Lem}

\begin{proof}
We first prove the supermartingale claim. Fix \(t\ge 1\). Since \(\{\tau\le t\}\in \mathscr F_t\), we have
\[
H_{(t+1)\wedge \tau}
=
H_{\tau}\1_{\{\tau\le t\}}+H_{t+1}\1_{\{\tau>t\}}.
\]
Taking conditional expectation with respect to \(\mathscr F_t\), we obtain
\[
\mathbb E\!\left[H_{(t+1)\wedge \tau}\,\middle|\, \mathscr F_t\right]
=
H_{\tau}\1_{\{\tau\le t\}}
+
\1_{\{\tau>t\}}
\mathbb E\!\left[H_{t+1}\,\middle|\, \mathscr F_t\right].
\]
On the event \(\{\tau>t\}\), by definition of \(\tau\), we have \(X_t\in \mathcal{M}_t\). Hence, by the assumed one-step inequality,
\[
\1_{\{\tau>t\}}
\mathbb E\!\left[H_{t+1}\,\middle|\, \mathscr F_t\right]
\le
\1_{\{\tau>t\}}\, H_t.
\]
Substituting this into the previous display gives
\[
\mathbb E\!\left[H_{(t+1)\wedge \tau}\,\middle|\, \mathscr F_t\right]
\le
H_{\tau}\1_{\{\tau\le t\}}
+
H_t\1_{\{\tau>t\}}
=
H_{t\wedge \tau}.
\]
Therefore, \(\bigl(H_{t\wedge \tau}\bigr)_{t\ge 1}\) is a supermartingale.

We now prove the additional estimate. Note that the intervals \([\sigma_k,\tau_k)\) are pairwise disjoint and 
\[
\left\{X_t\in\mathcal{M}_t\right\}=\bigcup_{k=0}^{+\infty}\{\sigma_k\le t<\tau_k\},
\]
hence,
\[
\mathbb E[H_t\1_{\{X_t\in\mathcal{M}_t\}}]
=
\sum_{k=0}^{+\infty}\mathbb E\!\left[H_t\,\1_{\{\sigma_k\le t<\tau_k\}}\right].
\]

Fix \(k\ge 0\). On the event \(\{\sigma_k<+\infty\}\), consider the shifted process \((\mathscr F_{\sigma_k+s})_{s\ge 1}\). By the first part of the lemma applied conditionally on \(\mathscr F_{\sigma_k}\), the stopped process
\(
\bigl(H_{(\sigma_k+s)\wedge \tau_k}\bigr)_{s\ge 0}
\)
is a supermartingale with respect to \((\mathscr F_{\sigma_k+s})_{s\ge 1}\). In particular, since on the event \(\{\sigma_k\le t<\tau_k\}\) we have \(t\wedge \tau_k=t\), it follows that
\[
\mathbb E\!\left[H_t\,\1_{\{\sigma_k\le t<\tau_k\}}\,\middle|\, \mathscr F_{\sigma_k}\right]
\leq
\1_{\{\sigma_k\le t\}}\,
\mathbb E\!\left[H_{t\wedge \tau_k}
% \1_{\{t<\tau_k\}}
\,\middle|\, \mathscr F_{\sigma_k}\right]
\le
\1_{\{\sigma_k\le t\}}\, H_{\sigma_k}.
\]
Taking expectations yields
\[
\mathbb E\!\left[H_t\,\1_{\{\sigma_k\le t<\tau_k\}}\right]
\le
\mathbb E\!\left[H_{\sigma_k}\1_{\{\sigma_k\le t\}}\right]
\le
\mathbb E\!\left[H_{\sigma_k}\1_{\{\sigma_k<+\infty\}}\right].
\]
Finally, it suffices to sum over $k\ge 0$.
\end{proof}

\begin{Lem}
\label{lem:standalone-sum-bound}
Let $q>0$, $\lambda>0$, $L>0$, and $N\ge 0$. Define \(
f(x):=x^q e^{-\lambda\sqrt{x}},x\ge 0.
\)
We have
\[
\sup_{x\ge 0} f(x)=
\left(\frac{2q}{e\lambda}\right)^{2q}.
\]
Then there exists a constant
$C_q>0$, depending only on $q$, such that
\[
\sum_{j=0}^{+\infty} (N+jL)^q e^{-\lambda\sqrt{N+jL}}
\le
C_q\,\lambda^{-2q}\left(1+\frac{1}{\lambda^2L}\right).
\]
\end{Lem}

\begin{proof}
Let
\(
x_*:=\left(\frac{2q}{\lambda}\right)^2.
\)
For every $x>0$,
\[
f'(x)
=
e^{-\lambda\sqrt{x}}x^{q-1}
\left(q-\frac{\lambda}{2}\sqrt{x}\right).
\]
Hence $f$ is nonincreasing on $[x_*,+\infty)$. Moreover, 
\[
\sup_{x\ge 0} f(x)
=
f(x_*)
=
\left(\frac{2q}{e\lambda}\right)^{2q}.
\]

Set
\[
S:=\sum_{j=0}^{+\infty} f(N+jL).
\]
We split the sum into the part below $x_*$ and the tail above $x_*$. First consider the indices $j\ge 0$ such that $N+jL\le x_*$. Their number is at most
\(
1+\frac{x_*}{L}.
\)
Therefore,
\[
\sum_{\{j:\,N+jL\le x_*\}} f(N+jL)
\le
\left(1+\frac{x_*}{L}\right)\sup_{x\ge 0} f(x)
=
\left(1+\frac{(2q/\lambda)^2}{L}\right)\left(\frac{2q}{e\lambda}\right)^{2q}.
\]

Now consider the remaining indices, i.e. those with $N+jL>x_*$. Let
\[
j_0:=\min\{j\ge 0:\ N+jL>x_*\}.
\]
Then
\[
\sum_{\{j:\,N+jL>x_*\}} f(N+jL)
=
f(N+j_0L)+\sum_{m=1}^{+\infty} f(N+(j_0+m)L).
\]
Since $N+j_0L>x_*$ and $f$ is nonincreasing on $[x_*,+\infty)$, we have
\[
f(N+j_0L)\le \sup_{x\ge 0} f(x)
=
\left(\frac{2q}{e\lambda}\right)^{2q},
\]
and also
\[
\sum_{m=1}^{+\infty} f(N+(j_0+m)L)
\le
\frac{1}{L}\int_{N+j_0L}^{+\infty} f(x)\,\d x
\le
\frac{1}{L}\int_{x_*}^{+\infty} x^q e^{-\lambda\sqrt{x}}\,\d x.
\]
With the change of variables
\(
y=\lambda\sqrt{x},
\)
we obtain
\[
\int_{x_*}^{+\infty} x^q e^{-\lambda\sqrt{x}}\,\d x
=
2\lambda^{-(2q+2)}\int_{2q}^{+\infty} y^{2q+1}e^{-y}\,dy
\le
2\Gamma(2q+2)\lambda^{-(2q+2)}.
\]
Hence
\[
\sum_{\{j:\,N+jL>x_*\}} f(N+jL)
\le
\left(\frac{2q}{e}\right)^{2q}\lambda^{-2q}
+
\frac{2\Gamma(2q+2)}{L}\lambda^{-(2q+2)}.
\]

Combining the two parts, we conclude that
\[
S
\le
2\left(\frac{2q}{e}\right)^{2q}\lambda^{-2q}
+
\left[
(2q)^2\left(\frac{2q}{e}\right)^{2q}
+
2\Gamma(2q+2)
\right]
\frac{\lambda^{-(2q+2)}}{L}.
\]
% Therefore the desired estimate holds with
% \[
% C_q
% :=
% 2\left(\frac{2q}{e}\right)^{2q}
% +
% (2q)^2\left(\frac{2q}{e}\right)^{2q}
% +
% 2\Gamma(2q+2).
% \]
This proves the lemma.
\end{proof}

\begin{Lem}[First two moments from a one-sided tail bound]
\label{lem:max_one_sided_tail_moments}
Let $\mathscr G$ be a sigma-field, let $m\ge1$, and let $Z_1,\dots,Z_m$ be real-valued random variables. Suppose that, for some $v>0$, every $i\in[m]$ and $x\ge0$ satisfy
\[
\P\left(Z_i\ge x\,\middle|\,\mathscr G\right)
\le \exp\left(-\frac{x^2}{2(v+x/3)}\right)
\qquad\text{almost surely}.
\]
No independence is required. Define
\[
M_m:=\left(\max_{1\le i\le m}Z_i\right)^+.
\]
There exists a universal constant $C>0$ such that
\[
\E[M_m\mid\mathscr G]
\le C\left(\sqrt{v\log(2m)}+\log(2m)\right)
\]
and
\[
\E[M_m^2\mid\mathscr G]
\le C\left(v\log(2m)+(\log(2m))^2\right).
\]
\end{Lem}

\begin{proof}
Set $a:=\log(2m)$ and $h(u):=\sqrt{2vu}+2u/3$. The function $h$ is strictly increasing from $[0,\infty)$ onto $[0,\infty)$. Since
\[
h(u)^2\ge2u\left(v+\frac{h(u)}3\right),
\]
the union bound gives, for every $y\ge0$,
\[
\P\left(M_m\ge h(a+y)\,\middle|\,\mathscr G\right)
\le m e^{-(a+y)}\le e^{-y}.
\]
Let $Y:=(h^{-1}(M_m)-a)^+$. Then
\[
\P(Y\ge y\mid\mathscr G)\le e^{-y},
\qquad
\E[Y\mid\mathscr G]\le1,
\qquad
\E[Y^2\mid\mathscr G]\le2.
\]
Moreover,
\[
M_m\le h(a+Y)
\le \sqrt{2va}+\sqrt{2vY}+\frac{2a}{3}+\frac{2Y}{3}.
\]
Taking conditional expectations and using conditional Jensen's inequality proves the first bound. Finally,
\[
h(a+Y)^2\le4v(a+Y)+\frac89(a+Y)^2.
\]
The bounds on the first two conditional moments of $Y$, together with $a\ge\log2$, prove the second bound.
\end{proof}

\begin{Lem}\label{bounded}
Let $\mathscr G$ be a sigma-field. If $Z$ is a random variable such that
\[
\E[Z\mid\mathscr G]=0,
\qquad
\E[Z^2\mid\mathscr G]\le\sigma^2,
\qquad
|Z|\le M \quad\text{almost surely},
\]
where $M>0$, then, for every $0\le\lambda\le1/M$,
\[
\E\left[e^{\lambda Z}\,\middle|\,\mathscr G\right]
\le e^{\lambda^2\sigma^2}.
\]
\end{Lem}
\begin{proof}
For every $-1\le a\le1$, we have $e^a\le1+a+a^2$. Since $|\lambda Z|\le1$,
\[
\E\left[e^{\lambda Z}\,\middle|\,\mathscr G\right]
\le1+\lambda\E[Z\mid\mathscr G]
+\lambda^2\E[Z^2\mid\mathscr G]
\le1+\lambda^2\sigma^2
\le e^{\lambda^2\sigma^2}.
\]
\end{proof}

\begin{Lem}
\label{lem:one_plus_u_negative_theta}
Let $\theta>0$. Then for every $u\in[0,1]$,
\[
(1+u)^{-\theta}\le 1-\theta u+\frac{\theta(\theta+1)}{2}u^2.
\]
\end{Lem}

\begin{proof}
Define
\[
f(u):=(1+u)^{-\theta}, \qquad u\in[0,1].
\]
Then
\[
f'(u)=-\theta(1+u)^{-\theta-1},
\qquad
f''(u)=\theta(\theta+1)(1+u)^{-\theta-2}.
\]
Since $u\in[0,1]$, we have $1+u\ge 1$, and therefore
\[
f''(u)\le \theta(\theta+1).
\]

Applying Taylor's theorem at $0$, for each $u\in[0,1]$ there exists some
$\xi\in(0,u)$ such that
\[
(1+u)^{-\theta}
=1-\theta u+\frac{f''(\xi)}{2}u^2
\le
1-\theta u+\frac{\theta(\theta+1)}{2}u^2.
\]
This completes the proof.
\end{proof}

\begin{Lem}\label{lem:zm-superlinear}
Let $r>0$ and $\alpha>0$, and suppose the sequence $(z_m)_{m\ge1}$ satisfies that 
\[
z_{m+1}\ge z_m+r\left(1+\frac{\alpha z_m}{2\sqrt m}\right)^2,\qquad m\ge1,
\]
with $z_1=0$. Then
\[
\frac{z_m}{m}\to+\infty \qquad \text{as } m\to+\infty.
\]
\end{Lem}

\begin{proof}
Since
\[
z_{m+1}-z_m
\ge r\left(1+\frac{\alpha z_m}{2\sqrt m}\right)^2
\ge r,
\]
we have $z_m\ge r(m-1)$ for all $m\ge1$. Therefore,
\[
z_{m+1}-z_m
\ge \frac{r\alpha^2}{4}\frac{z_m^2}{m}
\ge \frac{r\alpha^2}{4}\frac{r^2(m-1)^2}{m}.
\]
Hence there exists a constant $c>0$ such that, for all sufficiently large $m$,
\[
z_{m+1}-z_m\ge c\,m.
\]
Summing this inequality yields $z_m\ge c' m^2$ for all sufficiently large $m$, for some $c'>0$. Consequently,
\[
\frac{z_m}{m}\ge c' m \to +\infty.
\]
This proves the claim.
\end{proof}

\begin{Lem}\label{lem:harmonic-log}
For any integers $1\le t \le k$, one has
\[
\log\!\left(\frac{k}{t}\right)\leq \sum_{s=t}^{k}\frac1s \le 1+\log\left(\frac{k}{t}\right).
\]
\end{Lem}

\begin{proof}
It suffices to note that
\[
\sum_{s=t+1}^{k}\frac1s
\le
\int_t^k \frac{\d x}{x}
=
\log\left(\frac{k}{t}\right), \quad\sum_{s=t}^k \frac{1}{s}
\ge
\int_t^{k+1} \frac{\d x}{x}
=
\log\!\left(\frac{k+1}{t}\right).
\]
\end{proof}
\begin{Lem}\label{lem:half-mean-lower-tail}
Let $X$ be a random variable such that $0\le X\le 1$ almost surely. Then
\[
\mathbb{P}\!\left(X\ge \frac{\mathbb{E}[X]}{2}\right)\ge \frac{\mathbb{E}[X]}{2}.
\]
\end{Lem}

\begin{proof}
Let $m:=\mathbb{E}[X]$. Since $X\le m/2$ on $\{X<m/2\}$ and $X\le 1$ on $\{X\ge m/2\}$,
\[
m=\mathbb{E}[X]
\le \frac{m}{2}+\mathbb{P}(X\ge m/2).
\]
Hence
\[
\mathbb{P}(X\ge m/2)\ge \frac{m}{2}.
\]
\end{proof}

% \begin{Lem}\label{lem:positive-a.s.-positive-mean}
% Let $(\Omega,\mathscr{F},\mathbb{P})$ be a probability space, and let $X:\Omega\to[0,\infty)$ be a random variable such that $X(\omega)>0$ for $\mathbb{P}$-almost every $\omega\in\Omega$. Then
% \[
% \mathbb{E}[X]>0.
% \]
% \end{Lem}

% \begin{proof}
% Since $X>0$ almost surely,
% \[
% \bigl\{X>1/n\bigr\}\uparrow \Omega \quad \text{up to a $\mathbb{P}$-null set}.
% \]
% Hence there exists some $n\ge1$ such that
% \[
% \mathbb{P}(X>1/n)>0.
% \]
% Therefore,
% \[
% \mathbb{E}[X]\ge \frac{1}{n}\,\mathbb{P}(X>1/n)>0.
% \]
% \end{proof}

\begin{Lem}\label{lem:positive-a.s.-positive-mean}
Let \(X\) be a nonnegative random variable on a probability space
\((\Omega,\mathcal F,\mathbb P)\). If \(X\) is not almost surely zero, i.e.,
\[
    \mathbb P(X>0)>0,
\]
then
\[
    \mathbb E[X]>0.
\]
\end{Lem}

\begin{proof}
Since \(X\ge 0\), we have
\[
    \{X>0\}=\bigcup_{n=1}^{+\infty}\left\{X>\frac1n\right\}.
\]
Because \(\mathbb P(X>0)>0\), there exists \(n_0\ge 1\) such that
\[
    \mathbb P\left(X>\frac1{n_0}\right)>0.
\]
Therefore,
\[
    \mathbb E[X]
    \ge
    \mathbb E\left[X\1_{\{X>1/n_0\}}\right]
    \ge
    \frac1{n_0}\mathbb P\left(X>\frac1{n_0}\right)
    >0.
\]
This proves the claim.
\end{proof}

\begin{Lem}[One-sided Freedman inequality]\label{lem:one-sided-freedman}
Let $(\mathscr G_s)_{s=0}^n$ be a filtration, where $\mathscr G_0$ need not be trivial, and let $(X_s)_{s=1}^n$ be a martingale-difference sequence with respect to this filtration. Suppose that, for some $b,v>0$,
\[
X_s\le b \quad\text{almost surely for every $1\le s\le n$},
\qquad
\sum_{s=1}^n\E[X_s^2\mid\mathscr G_{s-1}]\le v
\quad\text{almost surely}.
\]
Then, for every $x\ge0$,
\[
\P\left(\sum_{s=1}^nX_s\ge x\,\middle|\,\mathscr G_0\right)
\le\exp\left(-\frac{x^2}{2(v+bx/3)}\right)
\qquad\text{almost surely}.
\]
\end{Lem}

\section{Verification of the parameter constraints}\label{verify}
In this section, we verify the parameter constraints in the proofs of both Theorem \ref{lya} and Lemma \ref{base}.

\subsection{Verification for Theorem \ref{lya}}\label{verify1}

Recall the parameter choices in Lemma~\ref{formalsuper}:

\[
\lambda=C\frac{\Delta}{d},
\qquad
L=\left\lfloor C_{\alpha,q,\rho}\frac{d\log(d)}{\Delta^2}\right\rfloor+1,
\qquad
N=\left\lfloor C'_{\alpha,q,\rho}\frac{d^3(\log(d))^2}{\Delta^4}\right\rfloor+1.
\]
Choosing the universal constant $C>0$ sufficiently small gives
\[
\lambda\le\frac1{8d},
\qquad
\lambda\le\frac{\Delta}{64d},
\]
which verifies \eq{eq:lya-lambda-small} and the $\lambda$-constraint in
\eq{eq:lya-neg-exp-params}. Choosing $C_{\alpha,q,\rho}$ sufficiently large gives
\[
L\ge\frac{4C'd\log(d)}{\Delta^2},
\qquad
L\ge\frac{2\log(d)}{\lambda\Delta},
\]
so \eq{eq:lya-L-ut1} and the remaining condition in
\eq{eq:lya-neg-exp-params} hold.

It remains to verify the lower bounds on $N$. Since
\[
L\lesssim_{\alpha,q,\rho}\frac{d\log(d)}{\Delta^2},
\]
we have
\[
256\alpha^2dL^2
\lesssim_{\alpha,q,\rho}
\frac{d^3(\log(d))^2}{\Delta^4}.
\]
Thus $C'_{\alpha,q,\rho}$ can be chosen so that
$4dL\le\sqrt{dN}/(4\alpha)$, which is \eq{eq:lya-time-block}. Moreover,
\[
\begin{gathered}
\frac{4dL+2/\lambda}{\Delta},\quad
\frac{dL}{\Delta^2},\quad
\frac{1}{\rho^2\Delta^2}
\left(4d\log(d)+\frac{(\log(d))^2}{L}\right),\\
\frac{L}{\rho},\quad
\frac{qL}{\lambda\Delta},\quad
L,\quad
\frac{d}{\Delta^2}
\end{gathered}
\;\lesssim_{\alpha,q,\rho}\;
\frac{d^3(\log(d))^2}{\Delta^4}.
\]
Here we use $d\ge2$ and $0<\Delta\le1$. After increasing
$C'_{\alpha,q,\rho}$ if necessary, the constraints
\eq{eq:lya-time-comparison}, \eq{eq:lya-time-rt2}, and
\eq{eq:lya-time-term2} all hold. The same choice verifies
\eq{eq:lya-time-one-plus}, \eq{eq:lya-time-neg-final}, $N\ge L$, and
$\Delta^2N\ge d$. This completes the verification for Theorem~\ref{lya}.

\subsection{Verification for Lemma \ref{base}}\label{verify2}

Recall definitions in Lemma \ref{base}:
\[
    B>\frac{Cd}{\mu_{i_*}\alpha^2},
    \qquad
    t\ge Cd^2,
\]
and
\[
    v_B=\frac{\mu_{i_*}\alpha^2B^2}{16},
    \qquad
    V_B=4\alpha^2B^2,
    \qquad
    a_B=\frac{\mu_{i_*}\alpha^2B}{48}.
\]
Choosing the universal constant \(C>0\)
sufficiently large gives
\[
    B\ge 8d,
    \qquad
    Cd^2\ge 144d^2,
    \qquad
    V_B\ge 8d,
    \qquad
    a_B>1.
\]
Moreover, since $q_B=C'\mu_{i_*}/(\alpha^2B^2)$ and $B>Cd/(\mu_{i_*}\alpha^2)$,
\[
0<q_B<C'\frac{\mu_{i_*}^3\alpha^2}{C^2d^2}<1
\]
after increasing $C$ if necessary.
Thus \eq{eq:lower-time-lone}, \eq{eq:lower-BgeK-mplus}
% ,
% \eq{eq:lower-start-basic}
, and \eq{eq:lower-aB} hold.

When $C\ge 5$, we have
\[
\left(\sqrt{3}-\frac32\right)\alpha B\ge \sqrt{d},
\]
hence \eq{eq:ptlower} holds.

Assume \(v_B\le D_t/t\le V_B\). When $C\ge144$,
\[
v_B = \frac{\mu_{i_*}\alpha^2 B^2}{16} > \frac{Cd}{16}B \ge 18B>3B,
\]
hence,
\[
m_{-}(t)\ge\frac{D_t}{3 B}\ge\frac{v_B t}{3B}>t.
\]
The same parameter choices give
\[
\frac{D_t}{B}\ge\frac{v_Bt}{B}=\frac{\mu_{i_*}\alpha^2Bt}{16}>\frac{Cdt}{16}\ge\frac{C^2d^3}{16}>42.
\]
Then
\[
    m_+(t)-m_-(t)
    \ge
    \frac{D_t}{6B}-2
    \ge
    \frac{v_Bt}{6B}-2>0.
\]
Hence \eq{eq:lower-m-order} holds. The bound $D_t/B>42$ also gives
\[
    \frac{m_+(t)}{m_-(t)+1}
    \ge
    \frac{D_t/(2B)-1}{D_t/(3B)+2}
    >\frac54,
\]
which verifies \eq{eq:lower-m-ratio}.

% The upper bound \(D_t/t\le V_B\) gives
% \[
%     \frac{k}{t}
%     \le
%     \frac{m_+(t)}{t}
%     \le
%     \frac{V_B}{2B}
%     =
%     2\alpha^2B,
% \]
% so \eq{eq:lower-k-over-t} holds. Consequently, for \(B\) sufficiently large,
% \eq{eq:lower-exp-control} follows. 

Since, when $C$ is large enough, $B\ge \alpha^{-1}$ and $B\ge 16$, we have
\[
\frac{8}{\alpha^2B^2}
    \left(1+2\alpha^2B\right)
    \le 9.
\]
Consequently, \eq{eq:lower-exp-control} follows. 

Also, since \(s\ge t\ge Cd^2\), increasing
\(C\) gives
\[
    \frac{4}{\alpha^2B^2s}\le \frac12,
\]
which verifies \eq{eq:lower-product-small}.

Next,
\[
    3B+3\alpha^2B^2\le 4\alpha^2B^2=V_B
\]
for \(B\) sufficiently large, so \eq{eq:lower-return-upper} holds. 
% Moreover,
% \[
%     \frac{\alpha^2B^2}{4}\frac{\mu_{i_*}}2
%     =
%     \frac{\mu_{i_*}\alpha^2B^2}{8}
%     =
%     2v_B,
% \]
% which gives \eq{eq:lower-return-lower}. 
% Finally,
% \[
%     V_B-2v_B
%     =
%     4\alpha^2B^2-\frac{\mu_{i_*}\alpha^2B^2}{8}
%     \ge
%     \frac{31}{8}\alpha^2B^2,
% \]
% which dominates \(2K\) when \(C\) is sufficiently large. Hence
% \eq{eq:lower-start-vB} holds.
This completes the verification for
Lemma \ref{base}.

\end{document}